\documentclass{article}

\usepackage{microtype}
\usepackage{graphicx}
\usepackage{subcaption}
\usepackage{booktabs} 

\usepackage{hyperref}

\usepackage[accepted]{icml2026}

\usepackage{amsmath,amssymb,amsfonts}
\usepackage{mathtools}
\usepackage{amsthm}
\usepackage{pifont} 
\usepackage{booktabs}      
\usepackage{multirow}      
\usepackage{makecell}      
\usepackage{arydshln}      
\usepackage{colortbl}      
\usepackage{xcolor}        
\usepackage{graphicx}      
\usepackage{algorithm}
\usepackage{algorithmic}

\usepackage[capitalize,noabbrev]{cleveref}

\theoremstyle{plain}
\newtheorem{theorem}{Theorem}[section]
\newtheorem{proposition}[theorem]{Proposition}
\newtheorem{lemma}[theorem]{Lemma}

\theoremstyle{definition}

\theoremstyle{remark}

\usepackage[textsize=tiny]{todonotes}

\newcommand{\tablestyle}[2]{\setlength{\tabcolsep}{#1}\renewcommand{\arraystretch}{#2}\centering\footnotesize}

\newcommand{\AlgNote}[1]{%
  \STATE \textcolor{gray}{\footnotesize #1}%
}

\usepackage{amsmath,amsfonts,bm}

\def\eqref#1{equation~\ref{#1}}

\def\1{\bm{1}}

\DeclareMathAlphabet{\mathsfit}{\encodingdefault}{\sfdefault}{m}{sl}
\SetMathAlphabet{\mathsfit}{bold}{\encodingdefault}{\sfdefault}{bx}{n}

\icmltitlerunning{TWLA: Achieving \underline{T}ernary \underline{W}eights and \underline{L}ow-Bit \underline{A}ctivations for LLMs via Post-Training Quantization}

\begin{document}

\twocolumn[
  \icmltitle{TWLA: Achieving \underline{T}ernary \underline{W}eights and \underline{L}ow-Bit \underline{A}ctivations for LLMs via Post-Training Quantization}

  \icmlsetsymbol{equal}{*}
  \icmlsetsymbol{intern}{\ensuremath{\dagger}}
  \icmlsetsymbol{corr}{\ensuremath{\ddagger}}

  \begin{icmlauthorlist}
    \icmlauthor{Zhixiong Zhao}{houmo,equal,intern}
    \icmlauthor{Zukang Xu}{houmo,equal}
    \icmlauthor{Zhixuan Chen}{houmo}
    \icmlauthor{Xing Hu}{houmo}
    \icmlauthor{Zhe Jiang}{seu}
    \icmlauthor{Dawei Yang}{houmo,corr}
  \end{icmlauthorlist}

  \icmlaffiliation{houmo}{Houmo AI, China}
  \icmlaffiliation{seu}{Southeast University, Nanjing, China}

  \icmlcorrespondingauthor{Dawei Yang}{dawei.yang@houmo.ai}

  \icmlkeywords{Machine Learning, Large Language Models, Post-Training Quantization}

  \vskip 0.3in
]

\printAffiliationsAndNotice{
  \icmlEqualContribution
  \textsuperscript{\ensuremath{\dagger}}This work was conducted during his internship at Houmo AI.
  \textsuperscript{\ensuremath{\ddagger}}Corresponding author.
}

\begin{abstract}
Large language models (LLMs) exhibit exceptional general language processing capabilities, but their memory and compute costs hinder deployment. Ternarization has emerged as a promising compression technique, offering significant reductions in model size and inference complexity. However, existing methods struggle with heavy-tailed activation distributions and therefore keep activations in high precision, fundamentally limiting end-to-end inference acceleration. 
To overcome this limitation, we propose \textbf{TWLA}, a post-training quantization (PTQ) framework that achieves 1.58-bit weight compression and 4-bit activation quantization while maintaining high accuracy. TWLA comprises three components: (1) Euclidean-to-Manifold Asymmetric Ternary Quantizer (E2M-ATQ) minimizes layer-output error under weight ternarization via a two-stage optimization from Euclidean initialization to manifold relocation; (2) Kronecker Orthogonal Tri-Modal Shaping (KOTMS) applies a Kronecker-structured orthogonal rotation to reshape weights into ternary-friendly tri-modal distributions, while the shared rotation statistically suppresses activation outliers; and (3) Inter-Layer Aware Activation Mixed Precision (ILA-AMP) explicitly introduces adjacent-layer second-order interaction costs in bit allocation and jointly optimizes for the layer-wise disparity of activation quantization gains induced by the shared orthogonal transform, preventing cascades triggered by a few weak layers.
Extensive experiments demonstrate that TWLA maintains high accuracy under \textbf{W1.58A4}, while delivering significant inference acceleration. The code is available at \href{https://github.com/Kishon-zzx/TWLA}{TWLA}.

\end{abstract}
\section{Introduction}
\label{sec:intro}

\begin{figure}
    \centering
    \includegraphics[width=0.92\linewidth]{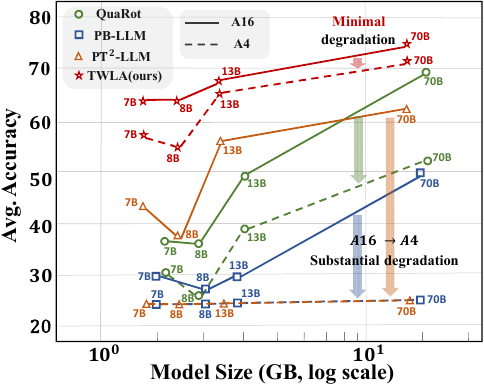}
    \caption{LLaMA-family performance on seven zero-shot tasks. TWLA remains robust under both weight-only and weight–activation quantization (at equal memory cost), while other methods degrade substantially with 4-bit activation quantization.}
    \label{fig:intro}
    \vspace{-0.2cm}
\end{figure}

In recent years, Large Language Models (LLMs) have achieved outstanding results across a wide range of diverse domains~\citep{zhu2025pathology}. However, this strong capability largely comes from scaling up model size. Many state-of-the-art models contain billions, or even hundreds of billions, of parameters. This scale creates heavy demands on memory and computation during inference. For example, DeepSeek-R1-671B~\citep{Guo_2025} has hundreds of billions of parameters. With FP16 inference, storing the weights alone can require no less than 1 TB of memory. Such a high resource requirement makes deployment on edge devices and other resource-limited platforms difficult~\citep{zhao2025quark}.  Therefore, reducing inference cost while preserving model quality has become a central challenge for practical and sustainable LLM deployment~\citep{li2024towards,li2025adaptive}.

\begin{figure*}
    \centering
    \includegraphics[width=0.94\linewidth]{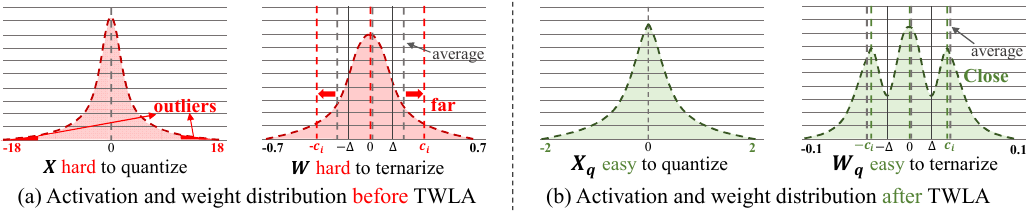}
    \caption{(a) Before applying TWLA, activation outliers hinder low-bit quantization and unimodal weights are misaligned with ternary codebooks. (b) After applying TWLA, activations are smoothed and weights are transformed into
symmetric tri-modal forms.}
    \label{fig2:vision}
\end{figure*}

Quantization has become a core technique for model compression~\citep{zhao2026specquant,xu2026kbvq,gul2026flrqfasterllmquantization} and ternarization is a representative option. It constrains weights to $\{-1, 0, +1\}$, which enables high compression and reduces compute complexity. 
Similar to binarization~\citep{zhao2026bwla,liu2025bimacosrbinaryonestepdiffusion}, ternarization can replace most floating-point multiplications with additions and simple branching to lower inference costs, while offering stronger representational power and a superior accuracy-efficiency trade-off.
Recent studies, such as TernaryLLM~\citep{chen2024ternaryllmternarizedlargelanguage} and PT$^2$-LLM~\citep{yan2025pt2llmposttrainingternarizationlarge}, mitigate the accuracy loss through distillation or iterative optimization. However, existing ternarization methods primarily focus on weight-only schemes and lack systematic modeling of activation quantization, leading to severe accuracy degradation at low bit-widths (see Fig.~\ref{fig:intro}). To maintain performance, they typically retain activations in full precision and dequantize ternary weights during inference, which fundamentally limits end-to-end acceleration. BitNet v2~\citep{wang2025bitnetv2native4bit} demonstrates joint deployment of ternary weights with 4-bit activations, but it relies on expensive quantization-aware training (QAT), which requires a reported training cost exceeding $10^4$ GPU hours. These limitations highlight an urgent need for a practical post-training quantization (PTQ) approach. It should jointly support ternary weights and low-bit activations,  and deliver efficient end-to-end inference.

We revisit the statistical properties of weights and activations in LLMs, as shown in Fig.~\ref{fig2:vision}(a). Empirically, per-channel weight distributions are often close to a unimodal Gaussian. This shape mismatches the ternary codebook $\{-1, 0, +1\}$, which leads to large quantization error when weights are projected into the ternary space. In contrast, a tri-modal distribution is more aligned with ternary representation, and it should reduce approximation error in principle. Meanwhile, activations exhibit heavy tails and extreme outliers. Under low-bit quantization, this heavy-tailed behavior often dominates the distortion. This contrast motivates our central question. \emph{Can we, within a PTQ framework, \ding{182} reshape weight distributions toward a ternary-friendly tri-modal form, and \ding{183} suppress activation outliers to weaken heavy tails, so that ternary weights and low-bit activations can work well together?}

To this end, we propose \textbf{TWLA} (\textbf{T}ernarized \textbf{W}eights and \textbf{L}ow-bit \textbf{A}ctivations), a PTQ framework for efficient inference, which consists of three tightly coupled modules.
First, we propose \textbf{E2M-ATQ}, a Euclidean-to-Manifold asymmetric ternary quantizer, to improve layer-output calibratability for weight ternarization. 
Second, we introduce \textbf{KOTMS}, a Kronecker-structured orthogonal rotation optimized with a Cayley parameterization, which reshapes weights toward a ternary-friendly tri-modal distribution (see Fig.~\ref{fig2:vision}(b)); by orthogonal equivalence, the same rotation enables shared orthogonal mixing of activations, statistically shrinking outliers and stabilizing the dynamic range for low-bit activation quantization. 
Notably, since the activation benefits arise from a shared rotation, the quantizability gains are inherently heterogeneous across layers; consequently, layers with smaller gains can still become bottlenecks for low-bit activation quantization. To mitigate this effect, we further propose \textbf{ILA-AMP}, an inter-layer aware activation mixed-precision strategy. It introduces a second-order interaction cost between adjacent layers for the first time and unifies per-layer sensitivity with inter-layer coupling induced by distribution shift into a single quadratic surrogate objective. This formulation explicitly incorporates the layer-wise disparity of activation gains caused by the shared orthogonal matrix into the optimization, preventing cascades initiated by a few weak layers. 
Through this coordinated design, TWLA enables stable and efficient deployment of ternary weights with low-bit activations for end-to-end inference.

 To summarize, our main contributions are:
\begin{itemize}
    \item We first identify the key bottlenecks of LLMs under W1.58A4: limited ternary-quantizer calibratability and a ternary codebook--distribution mismatch for weights, and heavy-tailed distributions for activations.
    \item We propose TWLA, a retraining-free PTQ framework that achieves 1.58-bit weights together with low-bit activation quantization (e.g., 4-bit).
    \item Extensive experiments show that TWLA outperforms SOTA 2-bit and sub-2-bit methods, advancing LLMs to the W1.58A4 regime for PTQ.
\end{itemize}

\section{Related Work}
\label{sec:related}

\subsection{Ternarization for Large Language Models.}
Ternarization constrains parameters to $\{-1,0,+1\}$, reducing storage and simplifying arithmetic for efficient deployment. Early work such as TWN~\citep{li2022ternaryweightnetworks} and TTQ~\citep{zhu2017trainedternaryquantization} established scale-aware ternary quantization and learned scaling factors, and later studies extended ternarization to activations for fully ternary networks~\citep{Wang_2018_CVPR,7966166}. More recent efforts scale ternarization to Transformer and LLM settings, often with distillation or end-to-end training, e.g., TernaryBERT~\citep{zhang2020ternarybertdistillationawareultralowbit}, BitNet v2~\citep{wang2025bitnetv2native4bit} and RobuQ~\citep{yang2025robuqpushingditsw158a2}. However, many of these methods rely on the training from scratch or QAT, which is costly and less transferable to arbitrary pretrained models. To reduce overhead, PTQ-based ternarization has been explored, such as PT$^2$-LLM~\citep{yan2025pt2llmposttrainingternarizationlarge}. However, existing PTQ approaches are largely weight-only, typically keeping activations in full precision and dequantizing weights during inference. As a result, jointly suppressing activation outliers and enabling low-bit activations together with ternary weights under PTQ remains an open challenge.

\subsection{Mixed-Precision Quantization.}
Mixed-Precision Quantization (MPQ) assigns different bit-widths to weights and/or activations to improve the accuracy--efficiency trade-off under a fixed budget. 
Although extensively explored in vision, MPQ for LLMs presents unique challenges. Recent works like APTQ~\citep{Guan_2024} extend MPQ to Transformer blocks via Hessian-trace criteria. Under extreme low-bit regimes (e.g., 2-bit), coarse layer/block allocation often becomes insufficient, motivating finer-grained mixed precision and stronger outlier handling. Representative weight-focused methods include SpQR~\citep{dettmers2023spqrsparsequantizedrepresentationnearlossless}, PB-LLM~\citep{shang2023pbllmpartiallybinarizedlarge}, and LLM-MQ~\citep{li2023llm}, as well as TreeQ~\citep{yang2025treeqpushingquantizationboundary} and SliM-LLM~\citep{huang2024slim}. Beyond weight-only designs, ResQ~\citep{saxena2024resq} allocates higher precision to high-variance activation subspaces and applies rotations to suppress activation outliers. Despite these advances, most MPQ methods still assume independent layer-wise sensitivity, while activation quantization can shift distributions and couple errors across layers. Modeling such cross-layer interactions remains an important open problem for MPQ in LLMs.

\begin{figure*}[h!]
    \centering
    \includegraphics[width=\linewidth]{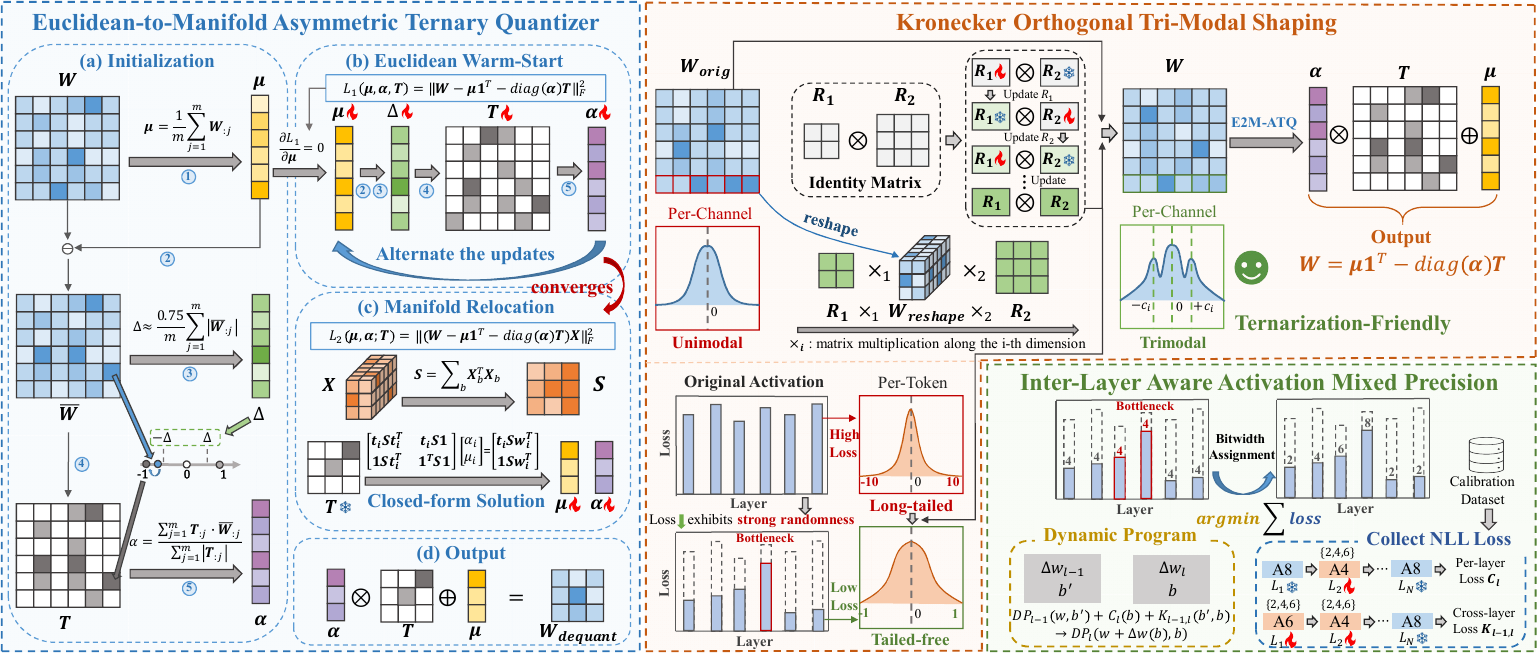}
    \caption{Overview of TWLA (initialization follows the same design as PT$^2$-LLM~\citep{yan2025pt2llmposttrainingternarizationlarge}). \textcolor[HTML]{1d73b6}{Euclidean-to-Manifold Asymmetric Ternary Quantizer (E2M-ATQ)}: minimizes layer-output error under weight ternarization via a two-stage optimization. \textcolor[HTML]{f27830}{Kronecker Orthogonal Tri-Modal Shaping (KOTMS)}: reshapes weight into ternary-friendly tri-modal distribution. \textcolor[HTML]{24a645}{Inter-Layer Aware Activation Mixed Precision (ILA-AMP)}: prevents cascades triggered by weak layers.}
    \label{fig:overview}
\end{figure*}

\section{Method}
\label{sec:method}
\paragraph{Overview.}
Fig.~\ref{fig:overview} illustrates the overall workflow of TWLA. We first review the standard ternarization formulation and basic notations in Sec.~\ref{sec:preliminaries}. Building on this foundation, Sec.~\ref{sec:e2m-atq} introduces E2M-ATQ, a training-free two-stage procedure to minimize layer output error under weight ternarization. Sec.~\ref{sec:kotms} then presents KOTMS, which applies a lightweight Kronecker-structured orthogonal rotation to reshape weights into ternary-friendly tri-modal distributions while suppressing activation outliers through shared rotation. Finally, Sec.~\ref{sec:ila-amp} proposes ILA-AMP, which models second-order interaction costs between adjacent layers to allocate activation bits across layers, thereby preventing accuracy cascades caused by a few weak layers. The complete pseudocode is provided in Appendix~\ref{app:pseudocode}, and weights/activations distribution visualizations are available in Appendix~\ref{app:distribution}.

\subsection{Preliminaries}
\label{sec:preliminaries}
\paragraph{Weight Ternarization.}
Ternarization converts floating-point weights in a model into the ternary set $\{-1,0,+1\}$. 
Specifically, given a full-precision weight matrix $\mathbf{W}\in\mathbb{R}^{n\times m}$, 
TWN~\citep{li2022ternaryweightnetworks} employs a row-wise threshold $\boldsymbol{\Delta}\in\mathbb{R}^{n\times 1}$ 
to map each element $W_{ij}$ to a ternary codebook 
$\mathbf{T}\in\{-1,0,+1\}^{n\times m}$ according to
\begin{equation}
T_{ij}=
\begin{cases}
+1, & W_{ij}>\Delta_i,\\
0,  & |W_{ij}|\le \Delta_i,\\
-1, & W_{ij}<-\Delta_i,
\end{cases}
\quad
\Delta_i \approx \frac{0.75}{m}\sum_{j=1}^{m}\left|W_{ij}\right|.
\label{eq:ternary}
\end{equation}

To improve reconstruction accuracy, TWN introduces a row-wise scaling factor 
$\boldsymbol{\alpha}\in\mathbb{R}^{n\times 1}$ to recover the magnitude of the original weights 
after the ternary matrix $\mathbf{T}$ is determined. 
For the $i$-th row, the optimal scaling factor $\alpha_i$ is obtained by minimizing the 
least-squares reconstruction error:
\begin{equation}
\alpha_i
=\arg\min_{\alpha}\ \left\|\mathbf{W}_{i,:}-\alpha\,\mathbf{T}_{i,:}\right\|_2^2
=\frac{\sum_{j=1}^{m}T_{ij}W_{ij}}{\sum_{j=1}^{m}|T_{ij}|}.
\label{eq:alpha}
\end{equation}

Accordingly, the quantized weight matrix can be expressed as $\hat{\mathbf{W}}=\boldsymbol{\alpha}\mathbf{T}$.
This ternarization scheme provides a practical solution for 
initializing ternary parameters in PTQ settings.

\subsection{Euclidean-to-Manifold Asymmetric Ternary Quantizer}
\label{sec:e2m-atq}
Pretrained LLM weights are often biased with non-zero row-wise means, which violates the common symmetry assumption in low-bit quantization. To capture such bias, we adopt an asymmetric ternary parameterization (e.g., \citep{yan2025pt2llmposttrainingternarizationlarge}) that represents a full-precision weight matrix $\mathbf{W}\in\mathbb{R}^{n\times m}$ as quantized weight matrix $\bar{\mathbf{W}}\in\mathbb{R}^{n\times m}$:
\begin{equation}
\bar{\mathbf{W}}
=
\boldsymbol{\mu}\,\mathbf{1}^{\top}
+
\mathrm{diag}(\boldsymbol{\alpha})\,\mathbf{T},
\qquad
\mathbf{T}\in\{-1,0,1\}^{n\times m},
\label{eq:atq-param}
\end{equation}
where $\boldsymbol{\mu},\boldsymbol{\alpha}\in\mathbb{R}^{n}$ denote the row-wise shift and scale, and $\mathbf{1}\in\mathbb{R}^{m}$ is the all-one vector. We initialize $\boldsymbol{\mu}$ by the row mean $\boldsymbol{\mu}=\frac{1}{m}\sum_{j=1}^{m}\mathbf{W}_{:j}.$
The discrete ternary codebook induces a stratified feasible set: with $\mathbf{T}$ fixed, the quantized weights vary only within the corresponding affine stratum and optimizing the continuous parameters $(\boldsymbol{\mu},\boldsymbol{\alpha})$ often admits closed-form updates, whereas allowing $\mathbf{T}$ to change entails discrete search and combinatorial complexity. However, under output-alignment objectives, calibration-induced inter-column correlations globally couple $\mathbf{T}$, turning an otherwise element-wise thresholding step into a difficult combinatorial optimization. We therefore adopt a two-stage pipeline: we first obtain a stable ternary pattern and a reliable initialization in the Euclidean weight domain, and then freeze $\mathbf{T}$ as a structural prior and relocate $(\boldsymbol{\mu},\boldsymbol{\alpha})$ on the metric manifold defined by the calibration second moment via a metric-consistent closed-form update, thereby aligning the layer output within the selected stratum.

\paragraph{Euclidean warm-start under Frobenius geometry.}
In the first stage, we obtain a stable ternary pattern by minimizing the weight-domain reconstruction error
\begin{equation}
L_1(\boldsymbol{\mu}, \boldsymbol{\alpha}, \mathbf{T})
=
\left\|\mathbf{W}
-
\boldsymbol{\mu}\mathbf{1}^\top
-
\mathrm{diag}(\boldsymbol{\alpha})\,\mathbf{T}
\right\|_F^2.
\label{eq:L1}
\end{equation}
Let $\mathbf{E}=\mathbf{W}-\bar{\mathbf{W}}$ denote the residual error, where the mean of $\mathbf{E}$ is not always zero due to inevitable errors during the ternarization process. Following~\citep{li2024arb}, we apply residual-mean correction:
\begin{equation}
\boldsymbol{\mu}
\leftarrow
\boldsymbol{\mu}
+
\frac{1}{m}\mathbf{E}\mathbf{1}.
\end{equation}
After correcting $\boldsymbol{\mu}$, we sequentially update $\boldsymbol{\alpha}$ and $\mathbf{T}$ according to Eq.~\ref{eq:alpha} and Eq.~\ref{eq:ternary}. This update scheme naturally extends to an iterative algorithm. Specifically, at each iteration we update $\boldsymbol{\mu}$, then $\boldsymbol{\alpha}$, and finally $\mathbf{T}$, so that each variable is optimal under the objective in Eq.~\ref{eq:L1} given the others. (Appendix~\ref{app:mono_decrease} proves the objective decreases monotonically throughout the iterations). Upon convergence, we obtain a stable ternary pattern $\mathbf{T}^{(0)}$ and a Euclidean initialization $(\boldsymbol{\mu}^{(0)},\boldsymbol{\alpha}^{(0)})$ for the subsequent manifold relocation stage.

\paragraph{Manifold relocation under calibration-induced metric.}
The Euclidean reconstruction error in the weight domain may not faithfully reflect forward output deviation. In this stage, we relocate the continuous parameters $(\boldsymbol{\mu},\boldsymbol{\alpha})$ on the metric manifold induced by calibration activations, while freezing the discrete stratum as $\mathbf{T}=\mathbf{T}^{(0)}$.
Given calibration activations $\mathbf{X}$, we minimize the layer output error
\begin{equation}
L_2(\boldsymbol{\mu},\boldsymbol{\alpha};\mathbf{T})
=
\left\|
\left(
\mathbf{W}
-
\boldsymbol{\mu}\mathbf{1}^\top
-
\mathrm{diag}(\boldsymbol{\alpha})\,\mathbf{T}
\right)\mathbf{X}
\right\|_F^2 .
\end{equation}
To avoid repeatedly multiplying high-dimensional $\mathbf{X}$, we precompute the activation second moment~\citep{li2024arb}
$\mathbf{S}=\sum_b \mathbf{X}_b^\top\mathbf{X}_b \in \mathbb{R}^{m\times m}$, which allows us to rewrite the $L_2$ reconstruction error
$L_2=\|\mathbf{W}\mathbf{X}-\mathbf{W}_c\mathbf{X}\|_F^2$ as the following equivalent quadratic form:
\begin{equation}
L_2(\boldsymbol{\mu},\boldsymbol{\alpha};\mathbf{T})
=
\mathrm{Tr}\!\left(\mathbf{E}\mathbf{S}\mathbf{E}^\top\right),
\mathbf{E}
=
\mathbf{W}
-
\boldsymbol{\mu}\mathbf{1}^\top
-
\mathrm{diag}(\boldsymbol{\alpha})\,\mathbf{T}.
\label{eq:L2_main}
\end{equation}
Equivalently, minimizing Eq.~\ref{eq:L2_main} amounts to projecting $\mathbf{W}$ onto the affine constraint set induced by a fixed $\mathbf{T}$. Since $\mathbf{S}$ is typically dense, optimizing over $\mathbf{T}$ would introduce global coupling and yield a challenging combinatorial problem; we thus fix $\mathbf{T}=\mathbf{T}^{(0)}$ and optimize only $(\boldsymbol{\mu},\boldsymbol{\alpha})$ for relocation. Under this constraint, the objective decouples across rows: for each row $i$ with $(\mathbf{w}_i,\mathbf{t}_i)$, imposing $\partial L_2/\partial \mu_i=0$ and $\partial L_2/\partial \alpha_i=0$ gives a $2\times2$ linear system.

\begin{equation}
\begin{bmatrix}
\mathbf{t}_i \mathbf{S}\mathbf{t}_i^\top & \mathbf{t}_i \mathbf{S}\mathbf{1} \\
\mathbf{1}^\top \mathbf{S}\mathbf{t}_i^\top & \mathbf{1}^\top \mathbf{S}\mathbf{1}
\end{bmatrix}
\begin{bmatrix}
\alpha_i \\
\mu_i
\end{bmatrix}
=
\begin{bmatrix}
\mathbf{t}_i \mathbf{S}\mathbf{w}_i^\top \\
\mathbf{1}^\top \mathbf{S}\mathbf{w}_i^\top
\end{bmatrix}.
\label{eq:2x2_main}
\end{equation}
The unique closed-form solution is
\begin{equation}
\begin{split}
\alpha_i^{*}
=
\frac{
(\mathbf{t}_i \mathbf{S}\mathbf{w}_i^\top)(\mathbf{1}^\top \mathbf{S}\mathbf{1})
-
(\mathbf{t}_i \mathbf{S}\mathbf{1})(\mathbf{1}^\top \mathbf{S}\mathbf{w}_i^\top)
}{(\mathbf{t}_i \mathbf{S}\mathbf{t}_i^\top)(\mathbf{1}^\top \mathbf{S}\mathbf{1})
-
(\mathbf{t}_i \mathbf{S}\mathbf{1})(\mathbf{1}^\top \mathbf{S}\mathbf{t}_i^\top)},
\\[0.6em]
\mu_i^{*}
=
\frac{
(\mathbf{t}_i \mathbf{S}\mathbf{t}_i^\top)(\mathbf{1}^\top \mathbf{S}\mathbf{w}_i^\top)
-
(\mathbf{1}^\top \mathbf{S}\mathbf{t}_i^\top)(\mathbf{t}_i \mathbf{S}\mathbf{w}_i^\top)
}{(\mathbf{t}_i \mathbf{S}\mathbf{t}_i^\top)(\mathbf{1}^\top \mathbf{S}\mathbf{1})
-
(\mathbf{t}_i \mathbf{S}\mathbf{1})(\mathbf{1}^\top \mathbf{S}\mathbf{t}_i^\top)}.
\end{split}
\end{equation}
The full derivation is provided in Appendix~\ref{app:e2m_rowwise_normal_eq}. After the two-stage Euclidean-to-Manifold (E2M) procedure, the resulting quantized weights under E2M-ATQ are
\begin{equation}
\bar{\mathbf{W}}
=
\boldsymbol{\mu}^{*}\mathbf{1}^\top
+
\mathrm{diag}(\boldsymbol{\alpha}^{*})\,\mathbf{T}^{(0)}.
\end{equation}

\subsection{Kronecker Orthogonal Tri-Modal Shaping}
\label{sec:kotms}

In the previous section, E2M-ATQ improves the continuous shift and scale parameters for a fixed ternary pattern. However, such parameter relocation does not change the geometric mismatch between pretrained weights and the ternary codebook. The asymmetric ternary levels $\{-\alpha_i+\mu_i,\ \mu_i,\ +\alpha_i+\mu_i\}$ are most effective when each channel contains three well-separated attraction regions, while pretrained LLM weights are usually concentrated around a near-unimodal shape~\citep{ye2025dbellquant}. Consequently, many entries remain close to ternary decision boundaries, making the final hard projection sensitive to small perturbations and increasing the ternarization error.

To reduce this mismatch before hard ternary projection, KOTMS introduces a structured orthogonal coordinate transformation and optimizes it with a ternary-codebook shaping loss. The orthogonal transformation preserves the full-precision mapping through its inverse rotation, while the shaping loss encourages the transformed weights to become more compatible with the asymmetric ternary levels used by E2M-ATQ. In this way, KOTMS complements E2M-ATQ: E2M-ATQ calibrates the ternary parameters within a selected discrete pattern, whereas KOTMS improves the coordinate system in which the ternary projection is performed.

\begin{table*}[h!]
\centering
\caption{Comparison of perplexity on WikiText2 and averaged accuracy on seven Zero-Shot tasks (Arc-Challenge, Arc-Easy , HellaSwag , LAMBADA-openai, LAMBADA-standard, PIQA, and WinoGrande). MP indicates that mixed-precision is adopted for weights and/or activations. Full results are in the Appendix~\ref{app:detailed main results}}
\resizebox{\textwidth}{!}{%
\tablestyle{2pt}{1.2}
\begin{tabular}{l|c|c|*{2}{cc:}cc|cc|cc:cc:cc}
\toprule
\multirow{3}{*}{\textbf{Method}} & \multirow{3}{*}{\makecell{\textbf{\#Bits}\\ (W)}} & \multirow{3}{*}{\makecell{\textbf{\#Bits}\\ (A)}}    &
\multicolumn{2}{c}{\textbf{LLaMA2-7B}} &
\multicolumn{2}{c}{\textbf{LLaMA2-13B}} &
\multicolumn{2}{c}{\textbf{LLaMA2-70B}} &
\multicolumn{2}{c}{\textbf{LLaMA3-8B}} &
\multicolumn{2}{c}{\textbf{Qwen3-8B}} &
\multicolumn{2}{c}{\textbf{Qwen3-14B}} &
\multicolumn{2}{c}{\textbf{Qwen3-32B}} \\
\cdashline{4-17}
 & & & 0-shot$^7$ & Wiki & 0-shot$^7$ & Wiki & 0-shot$^7$ & Wiki & 0-shot$^7$ & Wiki & 0-shot$^7$ & Wiki & 0-shot$^7$ & Wiki & 0-shot$^7$ & Wiki \\
 & & & Avg.($\uparrow$) &($\downarrow$) & Avg.($\uparrow$) &($\downarrow$) & Avg.($\uparrow$) &($\downarrow$) & Avg.($\uparrow$) &($\downarrow$) & Avg.($\uparrow$) &($\downarrow$) & Avg.($\uparrow$) &($\downarrow$) & Avg.($\uparrow$) &($\downarrow$)\\
\hdashline
FP16       & 16   & 16 & 69.49 & 5.47
& 72.19 & 4.88
& 76.71 & 3.32
& 72.51 & 6.14
& 69.09 & 9.00
& 72.47 & 8.64
& 72.42 & 7.61 \\
\hdashline
GPTQ        & 2    & \multirow{6}{*}{16} & 29.09 & 47.13
& 27.21 & 182.11
& 45.59 & 24.53
& 26.38 & 154.38
& 29.25 & 61.74
& 28.61 & 84.25
& 38.18 & 32.95 \\
QuaRot   & 2    &
& 36.30 & 19.97
& 49.30 & 10.36
& 68.79 & 5.56
& 35.66 & 23.05
& -- & --
& -- & --
& -- & -- \\
SliM-LLM & 2MP  & 
& 48.96 & 16.03
& 52.00 & 9.06
& -- & --
& 27.49 & 38.09
& 39.13 & 27.10
& 48.10 & 15.19
& 64.79 & 12.12 \\
PB-LLM   & 1.7  & 
& 29.01 & 41.04
& 26.29 & 335.22
& 50.88 & 12.02
& 29.32 & 42.59
& 38.43 & 26.20
& 44.80 & 22.90
& 63.92 & 13.17 \\
PT$^2$-LLM & 1.58 & 
& 42.82 & 11.56
& 56.54 & 9.19
& 62.47 & 6.27
& 39.04 & 32.19
& -- & --
& 46.35 & 16.48
& -- & -- \\
\rowcolor{gray!10}
\textbf{TWLA} & 1.58 &  
& \textbf{62.91} & \textbf{6.97}
& \textbf{67.70} & \textbf{5.79}
& \textbf{73.60} & \textbf{4.13}
& \textbf{62.98} & \textbf{9.39}
& \textbf{62.05} & \textbf{12.52}
& \textbf{68.48} & \textbf{10.42}
& \textbf{69.54} & \textbf{8.94} \\
\hdashline
GPTQ     & 2    & 6
& 27.94 & 120.21
& 25.90 & 2e3
& 35.89 & 287.83
& 25.69 & 386.19
& 26.65 & 125.31
& 26.76 & 200.91
& 31.15 & 125.47 \\
QuaRot   & 2    & 6
& 35.22 & 21.25
& 47.95 & 10.62
& 67.97 & 5.63
& 35.51 & 24.37
& -- & --
& -- & --
& -- & -- \\
ResQ     & 2.3MP & 6.1MP
& 51.18 & 9.93
& 56.49 & 8.62
& 57.81 & 5.38
& 47.57 & 16.93
& -- & --
& -- & --
& -- & -- \\
SliM-LLM & 2MP  & 6
& 37.06 & 565.93
& 43.32 & 11.32
& -- & --
& 27.15 & 233.82
& 31.41 & 67.61
& 37.46 & 28.91
& 52.43 & 20.88 \\
PB-LLM   & 1.7  & 6
& 30.60 & 44.36
& 26.61 & 767.99
& 40.89 & 23.66
& 30.28 & 55.75
& 31.67 & 66.83
& 37.28 & 33.15
& 51.45 & 21.05 \\
PT$^2$-LLM & 1.58 & 6
& 38.61 & 22.12
& 50.82 & 16.33
& 48.96 & 13.78
& 30.81 & 69.83
& -- & --
& 36.00 & 30.90
& -- & -- \\
\rowcolor{gray!10}
\textbf{TWLA} & 1.58 & 6MP
& \textbf{61.98} & \textbf{7.14}
& \textbf{66.88} & \textbf{5.90}
& \textbf{73.45} & \textbf{4.21}
& \textbf{61.44} & \textbf{9.76}
& \textbf{59.03} & \textbf{12.83}
& \textbf{67.29} & \textbf{10.69}
& \textbf{68.13} & \textbf{9.00} \\
\hdashline
GPTQ     & 2    & 4
& 25.80 & 2e4
& 25.35 & 7e3
& 25.92 & 2e4
& 25.64 & 2e4
& 25.29 & 1e4
& 25.79 & 3e4
& 24.85 & 2e7 \\
QuaRot   & 2    & 4
& 30.83 & 37.79
& 39.42 & 14.49
& 53.64 & 6.77
& 27.69 & 119.90
& -- & --
& -- & --
& -- & -- \\
ResQ     & 2.3MP & 4.2MP
& 45.23 & 11.97
& 51.84 & 9.29
& 56.43 & 7.92
& 39.12 & 21.95
& -- & --
& -- & --
& -- & -- \\
SliM-LLM & 2MP  & 4
& 25.67 & 4e3
& 26.44 & 1e3
& -- & --
& 25.60 & 1e4
& 19.18 & 2e4
& 25.62 & 1e4
& 26.31 & 6e3 \\
PB-LLM   & 1.7  & 4
& 26.46 & 458.08
& 25.54 & 3e3
& 26.18 & 2e3
& 26.29 & 684.28
& 26.26 & 5e3
& 29.23 & 914.41
& 26.36 & 1e4 \\
PT$^2$-LLM & 1.58 & 4
& 27.31 & 341.08
& 29.07 & 2e3
& 26.19 & 1e3
& 25.81 & 720.21
& -- & --
& 25.53 & 5e3
& -- & -- \\
\rowcolor{gray!10}
\textbf{TWLA} & 1.58 & 4MP
& \textbf{58.00} & \textbf{8.31}
& \textbf{64.30} & \textbf{6.68}
& \textbf{71.10} & \textbf{4.77}
& \textbf{55.23} & \textbf{12.83}
& \textbf{50.42} & \textbf{16.15}
& \textbf{62.00} & \textbf{12.84}
& \textbf{65.25} & \textbf{9.71} \\ 
\bottomrule
\end{tabular}}
\label{tab:main}
\end{table*}

\paragraph{Codebook-Aligned Shaping Objective.}
We use a symmetric three-component Gaussian mixture as a smooth surrogate for ternary-codebook alignment. The three modes serve as adaptive row-wise anchors corresponding to the negative, zero, and positive ternary regions. For the $i$-th row $\mathbf{w}_i$, let $\mathbf{R}$ denote a learnable orthogonal transform and define
\begin{equation}
    \mathbf{z}_i = \mathbf{w}_i\mathbf{R}, 
    \qquad 
    z_{ij}=(\mathbf{z}_i)_j .
\end{equation}
We assign three anchors $\{-c_i,0,+c_i\}$ to each row and minimize the following negative log-likelihood:
\begin{equation}
\begin{split}
\mathcal{L}_{\mathrm{TriGMM}}
=
-\frac{1}{nm}
\sum_{i=1}^{n}
\sum_{j=1}^{m}
\log
\Big[
&\pi_+\phi(z_{ij};+c_i,\sigma_i^2)\\
+\pi_0\phi(z_{ij};0,\sigma_i^2)
&+\pi_-\phi(z_{ij};-c_i,\sigma_i^2)
\Big],
\end{split}
\label{eq:loss}
\end{equation}
where $\phi(\cdot;c,\sigma^2)$ denotes the Gaussian density, $c_i=\frac{1}{m}\sum_{j=1}^{m}|z_{ij}|$, $\sigma_i=\mathrm{std}(\{z_{ij}\}_{j=1}^{m})$, and $\pi_+=\pi_-=\frac{1-\pi_0}{2}$ with $\pi_0\in(0,1)$. This objective gives higher likelihood to values close to the ternary-aligned anchors and lower likelihood to values far from them. Therefore, minimizing $\mathcal{L}_{\mathrm{TriGMM}}$ provides a differentiable approximation to moving transformed entries toward $\{-c_i,0,+c_i\}$, which is consistent with the subsequent hard ternary projection $T_{ij}\leftarrow\Pi_{\{-1,0,1\}}(\cdot)$.

To avoid degenerate solutions, such as $c_i\!\to\!0$ or excessive concentration around the zero component, we further regularize the zero-mode responsibility:
\begin{equation}
    \mathcal{L}_{\mathrm{shape}}
    =
    \mathcal{L}_{\mathrm{TriGMM}}
    +
    \beta \mathcal{L}_{\mathrm{zero}} ,
\end{equation}
where $\mathcal{L}_{\mathrm{zero}}$ controls the mass assigned to the central mode. Details of this regularizer and the soft-projection interpretation are provided in Appendix~\ref{app:kotms-zero-constraint} and Appendix~\ref{app:kotms-soft-proj}. The resulting objective supplies the optimization signal for learning the orthogonal transform used in KOTMS.

\paragraph{Kronecker-Structured Orthogonal Transform.}
A dense orthogonal matrix $\mathbf{R}\in\mathbb{R}^{m\times m}$ is impractical for LLM layers because it requires large storage and expensive matrix multiplication. KOTMS therefore restricts the transform to a Kronecker-structured orthogonal family~\citep{gu2026loproenhancinglowrankquantization,xiao2025singlequantefficientquantizationlarge}:
\begin{equation}
\mathbf{R}
=
\mathbf{R}_1\otimes\mathbf{R}_2,
\quad
\mathbf{R}_1\in\mathcal{O}(n_1),\ 
\mathbf{R}_2\in\mathcal{O}(n_2),\ 
n_1n_2=m .
\end{equation}
This parameterization keeps the transformation exactly invertible: $\mathbf{R}^{-1}=\mathbf{R}^{\top}=\mathbf{R}_1^{\top}\otimes\mathbf{R}_2^{\top}.$ Thus, the inverse rotation can be consistently applied to the activation side, preserving the full-precision function while allowing the weight coordinates to be shaped for ternary projection.

For implementation, a row vector $\mathbf{v}\in\mathbb{R}^{1\times m}$ is reshaped into $\mathbf{V}_{\mathrm{mat}}\in\mathbb{R}^{n_1\times n_2}$. The product with the Kronecker transform can then be evaluated as
\begin{equation}
\mathbf{v}\mathbf{R}
=
\operatorname{vec}\!\left(
\mathbf{R}_2^{\top}
\mathbf{V}_{\mathrm{mat}}
\mathbf{R}_1
\right)^{\top}.
\label{eq:kron}
\end{equation}
As a result, KOTMS stores and applies two small orthogonal factors instead of a full dense rotation matrix. When $n_1$ and $n_2$ are chosen with comparable sizes, the application is dominated by two compact matrix multiplications. This makes the transform suitable for PTQ calibration and for deployment-time activation rotation with limited overhead. In addition, the shared orthogonal mixing disperses concentrated activation directions and helps reduce heavy-tailed activation outliers, which facilitates low-bit activation quantization; empirical evidence is provided in Appendix~\ref{app:act_outlier}.

\paragraph{Orthogonality-Preserving Optimization.}
We optimize the Kronecker factors under the shaping objective in Eq.~\eqref{eq:loss}. To maintain strict orthogonality during optimization, each factor is parameterized by a Cayley transform. Specifically, for $k\in\{1,2\}$, we introduce a free matrix $\mathbf{S}_k\in\mathbb{R}^{n_k\times n_k}$ and construct a skew-symmetric generator
\begin{equation}
    \mathbf{A}_k=\mathbf{S}_k-\mathbf{S}_k^{\top}.
\end{equation}
The corresponding orthogonal factor is defined as
\begin{equation}
\mathbf{R}_k
=
\mathrm{cayley}(\mathbf{A}_k)
=
(\mathbf{I}+\mathbf{A}_k)^{-1}
(\mathbf{I}-\mathbf{A}_k),
\qquad k\in\{1,2\}.
\end{equation}
This guarantees $\mathbf{R}_k^{\top}\mathbf{R}_k=\mathbf{I}$ throughout optimization and avoids explicit projection steps. During calibration, gradients of $\mathcal{L}_{\mathrm{shape}}$ update the free matrices $\mathbf{S}_1$ and $\mathbf{S}_2$, and the resulting orthogonal factors define the final KOTMS rotation.

\subsection{Inter-Layer Aware Activation Mixed Precision}
\label{sec:ila-amp}
\paragraph{Discussion.}
Although KOTMS can statistically mitigate activation outliers, it is important to note that its learning objective is defined purely in the weight domain. Consequently, the improvement in activation quantizability is largely an indirect by-product of applying the same orthogonal mixing, and thus can be highly uneven across layers (See Appendix~\ref{app:vis_kotms_cross_layer} for a detailed distributional visualization.). As a result, a small number of less-benefited layers may become the primary bottleneck for low-bit activation quantization. Motivated by this observation, we allocate activation bitwidths across layers under a global budget constraint.

Consider a network with $L$ layers. Let $b_\ell$ denote the activation bitwidth assigned to layer $\ell$, with candidates $\mathcal{B}=\{2,4,6,8\}$, and let $\mathbf{b}=(b_1,\ldots,b_L)$ denote a layer-wise configuration. We adopt the per-token negative log-likelihood (NLL) on a validation set as the value function:
\begin{equation}
v_{\mathrm{NLL}}(\mathbf{b})
=
\mathbb{E}_{(\mathbf{x},t)\sim\mathcal{D}}
\Big[
-\log p_{\boldsymbol{\theta}}(x_{t+1}\mid \mathbf{x}_{\le t}; \mathbf{b})
\Big].
\end{equation}
Here, $\mathcal{D}$ is the validation set and $p_{\boldsymbol{\theta}}(\cdot;\mathbf{b})$ is the activation-quantized model under assignment $\mathbf{b}$. Most MPQ methods assume layer-wise independent sensitivity, i.e., quantization effects are approximately additive across layers. This is often invalid for LLMs: quantizing layer $\ell$ changes its output distribution, shifts the input statistics of layer $\ell!+!1$, and perturbs the downstream quantizer. These shifts accumulate and amplify, inducing cross-layer coupling that is especially severe at 2--4 bits and can trigger accuracy collapse. We therefore add explicit adjacent-layer interaction terms to model local error propagation along the stack.

Directly minimizing $v_{\mathrm{NLL}}(\mathbf{b})$ over $\mathcal{B}^L$ is intractable. We instead build a DP-friendly second-order surrogate that approximates the NLL change via (i) per-layer costs $C_\ell(b)$ and (ii) adjacent interaction costs $K_{\ell-1,\ell}(b',b)$. We use $b_{\max}=8$ as a reference: start from uniform 8-bit quantization and then lower selected layers around this baseline, yielding controlled perturbations and lower-variance sensitivity estimates. The first-order cost is defined as
\begin{equation}
C_\ell(b)=v_{\mathrm{NLL}}(b_\ell=b,\; b_{k\neq \ell}=b_{\max})-v_{\mathrm{NLL}}(\mathbf{b}_{\max}),
\end{equation}
and the adjacent interaction cost as
\begin{equation} 
\resizebox{\hsize}{!}{$ 
\begin{split} 
K_{\ell-1,\ell}(b', b) = v_{\mathrm{NLL}}\!\Big(b_{\ell-1}=b',\; b_\ell=b,\; b_{k\notin\{\ell-1,\ell\}}=b_{\max}\Big)\\ -
v_{\mathrm{NLL}}(\mathbf{b}_{\max}) - C_{\ell-1}(b') - C_\ell(b), 
\end{split} 
$} 
\end{equation}
When $K_{\ell-1,\ell}(b',b)>0$, it indicates coupling amplification caused by error propagation between adjacent layers. We provide further justification for modeling interactions only between adjacent layers in Appendix~\ref{app:adjacent_only}.

Then we formulate mixed-precision activation allocation as a budget-constrained minimization problem with adjacent second-order terms:
\begin{equation}
\min_{\{b_\ell\}}
\sum_{\ell=1}^L C_\ell(b_\ell)
+
\sum_{\ell=2}^L K_{\ell-1,\ell}(b_{\ell-1}, b_\ell),
\text{s.t. }
\sum_{\ell=1}^L b_\ell \le B,
\end{equation}
where $B$ is the global bit budget. 
Since $K_{\ell-1,\ell}$ is defined only on adjacent layer pairs, the surrogate objective has a chain structure.
We therefore solve the global bit allocation exactly via dynamic programming under the budget constraint $c\le B$,
and recover the optimal assignment $\{b_\ell^*\}_{\ell=1}^L$ by backtracking. Full details are deferred to Appendix~\ref{app:dp}.

\section{Experiments}
\label{sec:Experiments}
\begin{table}[t]
\centering
\caption{Results on Qwen3-32B-Instruct across three challenging benchmarks: MMLU, HumanEval, and GSM8K. Best and second-best results are marked in \textbf{bold} and \underline{underlined}, respectively.}
\resizebox{0.48\textwidth}{!}{%
\tablestyle{2pt}{1.2}
\begin{tabular}{clccccc}
\toprule
\textbf{Model} & \textbf{Method} & \textbf{\#Bits(W)} & \textbf{\#Bits(A)} & \textbf{MMLU} & \textbf{HumanEval} & \textbf{GSM8K} \\
\midrule
\multirow{11}{*}{\makecell{Qwen3-32B \\ -Instruct}}
& FP16    & 16   & 16 & 80.75 & 50.41 & 66.55 \\
\cline{2-7}
& GPTQ    & 3    & \multirow{5}{*}{16} & \underline{73.97} & \underline{41.00} & \underline{53.33} \\
& GPTQ    & 2    &  & 27.10 &  0 &  1.23 \\
& SliM-LLM   & 2MP &  & 55.38 & 23.74 & 35.21 \\
& PB-LLM & 1.7 &  & 48.21 & 12.20 & 24.33 \\
& \textbf{TWLA}    & 1.58 &  &\textbf{75.59} &\textbf{42.99} &\textbf{55.72} \\
\cline{2-7}
& GPTQ    & 3    &  4 & 23.28 &  0 &  1.02 \\
& GPTQ    & 2    & 4  & 23.20 &  0 &  0.13 \\
& SliM-LLM   & 2MP &  4 & 23.71 &  0 &  0.97 \\
& PB-LLM & 1.7 &  4 & 23.12 &  0 &  1.33 \\
& \textbf{TWLA}    & 1.58 & 4MP   & \textbf{70.21} & \textbf{37.58} & \textbf{48.67} \\
\bottomrule
\end{tabular}}
\label{tab:qwen32b_instruct}
\end{table}

\subsection{Experiment setup}
\paragraph{Models and Datasets.}
We conduct experiments on the LLaMA families~\citep{touvron2023llamaopenefficientfoundation} and the Qwen3 family~\citep{qwen3technicalreport}. In addition, we evaluate instruction-tuned variants to further demonstrate the effectiveness of our method. 
Beyond standard perplexity evaluation on WikiText2~\citep{wikitext2} and C4~\citep{c4}, we assess TWLA on a broad set of zero-shot tasks, including ARC-Challenge and ARC-Easy~\citep{clark2018thinksolvedquestionanswering}, HellaSwag~\citep{zellers2019hellaswagmachinereallyfinish}, LAMBADA-openai and LAMBADA-standard~\citep{paperno2016lambadadatasetwordprediction}, PIQA~\citep{bisk2019piqareasoningphysicalcommonsense}, and WinoGrande~\citep{sakaguchi2019winograndeadversarialwinogradschema}. 
We further evaluate TWLA on more challenging reasoning benchmarks, including the multi-domain knowledge task MMLU~\citep{mmlu}, the mathematical reasoning benchmark GSM8K~\citep{gsm8k}, and the code generation benchmark HumanEval~\citep{humaneval}.

\paragraph{Baseline Methods.} 
We compare TWLA with diverse PTQ baselines in the 2-bit and sub-2-bit regimes, covering both weight-only and weight--activation quantization. For weight-only PTQ, we include SliM-LLM~\citep{huang2024slim} as a strong mixed-precision baseline, PB-LLM~\citep{shang2023pbllmpartiallybinarizedlarge} for sub-2-bit quantization (with an average weight bit-width close to ours), and PT$^{2}$-LLM~\citep{yan2025pt2llmposttrainingternarizationlarge} as a representative ternarization method with the same average weight bit-width as TWLA. For weight--activation PTQ, we benchmark ResQ~\citep{saxena2024resq}, which applies mixed-precision activation quantization with an average of about 4.25 bits. We further report GPTQ~\citep{frantar2023gptqaccurateposttrainingquantization} and QuaRot~\citep{ashkboos2024quarot} as widely adopted PTQ baselines for completeness.

\paragraph{Implementation Details.} 
All experiments are conducted on NVIDIA A6000 GPUs. For E2M-ATQ, we perform 15 iterations to ensure the convergence of the ternarization parameters. Following~\citep{zhao2026bwla}, we select 128 calibration samples from WikiText2~\citep{wikitext2}, each with a sequence length of 2048. Based on this sample set, we optimize the parameters in KOTMS for 100 iterations with a fixed learning rate of 0.01, and the same sample set is also employed as the calibration data for the ILA-AMP.

\begin{table}[t]
\centering
\caption{Impact of different components in TWLA.}
\label{tab:method_ablation}
\resizebox{\linewidth}{!}{
\tablestyle{2pt}{1.2}
\begin{tabular}{c |c:c:c| cc|cc}
\toprule
\multirow{2}{*}{\textbf{\makecell{\#Bits \\(A)}}} &\multirow{2}{*}{\textbf{\makecell{E2M- \\ ATQ}}} & \multirow{2}{*}{\textbf{KOTMS}} &\multirow{2}{*}{\textbf{\makecell{ILA- \\ AMP}}}
&\multicolumn{2}{c}{\textbf{LLaMA2-13B}} &\multicolumn{2}{c}{\textbf{Qwen3-14B}} \\

& & & & C4 ($\downarrow$) & MMLU ($\uparrow$) & C4 ($\downarrow$) & MMLU ($\uparrow$) \\
\midrule
\multirow{4}{*}{16}
 & $\times$       & $\times$  & -     & 6e3 &23.01 &1e4 &23.22 \\
 & \checkmark     & $\times$  & -     & 18.32 &30.15 &23.64 &57.64 \\
 & $\times$       & \checkmark & -     & 57.98 &24.12 &102.12 &25.02 \\
 & \checkmark     & \checkmark   & -  & \textbf{8.64} & \textbf{44.86} &
                                  \textbf{17.26} & \textbf{69.15}  \\
\midrule
\multirow{4}{*}{4}
 & \checkmark       & $\times$  & $\times$     & 3e3 &23.76 &2e3 &24.12 \\
 & \checkmark     & \checkmark  & $\times$     & 25.03 &27.52 &31.78 &47.60 \\
 & \checkmark       & $\times$ & \checkmark     & 55.86 &24.01 &390.95 &25.33 \\
 & \checkmark     & \checkmark   & \checkmark  & \textbf{10.07} & \textbf{38.17} &
                                  \textbf{21.00} & \textbf{60.82}  \\
\bottomrule
\end{tabular}}
\end{table}

\begin{figure*}[t!]
    \centering
    \includegraphics[width=0.95\linewidth]{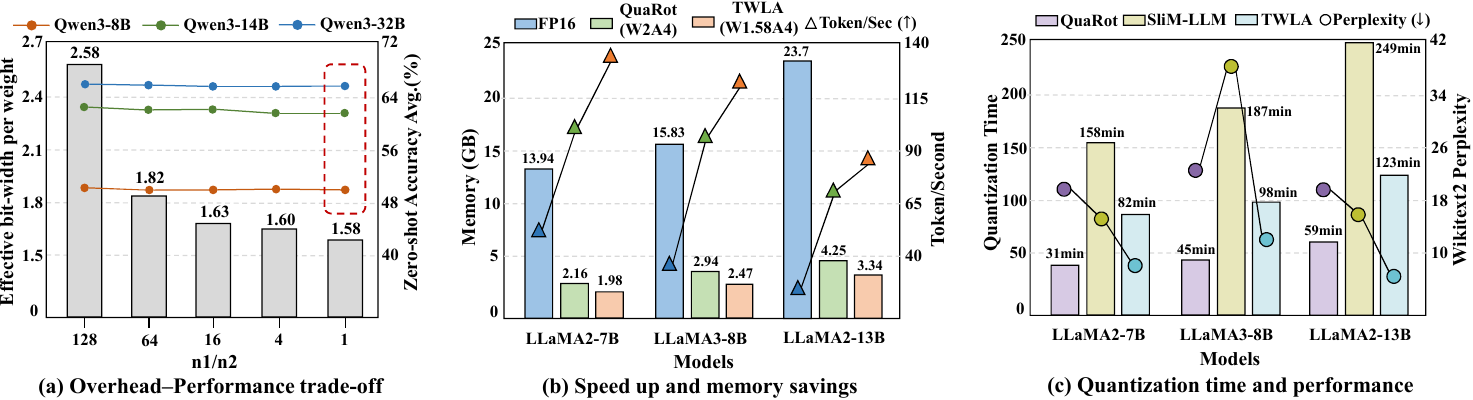}
    \caption{(a) Ablation of the Overhead–Performance trade-off for KOTMS. (b) Comparison of throughput and memory consumption across FP16, QuaRot, and TWLA. (c) Comparison of quantization time and perplexity across QuaRot, SliM-LLM, and TWLA.}
    \label{fig:overhead}
    \vspace{-0.3cm}
\end{figure*}

\subsection{Main Results}

\paragraph{Comparison Results.}
We systematically evaluate ternarization across the LLaMA and Qwen3 families under different activation bit-widths, where all non-ternary PTQ baselines use 2-bit weights. As shown in Table~\ref{tab:main}, TWLA achieves the best performance consistently across model scales. Under A16 setting, TWLA operates at the lowest average weight precision (1.58 bits) yet attains both the lowest WikiText2 perplexity and the highest average accuracy, clearly outperforming 2-bit baselines such as GPTQ, QuaRot, and SliM-LLM, and delivering substantial gains over the current SOTA ternarization method PT$^2$-LLM (e.g., on LLaMA-3-8B, average accuracy improves from 39.04 to 62.98 while WikiText2 PPL drops by 70\%). Under the more challenging A4 setting, prior methods generally suffer perplexity explosion, whereas TWLA remains stable and further surpasses ResQ, which also employs mixed-precision activations; on LLaMA-2-70B, TWLA raises average accuracy from 56.43 to 71.10 (exceeding 92\% of FP16) while reducing memory consumption by over 30\% relative to ResQ. Full expanded results are provided in Appendix~\ref{app:detailed main results}.

\paragraph{Experiments of Instruction-tuned Models.}
Instruction tuning substantially improves the practical utility of LLMs, but it also makes PTQ considerably more challenging. We evaluate Qwen3-32B-Instruct on three challenging benchmarks (Table \ref{tab:qwen32b_instruct}). With FP16 activations, TWLA retains roughly 85\% of FP16 performance, recovering most of the reasoning capability of the model. Notably, TWLA even outperforms 3-bit GPTQ while using roughly half the memory, yielding a favorable accuracy--efficiency trade-off.
When activations are further quantized to 4-bit, existing methods largely collapse: MMLU accuracy drops to near random-guessing level (around 25\%), while HumanEval and GSM8K fall to almost zero. In contrast, TWLA preserves about 90\% of its FP16-activation performance, indicating strong robustness and high quantization quality in the low-bit activation regime.

\subsection{Ablation Studies}
TWLA consists of three modules (E2M-ATQ, KOTMS, and ILA-AMP) designed to reduce quantization error. Our ablations include (i) evaluating the individual contribution of each module and (ii) ablating the orthogonal matrix size in KOTMS to examine the overhead–performance trade-off. We provide calibration-data ablations in Appendix~\ref{app:dataset} (seed, size, and type), and report additional ILA-AMP analyses on sensitivity metrics, layer-wise bit allocations, and interaction-order cost--benefit trade-offs in Appendix~\ref{app:amp}.

\paragraph{Modular Sensitivity Study.}
We report C4 perplexity and MMLU accuracy to isolate the effects of each component (Table~\ref{tab:method_ablation}). Without activation quantization, either E2M-ATQ or KOTMS alone improves both metrics, with larger gains from E2M-ATQ; combining them performs best (e.g., Qwen3-14B gains $\sim$45\% MMLU). Under 4-bit activations, using only E2M-ATQ+KOTMS degrades accuracy, whereas adding ILA-AMP recovers a further $\sim$12\% MMLU on LLaMA2-13B and Qwen3-14B, bringing results close to the no-activation-quantization setting. These results highlight the strong synergy among E2M-ATQ, KOTMS, and ILA-AMP, and underscore the necessity of using them together.


\paragraph{Overhead–Performance Trade-off.}
We ablate the orthogonal matrix size in KOTMS to quantify the overhead--performance trade-off (Fig.~\ref{fig:overhead}(a)). By varying the Kronecker dimensions $n_1$ and $n_2$ to control the structure of the orthogonal transform, we find that making them more balanced (smaller $n_1/n_2$) sharply reduces KOTMS’s effective weight-bit overhead with only minor accuracy loss. On the Qwen3 family, higher symmetry reduces the OKT overhead from roughly 1 bit to a negligible level ($<$0.01 bit), with the average accuracy dropping by $<2\%$. We therefore choose $n_1$ and $n_2$ to be as close as possible in all experiments.

\subsection{Efficiency Analysis of TWLA}
\paragraph{Speedup and Memory Savings.} 
We evaluate the inference efficiency of TWLA on the LLaMA family with \texttt{llama.cpp}\footnote{\url{https://github.com/ggml-org/llama.cpp}}
on NVIDIA RTX A6000 GPUs, leveraging its native ternary operators for optimized low-bit kernels. We compare FP16, QuaRot (W2A4), and TWLA (W1.58A4) in throughput and memory. Under batch size 4 with 128-token prefill and 256-token decoding, Fig.~\ref{fig:overview}(b) shows that ternary weights further reduce memory and improve speed beyond 2-bit weights. On LLaMA2-13B, QuaRot improves throughput from 23.70 tokens/s (FP16) to 66.42 tokens/s, while TWLA further increases it to 86.35 tokens/s, yielding a 3.64$\times$ speedup over FP16 and a 1.3$\times$ speedup over QuaRot. Meanwhile, the parameter memory is reduced from 23.7GB to 3.34GB by packing five ternary values into one 8-bit integer (effective $\sim$1.6 bits/weight), saving $>$ 80\% of the storage. These results highlight the substantial efficiency gains brought by TWLA.

\vspace{-0.1cm}
 
\paragraph{Quantization Time Comparison.}
As shown in Fig.~\ref{fig:overview}(c), TWLA exhibits a clear performance--efficiency trade-off. Compressing LLaMA2-7B takes only 82 minutes, substantially faster than SliM-LLM (158 minutes). Although TWLA is slightly slower than QuaRot, the accuracy gain it delivers (an average reduction of about 12 points in WikiText2 perplexity across the LLaMA family) makes its overhead moderate and practically acceptable.

\section{Conclusion}
\label{sec:conclusion}
In this work, we propose TWLA, a PTQ framework that achieves 1.58-bit weight compression together with 4-bit activation quantization while maintaining high accuracy. TWLA targets the key bottlenecks in this extreme quantization regime by systematically mitigating the mismatch between weight distributions and ternary codebooks and suppressing heavy-tailed activation outliers, thereby achieving a strong balance between efficiency and accuracy in a highly challenging setting. Extensive experiments show that TWLA provides a practical PTQ solution for ternary weights with low-bit activations and offers a scalable path toward efficient LLM inference.

\clearpage

\section*{Impact Statement}
TWLA is a post-training quantization (PTQ) framework that enables accurate inference with ternary weights and low-bit activations, substantially reducing memory footprint and improving throughput for large language models. This capability can lower the cost and energy required to deploy LLMs, and may broaden access to LLM-based applications in resource-constrained environments (e.g., on-device or edge deployment), potentially benefiting privacy-sensitive and latency-critical use cases.
Although our approach aims to make LLMs more accessible and widely used, it does not address the potential risks of misuse for malicious purposes. To mitigate these risks, a strong commitment to user data protection, clear ethical guidelines, and transparency mechanisms is essential.

\nocite{langley00}

\bibliography{ref}

@inproceedings{langley00,
 author    = {P. Langley},
 title     = {Crafting Papers on Machine Learning},
 year      = {2000},
 pages     = {1207--1216},
 editor    = {Pat Langley},
 booktitle     = {Proceedings of the 17th International Conference
              on Machine Learning (ICML 2000)},
 address   = {Stanford, CA},
 publisher = {Morgan Kaufmann}
}

@misc{li2022ternaryweightnetworks,
      title={Ternary Weight Networks}, 
      author={Fengfu Li and Bin Liu and Xiaoxing Wang and Bo Zhang and Junchi Yan},
      year={2022},
      eprint={1605.04711},
      archivePrefix={arXiv},
      primaryClass={cs.CV},
      url={https://arxiv.org/abs/1605.04711}, 
}

@misc{zhu2017trainedternaryquantization,
      title={Trained Ternary Quantization}, 
      author={Chenzhuo Zhu and Song Han and Huizi Mao and William J. Dally},
      year={2017},
      eprint={1612.01064},
      archivePrefix={arXiv},
      primaryClass={cs.LG},
      url={https://arxiv.org/abs/1612.01064}, 
}

@misc{zhang2020ternarybertdistillationawareultralowbit,
      title={TernaryBERT: Distillation-aware Ultra-low Bit BERT}, 
      author={Wei Zhang and Lu Hou and Yichun Yin and Lifeng Shang and Xiao Chen and Xin Jiang and Qun Liu},
      year={2020},
      eprint={2009.12812},
      archivePrefix={arXiv},
      primaryClass={cs.CL},
      url={https://arxiv.org/abs/2009.12812}, 
}

@InProceedings{Wang_2018_CVPR,
author = {Wang, Peisong and Hu, Qinghao and Zhang, Yifan and Zhang, Chunjie and Liu, Yang and Cheng, Jian},
title = {Two-Step Quantization for Low-Bit Neural Networks},
booktitle = {Proceedings of the IEEE Conference on Computer Vision and Pattern Recognition (CVPR)},
month = {June},
year = {2018}
}

@INPROCEEDINGS{7966166,
  author={Alemdar, Hande and Leroy, Vincent and Prost-Boucle, Adrien and Pétrot, Frédéric},
  booktitle={2017 International Joint Conference on Neural Networks (IJCNN)}, 
  title={Ternary neural networks for resource-efficient AI applications}, 
  year={2017},
  volume={},
  number={},
  pages={2547-2554},
  doi={10.1109/IJCNN.2017.7966166}}

@misc{chen2024ternaryllmternarizedlargelanguage,
      title={TernaryLLM: Ternarized Large Language Model}, 
      author={Tianqi Chen and Zhe Li and Weixiang Xu and Zeyu Zhu and Dong Li and Lu Tian and Emad Barsoum and Peisong Wang and Jian Cheng},
      year={2024},
      eprint={2406.07177},
      archivePrefix={arXiv},
      primaryClass={cs.LG},
      url={https://arxiv.org/abs/2406.07177}, 
}

@misc{yan2025pt2llmposttrainingternarizationlarge,
      title={PT$^2$-LLM: Post-Training Ternarization for Large Language Models}, 
      author={Xianglong Yan and Chengzhu Bao and Zhiteng Li and Tianao Zhang and Kaicheng Yang and Haotong Qin and Ruobing Xie and Xingwu Sun and Yulun Zhang},
      year={2025},
      eprint={2510.03267},
      archivePrefix={arXiv},
      primaryClass={cs.LG},
      url={https://arxiv.org/abs/2510.03267}, 
}

@misc{wang2025bitnetv2native4bit,
      title={BitNet v2: Native 4-bit Activations with Hadamard Transformation for 1-bit LLMs}, 
      author={Hongyu Wang and Shuming Ma and Furu Wei},
      year={2025},
      eprint={2504.18415},
      archivePrefix={arXiv},
      primaryClass={cs.CL},
      url={https://arxiv.org/abs/2504.18415}, 
}

@inproceedings{Guan_2024, series={DAC ’24},
   title={APTQ: Attention-aware Post-Training Mixed-Precision Quantization for Large Language Models},
   url={http://dx.doi.org/10.1145/3649329.3658498},
   DOI={10.1145/3649329.3658498},
   booktitle={Proceedings of the 61st ACM/IEEE Design Automation Conference},
   publisher={ACM},
   author={Guan, Ziyi and Huang, Hantao and Su, Yupeng and Huang, Hong and Wong, Ngai and Yu, Hao},
   year={2024},
   month=jun, pages={1–6},
   collection={DAC ’24} }

@misc{dettmers2023spqrsparsequantizedrepresentationnearlossless,
      title={SpQR: A Sparse-Quantized Representation for Near-Lossless LLM Weight Compression}, 
      author={Tim Dettmers and Ruslan Svirschevski and Vage Egiazarian and Denis Kuznedelev and Elias Frantar and Saleh Ashkboos and Alexander Borzunov and Torsten Hoefler and Dan Alistarh},
      year={2023},
      eprint={2306.03078},
      archivePrefix={arXiv},
      primaryClass={cs.CL},
      url={https://arxiv.org/abs/2306.03078}, 
}

@misc{shang2023pbllmpartiallybinarizedlarge,
      title={PB-LLM: Partially Binarized Large Language Models}, 
      author={Yuzhang Shang and Zhihang Yuan and Qiang Wu and Zhen Dong},
      year={2023},
      eprint={2310.00034},
      archivePrefix={arXiv},
      primaryClass={cs.LG},
      url={https://arxiv.org/abs/2310.00034}, 
}

@inproceedings{li2023llm,
  title={Llm-mq: Mixed-precision quantization for efficient llm deployment},
  author={Li, Shiyao and Ning, Xuefei and Hong, Ke and Liu, Tengxuan and Wang, Luning and Li, Xiuhong and Zhong, Kai and Dai, Guohao and Yang, Huazhong and Wang, Yu},
  booktitle={The Efficient Natural Language and Speech Processing Workshop with NeurIPS},
  volume={9},
  pages={3},
  year={2023}
}

@article{huang2024slim,
  title={SliM-LLM: Salience-driven mixed-precision quantization for large language models},
  author={Huang, Wei and Qin, Haotong and Liu, Yangdong and Li, Yawei and Liu, Qinshuo and Liu, Xianglong and Benini, Luca and Magno, Michele and Zhang, Shiming and Qi, Xiaojuan},
  journal={arXiv preprint arXiv:2405.14917},
  year={2024}
}

@article{saxena2024resq,
  title={Resq: Mixed-precision quantization of large language models with low-rank residuals},
  author={Saxena, Utkarsh and Sharify, Sayeh and Roy, Kaushik and Wang, Xin},
  journal={arXiv preprint arXiv:2412.14363},
  year={2024}
}

@article{li2024arb,
  title={Arb-llm: Alternating refined binarizations for large language models},
  author={Li, Zhiteng and Yan, Xianglong and Zhang, Tianao and Qin, Haotong and Xie, Dong and Tian, Jiang and Kong, Linghe and Zhang, Yulun and Yang, Xiaokang and others},
  journal={arXiv preprint arXiv:2410.03129},
  year={2024}
}

@article{ye2025dbellquant,
  title={DBellQuant: Breaking the Bell with Double-Bell Transformation for LLMs Post Training Binarization},
  author={Ye, Zijian and Huang, Wei and Yu, Yifei and Ren, Tianhe and Wang, Zhongrui and Qi, Xiaojuan},
  journal={arXiv preprint arXiv:2507.01027},
  year={2025}
}

@article{Guo_2025,
   title={DeepSeek-R1 incentivizes reasoning in LLMs through reinforcement learning},
   volume={645},
   ISSN={1476-4687},
   url={http://dx.doi.org/10.1038/s41586-025-09422-z},
   DOI={10.1038/s41586-025-09422-z},
   number={8081},
   journal={Nature},
   publisher={Springer Science and Business Media LLC},
   author={Guo, Daya and Yang, Dejian and Zhang, Haowei and Song, Junxiao and Wang, Peiyi and Zhu, Qihao and Xu, Runxin and Zhang, Ruoyu and Ma, Shirong and Bi, Xiao and Zhang, Xiaokang and Yu, Xingkai and Wu, Yu and Wu, Z. F. and Gou, Zhibin and Shao, Zhihong and Li, Zhuoshu and Gao, Ziyi and Liu, Aixin and Xue, Bing and Wang, Bingxuan and Wu, Bochao and Feng, Bei and Lu, Chengda and Zhao, Chenggang and Deng, Chengqi and Ruan, Chong and Dai, Damai and Chen, Deli and Ji, Dongjie and Li, Erhang and Lin, Fangyun and Dai, Fucong and Luo, Fuli and Hao, Guangbo and Chen, Guanting and Li, Guowei and Zhang, H. and Xu, Hanwei and Ding, Honghui and Gao, Huazuo and Qu, Hui and Li, Hui and Guo, Jianzhong and Li, Jiashi and Chen, Jingchang and Yuan, Jingyang and Tu, Jinhao and Qiu, Junjie and Li, Junlong and Cai, J. L. and Ni, Jiaqi and Liang, Jian and Chen, Jin and Dong, Kai and Hu, Kai and You, Kaichao and Gao, Kaige and Guan, Kang and Huang, Kexin and Yu, Kuai and Wang, Lean and Zhang, Lecong and Zhao, Liang and Wang, Litong and Zhang, Liyue and Xu, Lei and Xia, Leyi and Zhang, Mingchuan and Zhang, Minghua and Tang, Minghui and Zhou, Mingxu and Li, Meng and Wang, Miaojun and Li, Mingming and Tian, Ning and Huang, Panpan and Zhang, Peng and Wang, Qiancheng and Chen, Qinyu and Du, Qiushi and Ge, Ruiqi and Zhang, Ruisong and Pan, Ruizhe and Wang, Runji and Chen, R. J. and Jin, R. L. and Chen, Ruyi and Lu, Shanghao and Zhou, Shangyan and Chen, Shanhuang and Ye, Shengfeng and Wang, Shiyu and Yu, Shuiping and Zhou, Shunfeng and Pan, Shuting and Li, S. S. and Zhou, Shuang and Wu, Shaoqing and Yun, Tao and Pei, Tian and Sun, Tianyu and Wang, T. and Zeng, Wangding and Liu, Wen and Liang, Wenfeng and Gao, Wenjun and Yu, Wenqin and Zhang, Wentao and Xiao, W. L. and An, Wei and Liu, Xiaodong and Wang, Xiaohan and Chen, Xiaokang and Nie, Xiaotao and Cheng, Xin and Liu, Xin and Xie, Xin and Liu, Xingchao and Yang, Xinyu and Li, Xinyuan and Su, Xuecheng and Lin, Xuheng and Li, X. Q. and Jin, Xiangyue and Shen, Xiaojin and Chen, Xiaosha and Sun, Xiaowen and Wang, Xiaoxiang and Song, Xinnan and Zhou, Xinyi and Wang, Xianzu and Shan, Xinxia and Li, Y. K. and Wang, Y. Q. and Wei, Y. X. and Zhang, Yang and Xu, Yanhong and Li, Yao and Zhao, Yao and Sun, Yaofeng and Wang, Yaohui and Yu, Yi and Zhang, Yichao and Shi, Yifan and Xiong, Yiliang and He, Ying and Piao, Yishi and Wang, Yisong and Tan, Yixuan and Ma, Yiyang and Liu, Yiyuan and Guo, Yongqiang and Ou, Yuan and Wang, Yuduan and Gong, Yue and Zou, Yuheng and He, Yujia and Xiong, Yunfan and Luo, Yuxiang and You, Yuxiang and Liu, Yuxuan and Zhou, Yuyang and Zhu, Y. X. and Huang, Yanping and Li, Yaohui and Zheng, Yi and Zhu, Yuchen and Ma, Yunxian and Tang, Ying and Zha, Yukun and Yan, Yuting and Ren, Z. Z. and Ren, Zehui and Sha, Zhangli and Fu, Zhe and Xu, Zhean and Xie, Zhenda and Zhang, Zhengyan and Hao, Zhewen and Ma, Zhicheng and Yan, Zhigang and Wu, Zhiyu and Gu, Zihui and Zhu, Zijia and Liu, Zijun and Li, Zilin and Xie, Ziwei and Song, Ziyang and Pan, Zizheng and Huang, Zhen and Xu, Zhipeng and Zhang, Zhongyu and Zhang, Zhen},
   year={2025},
   month=sep, pages={633–638} }

@misc{touvron2023llamaopenefficientfoundation,
      title={LLaMA: Open and Efficient Foundation Language Models}, 
      author={Hugo Touvron and Thibaut Lavril and Gautier Izacard and Xavier Martinet and Marie-Anne Lachaux and Timothée Lacroix and Baptiste Rozière and Naman Goyal and Eric Hambro and Faisal Azhar and Aurelien Rodriguez and Armand Joulin and Edouard Grave and Guillaume Lample},
      year={2023},
      eprint={2302.13971},
      archivePrefix={arXiv},
      primaryClass={cs.CL},
      url={https://arxiv.org/abs/2302.13971}, 
}

@misc{clark2018thinksolvedquestionanswering,
      title={Think you have Solved Question Answering? Try ARC, the AI2 Reasoning Challenge}, 
      author={Peter Clark and Isaac Cowhey and Oren Etzioni and Tushar Khot and Ashish Sabharwal and Carissa Schoenick and Oyvind Tafjord},
      year={2018},
      eprint={1803.05457},
      archivePrefix={arXiv},
      primaryClass={cs.AI},
      url={https://arxiv.org/abs/1803.05457}, 
}

@misc{zellers2019hellaswagmachinereallyfinish,
      title={HellaSwag: Can a Machine Really Finish Your Sentence?}, 
      author={Rowan Zellers and Ari Holtzman and Yonatan Bisk and Ali Farhadi and Yejin Choi},
      year={2019},
      eprint={1905.07830},
      archivePrefix={arXiv},
      primaryClass={cs.CL},
      url={https://arxiv.org/abs/1905.07830}, 
}

@misc{paperno2016lambadadatasetwordprediction,
      title={The LAMBADA dataset: Word prediction requiring a broad discourse context}, 
      author={Denis Paperno and Germán Kruszewski and Angeliki Lazaridou and Quan Ngoc Pham and Raffaella Bernardi and Sandro Pezzelle and Marco Baroni and Gemma Boleda and Raquel Fernández},
      year={2016},
      eprint={1606.06031},
      archivePrefix={arXiv},
      primaryClass={cs.CL},
      url={https://arxiv.org/abs/1606.06031}, 
}

@misc{bisk2019piqareasoningphysicalcommonsense,
      title={PIQA: Reasoning about Physical Commonsense in Natural Language}, 
      author={Yonatan Bisk and Rowan Zellers and Ronan Le Bras and Jianfeng Gao and Yejin Choi},
      year={2019},
      eprint={1911.11641},
      archivePrefix={arXiv},
      primaryClass={cs.CL},
      url={https://arxiv.org/abs/1911.11641}, 
}

@misc{sakaguchi2019winograndeadversarialwinogradschema,
      title={WinoGrande: An Adversarial Winograd Schema Challenge at Scale}, 
      author={Keisuke Sakaguchi and Ronan Le Bras and Chandra Bhagavatula and Yejin Choi},
      year={2019},
      eprint={1907.10641},
      archivePrefix={arXiv},
      primaryClass={cs.CL},
      url={https://arxiv.org/abs/1907.10641}, 
}

@article{wikitext2,
  title={Pointer sentinel mixture models},
  author={Merity, Stephen and Xiong, Caiming and Bradbury, James and Socher, Richard},
  journal={arXiv preprint arXiv:1609.07843},
  year={2016}
}

@misc{mmlu,
      title={Measuring Massive Multitask Language Understanding}, 
      author={Dan Hendrycks and Collin Burns and Steven Basart and Andy Zou and Mantas Mazeika and Dawn Song and Jacob Steinhardt},
      year={2021},
      eprint={2009.03300},
      archivePrefix={arXiv},
      primaryClass={cs.CY},
      url={https://arxiv.org/abs/2009.03300}, 
}

@misc{humaneval,
      title={Evaluating Large Language Models Trained on Code}, 
      author={Mark Chen and Jerry Tworek and Heewoo Jun and Qiming Yuan and Henrique Ponde de Oliveira Pinto and Jared Kaplan and Harri Edwards and Yuri Burda and Nicholas Joseph and Greg Brockman and Alex Ray and Raul Puri and Gretchen Krueger and Michael Petrov and Heidy Khlaaf and Girish Sastry and Pamela Mishkin and Brooke Chan and Scott Gray and Nick Ryder and Mikhail Pavlov and Alethea Power and Lukasz Kaiser and Mohammad Bavarian and Clemens Winter and Philippe Tillet and Felipe Petroski Such and Dave Cummings and Matthias Plappert and Fotios Chantzis and Elizabeth Barnes and Ariel Herbert-Voss and William Hebgen Guss and Alex Nichol and Alex Paino and Nikolas Tezak and Jie Tang and Igor Babuschkin and Suchir Balaji and Shantanu Jain and William Saunders and Christopher Hesse and Andrew N. Carr and Jan Leike and Josh Achiam and Vedant Misra and Evan Morikawa and Alec Radford and Matthew Knight and Miles Brundage and Mira Murati and Katie Mayer and Peter Welinder and Bob McGrew and Dario Amodei and Sam McCandlish and Ilya Sutskever and Wojciech Zaremba},
      year={2021},
      eprint={2107.03374},
      archivePrefix={arXiv},
      primaryClass={cs.LG},
      url={https://arxiv.org/abs/2107.03374}, 
}

@misc{gsm8k,
      title={Training Verifiers to Solve Math Word Problems}, 
      author={Karl Cobbe and Vineet Kosaraju and Mohammad Bavarian and Mark Chen and Heewoo Jun and Lukasz Kaiser and Matthias Plappert and Jerry Tworek and Jacob Hilton and Reiichiro Nakano and Christopher Hesse and John Schulman},
      year={2021},
      eprint={2110.14168},
      archivePrefix={arXiv},
      primaryClass={cs.LG},
      url={https://arxiv.org/abs/2110.14168}, 
}

@misc{c4,
      title={Exploring the Limits of Transfer Learning with a Unified Text-to-Text Transformer}, 
      author={Colin Raffel and Noam Shazeer and Adam Roberts and Katherine Lee and Sharan Narang and Michael Matena and Yanqi Zhou and Wei Li and Peter J. Liu},
      year={2023},
      eprint={1910.10683},
      archivePrefix={arXiv},
      primaryClass={cs.LG},
      url={https://arxiv.org/abs/1910.10683}, 
}

@misc{qwen3technicalreport,
      title={Qwen3 Technical Report}, 
      author={An Yang and Anfeng Li and Baosong Yang and Beichen Zhang and Binyuan Hui and Bo Zheng and Bowen Yu and Chang Gao and Chengen Huang and Chenxu Lv and Chujie Zheng and Dayiheng Liu and Fan Zhou and Fei Huang and Feng Hu and Hao Ge and Haoran Wei and Huan Lin and Jialong Tang and Jian Yang and Jianhong Tu and Jianwei Zhang and Jianxin Yang and Jiaxi Yang and Jing Zhou and Jingren Zhou and Junyang Lin and Kai Dang and Keqin Bao and Kexin Yang and Le Yu and Lianghao Deng and Mei Li and Mingfeng Xue and Mingze Li and Pei Zhang and Peng Wang and Qin Zhu and Rui Men and Ruize Gao and Shixuan Liu and Shuang Luo and Tianhao Li and Tianyi Tang and Wenbiao Yin and Xingzhang Ren and Xinyu Wang and Xinyu Zhang and Xuancheng Ren and Yang Fan and Yang Su and Yichang Zhang and Yinger Zhang and Yu Wan and Yuqiong Liu and Zekun Wang and Zeyu Cui and Zhenru Zhang and Zhipeng Zhou and Zihan Qiu},
      year={2025},
      eprint={2505.09388},
      archivePrefix={arXiv},
      primaryClass={cs.CL},
      url={https://arxiv.org/abs/2505.09388}, 
}

@misc{frantar2023gptqaccurateposttrainingquantization,
      title={GPTQ: Accurate Post-Training Quantization for Generative Pre-trained Transformers}, 
      author={Elias Frantar and Saleh Ashkboos and Torsten Hoefler and Dan Alistarh},
      year={2023},
      eprint={2210.17323},
      archivePrefix={arXiv},
      primaryClass={cs.LG},
      url={https://arxiv.org/abs/2210.17323}, 
}

@article{ashkboos2024quarot,
  title={Quarot: Outlier-free 4-bit inference in rotated llms},
  author={Ashkboos, Saleh and Mohtashami, Amirkeivan and Croci, Maximilian L and Li, Bo and Cameron, Pashmina and Jaggi, Martin and Alistarh, Dan and Hoefler, Torsten and Hensman, James},
  journal={Advances in Neural Information Processing Systems},
  volume={37},
  pages={100213--100240},
  year={2024}
}

@inproceedings{zhao2026specquant,
  title={Specquant: Spectral decomposition and adaptive truncation for ultra-low-bit llms quantization},
  author={Zhao, Zhixiong and Liu, Fangxin and Wang, Junjie and Guan, Chenyang and Wang, Zongwu and Jiang, Li and Guan, Haibing},
  booktitle={Proceedings of the AAAI Conference on Artificial Intelligence},
  volume={40},
  number={34},
  pages={28786--28794},
  year={2026}
}

@inproceedings{zhao2025quark,
  title={QUARK: Quantization-Enabled Circuit Sharing for Transformer Acceleration by Exploiting Common Patterns in Nonlinear Operations},
  author={Zhao, Zhixiong and Li, Haomin and Liu, Fangxin and Lu, Yuncheng and Wang, Zongwu and Yang, Tao and Jiang, Li and Guan, Haibing},
  booktitle={2025 IEEE/ACM International Conference On Computer Aided Design (ICCAD)},
  pages={1--9},
  year={2025},
  organization={IEEE}
}

@article{xu2026kbvq,
  title={KBVQ-MoE: KLT-guided SVD with Bias-Corrected Vector Quantization for MoE Large Language Models},
  author={Xu, Zukang and Zhao, Zhixiong and Hu, Xing and Chen, Zhixuan and Yang, Dawei},
  journal={arXiv preprint arXiv:2602.11184},
  year={2026}
}

@article{zhao2026bwla,
  title={Bwla: Breaking the barrier of w1ax post-training quantization for llms},
  author={Zhao, Zhixiong and Xu, Zukang and Yang, Dawei},
  journal={arXiv preprint arXiv:2605.00422},
  year={2026}
}

@misc{liu2025bimacosrbinaryonestepdiffusion,
      title={BiMaCoSR: Binary One-Step Diffusion Model Leveraging Flexible Matrix Compression for Real Super-Resolution}, 
      author={Kai Liu and Kaicheng Yang and Zheng Chen and Zhiteng Li and Yong Guo and Wenbo Li and Linghe Kong and Yulun Zhang},
      year={2025},
      eprint={2502.00333},
      archivePrefix={arXiv},
      primaryClass={cs.CV},
      url={https://arxiv.org/abs/2502.00333}, 
}

@misc{yang2025robuqpushingditsw158a2,
      title={RobuQ: Pushing DiTs to W1.58A2 via Robust Activation Quantization}, 
      author={Kaicheng Yang and Xun Zhang and Haotong Qin and Yucheng Lin and Kaisen Yang and Xianglong Yan and Yulun Zhang},
      year={2025},
      eprint={2509.23582},
      archivePrefix={arXiv},
      primaryClass={cs.CV},
      url={https://arxiv.org/abs/2509.23582}, 
}

@misc{yang2025treeqpushingquantizationboundary,
      title={TreeQ: Pushing the Quantization Boundary of Diffusion Transformer via Tree-Structured Mixed-Precision Search}, 
      author={Kaicheng Yang and Kaisen Yang and Baiting Wu and Xun Zhang and Qianrui Yang and Haotong Qin and He Zhang and Yulun Zhang},
      year={2025},
      eprint={2512.06353},
      archivePrefix={arXiv},
      primaryClass={cs.CV},
      url={https://arxiv.org/abs/2512.06353}, 
}

@misc{gul2026flrqfasterllmquantization,
      title={FLRQ: Faster LLM Quantization with Flexible Low-Rank Matrix Sketching}, 
      author={Hongyaoxing Gul and Lijuan Hu and Shuzi Niu and Fangfang Liu},
      year={2026},
      eprint={2601.05684},
      archivePrefix={arXiv},
      primaryClass={cs.LG},
      url={https://arxiv.org/abs/2601.05684}, 
}

@misc{gu2026loproenhancinglowrankquantization,
      title={LoPRo: Enhancing Low-Rank Quantization via Permuted Block-Wise Rotation}, 
      author={Hongyaoxing Gu and Lijuan Hu and Liye Yu and Haowei Li and Fangfang Liu},
      year={2026},
      eprint={2601.19675},
      archivePrefix={arXiv},
      primaryClass={cs.LG},
      url={https://arxiv.org/abs/2601.19675}, 
}

@misc{xiao2025singlequantefficientquantizationlarge,
      title={SingleQuant: Efficient Quantization of Large Language Models in a Single Pass}, 
      author={Jinying Xiao and Bin Ji and Shasha Li and Xiaodong Liu and Ma Jun and Ye Zhong and Wei Li and Xuan Xie and Qingbo Wu and Jie Yu},
      year={2025},
      eprint={2511.22316},
      archivePrefix={arXiv},
      primaryClass={cs.LG},
      url={https://arxiv.org/abs/2511.22316}, 
}

@inproceedings{li2024towards,
  title={Towards effective data-free knowledge distillation via diverse diffusion augmentation},
  author={Li, Muquan and Zhang, Dongyang and He, Tao and Xie, Xiurui and Li, Yuan-Fang and Qin, Ke},
  booktitle={Proceedings of the 32nd ACM International Conference on Multimedia},
  pages={4416--4425},
  year={2024}
}

@inproceedings{li2025adaptive,
  title={Adaptive dataset quantization},
  author={Li, Muquan and Zhang, Dongyang and Dong, Qiang and Xie, Xiurui and Qin, Ke},
  booktitle={Proceedings of the AAAI Conference on Artificial Intelligence},
  volume={39},
  number={11},
  pages={12093--12101},
  year={2025}
}

@inproceedings{zhu2025pathology,
  title={Pathology-Aware Prototype Evolution via LLM-Driven Semantic Disambiguation for Multicenter Diabetic Retinopathy Diagnosis},
  author={Zhu, Chunzheng and Lin, Yangfang and Shao, Jialin and Lin, Jianxin and Wang, Yijun},
  booktitle={Proceedings of the 33rd ACM International Conference on Multimedia},
  pages={9196--9205},
  year={2025}
}
\bibliographystyle{icml2026}

\newpage
\appendix
\onecolumn
\section*{Appendix Overview}
\begin{itemize}
    \item Section~\ref{app:pseudocode}: Pseudocode of TWLA.
    \item Section~\ref{sec:proof}: Detailed Proofs and Derivations.
    \begin{itemize}
        \item Section~\ref{app:mono_decrease}: Monotonic Decrease of the Euclidean Initialization Updates.
        \item Section~\ref{app:e2m_rowwise_normal_eq}: Row-wise Normal Equations and Numerical Stabilization for Metric-Aware Relocation.
        \item Section~\ref{app:kron_complexity}: Storage and Computational Complexity of Kronecker-Structured Orthogonal Transforms.
        \item Section~\ref{app:act_outlier}: Why Orthogonal Mixing Suppresses Activation Outliers.
        \item Section~\ref{app:kotms-zero-constraint}: Zero-Peak Mass Regularization for TriGMM.
        \item Section~\ref{app:kotms-soft-proj}: Small-Variance Limit: TriGMM as a Soft Projection.
        \item Section~\ref{app:adjacent_only}: Why Modeling Only Adjacent-Layer Interactions is Sufficient.
        \item Section~\ref{app:dp}: Dynamic Programming for Chain-Structured Mixed-Precision Allocation.
    \end{itemize}
    \item Section~\ref{app:results}: Additional Experimental Results.
    \begin{itemize}
        \item Section~\ref{app:detailed main results}: More Detailed Results.
        \item Section~\ref{app:dataset}: Ablation study on Calibration Data.
        \item Section~\ref{app:amp}: Ablation study on ILA-AMP.
    \end{itemize}
    \item Section~\ref{app:distribution}: Distribution Visualizations.
    \begin{itemize}
        \item Section~\ref{app:vis_kotms_single_layer}: Single-layer visualization of KOTMS.
        \item Section~\ref{app:vis_kotms_cross_layer}: Cross-layer heterogeneity in activation quantizability after KOTMS.
    \end{itemize}
\end{itemize}

\section{Pseudocode of TWLA}
\label{app:pseudocode}

For completeness, this section provides the detailed pseudocode of our TWLA framework. TWLA is a unified post-training quantization (PTQ) pipeline for large language models, aiming to jointly enable 1.58-bit ternary weights and low-bit activations (e.g., 2–6-bit mixed precision) while maintaining competitive accuracy.

As described in Section~\ref{sec:method}, TWLA consists of three key modules, highlighted in the pseudocode via color-coded comment blocks:
(1) KOTMS (Kronecker Orthogonal Tri-Modal Shaping): For each layer, we construct a structured rotation matrix $R_\ell = R_{1,\ell}\otimes R_{2,\ell}$ from two lightweight Kronecker-structured orthogonal factors $R_{1,\ell}\in\mathbb{R}^{m_1\times m_1}$ and $R_{2,\ell}\in\mathbb{R}^{m_2\times m_2}$. To obtain a suitable decomposition, we first factorize the rotation dimension into $(m_1,m_2)$ using Algorithm~\ref{alg:kron_factor_twla_alg}, enabling efficient construction of the orthogonal Kronecker structure. KOTMS performs function-preserving orthogonal mixing to reshape the typically unimodal, near-zero-concentrated and heavy-tailed weights into a tri-modal form that better matches ternary codebooks, and it statistically suppresses activation heavy tails and outliers through the same orthogonal mixing. The full optimization procedure is summarized in Algorithm~\ref{alg:kotms_alg}.

\begin{algorithm}[h!]
\caption{Kronecker Dimension Factorization (for KOTMS)}
\label{alg:kron_factor_twla_alg}
\begin{algorithmic}[1]
\REQUIRE Hidden dimension $d$ to be rotated (typically $d=m$ for $W\in\mathbb{R}^{n\times m}$), $d\ge 1$
\ENSURE Factor pair $(d_1,d_2)$ such that $d=d_1d_2$ and $d_1\approx d_2$
\STATE $a \leftarrow \lfloor \sqrt{d}\rfloor$ \hfill // start near $\sqrt{d}$
\WHILE{$a>1$ \AND $d \bmod a \ne 0$}
    \STATE $a \leftarrow a-1$ \hfill // search closest divisor below $\sqrt{d}$
\ENDWHILE
\STATE $d_2 \leftarrow a$
\STATE $d_1 \leftarrow d/d_2$
\STATE \textbf{return} $(d_1,d_2)$
\end{algorithmic}
\end{algorithm}

\begin{algorithm}[h!]
\caption{KOTMS: Kronecker Orthogonal Tri-Modal Shaping}
\label{alg:kotms_alg}
\begin{algorithmic}[1]
\REQUIRE Weight matrix $W\in\mathbb{R}^{n\times m}$; target sparsity $\rho$; prior $\pi_0$; penalty $\beta$
\REQUIRE Steps $T_{\mathrm{kotms}}$; step size $\eta$; numerical floor $\sigma_{\min}$
\ENSURE Kronecker orthogonal factors $(R_1,R_2)$ and implicit rotation $R=R_1\otimes R_2$

\STATE $\pi_+ \leftarrow (1-\pi_0)/2$, \quad $\pi_- \leftarrow (1-\pi_0)/2$
\STATE $(m_1,m_2)\leftarrow \textsc{KroneckerFactorize}(m)$ \hfill // see Alg.~\ref{alg:kron_factor_twla_alg}
\STATE Initialize $S_1\in\mathbb{R}^{m_1\times m_1}$, $S_2\in\mathbb{R}^{m_2\times m_2}$ (e.g., zeros)

\textbf{\textcolor[HTML]{f27830}{Orthogonal Kronecker rotation via Cayley parameterization}}
\FOR{$t=1$ \textbf{to} $T_{\mathrm{kotms}}$}
    \STATE $A_1 \leftarrow S_1-S_1^\top$, \quad $A_2 \leftarrow S_2-S_2^\top$
    \STATE $R_1 \leftarrow (I+A_1)^{-1}(I-A_1)$,\quad $R_2 \leftarrow (I+A_2)^{-1}(I-A_2)$
    \STATE $R \leftarrow R_1\otimes R_2$, \quad $Z \leftarrow WR$

    \AlgNote{Tri-modal shaping loss (row-wise TriGMM + zero-mass regularizer)}
    \FOR{$i=1,2,\dots,n$}
        \STATE $c_i \leftarrow \frac{1}{m}\sum_{j=1}^m |z_{ij}|$,\quad $\sigma_i \leftarrow \max(\mathrm{std}(\{z_{ij}\}_{j=1}^m),\sigma_{\min})$
        \STATE Compute responsibilities $\{r_{ij+},r_{ij0},r_{ij-}\}_{j=1}^m$ using $(\pi_+,\pi_0,\pi_-)$ and means $(+c_i,0,-c_i)$
        \STATE $\bar r_{i0}\leftarrow \frac{1}{m}\sum_{j=1}^m r_{ij0}$
    \ENDFOR
    \STATE $L_{\mathrm{shape}}\leftarrow L_{\mathrm{TriGMM}}(Z)+\beta\cdot \frac{1}{n}\sum_i(\bar r_{i0}-\rho)^2$
    \STATE $(S_1,S_2)\leftarrow (S_1,S_2) - \eta \nabla L_{\mathrm{shape}}$
\ENDFOR

\STATE \textbf{return} $(R_1,R_2)$
\end{algorithmic}
\end{algorithm}

(2) E2M-ATQ (Euclidean-to-Manifold Asymmetric Ternary Quantizer): After KOTMS shaping, we ternarize each layer using a two-stage Euclidean-to-Manifold strategy. In Stage I, we alternate updates of the discrete ternary pattern and continuous row-wise parameters under Frobenius (Euclidean) geometry to obtain a stable ternary structural prior. In Stage II, we fix the ternary pattern and relocate the row-wise shift/scale parameters under the calibration-induced metric defined by the activation second moment, yielding a metric-consistent closed-form update that directly aligns the forward output error. The pseudocode is provided in Algorithm~\ref{alg:e2m_atq_alg}.

\begin{algorithm}[h!]
\caption{E2M-ATQ: Euclidean Warm-start $\rightarrow$ Manifold Relocation}
\label{alg:e2m_atq_alg}
\begin{algorithmic}[1]
\REQUIRE Weight matrix $W\in\mathbb{R}^{n\times m}$; calibration inputs $X$ for this layer
\REQUIRE Euclidean iters $T_{\mathrm{euc}}$; TWN-style threshold rule (e.g., $\Delta$); metric floor (optional)
\ENSURE Ternary codebook $T^{(0)}\in\{-1,0,+1\}^{n\times m}$, row-wise $(\mu,\alpha)$, ternary-dequantized $\bar W$

\textbf{\textcolor[HTML]{1d73b6}{Stage I: Euclidean warm-start (stabilize discrete pattern)}}
\STATE $\mu \leftarrow \mathrm{row\_mean}(W)$
\STATE Initialize $T \leftarrow \textsc{TernaryThreshold}(W-\mu,\Delta)$
\STATE $\alpha \leftarrow \textsc{SupportAwareLS}(W,\mu,T)$
\FOR{$t=1$ \textbf{to} $T_{\mathrm{euc}}$}
    \STATE $R \leftarrow W - \mu \mathbf{1}^\top - \mathrm{diag}(\alpha)T$ \hfill // residual
    \STATE $\mu \leftarrow \mu + \frac{1}{m}R\mathbf{1}$ \hfill // residual-mean correction
    \STATE $\alpha \leftarrow \textsc{SupportAwareLS}(W,\mu,T)$
    \STATE Update $\Delta$ and $T \leftarrow \textsc{TernaryThreshold}(W-\mu,\Delta)$
\ENDFOR
\STATE $T^{(0)} \leftarrow T$ \hfill // freeze stratum

\textbf{\textcolor[HTML]{1d73b6}{Stage II: Manifold relocation under calibration-induced metric}}
\STATE $S \leftarrow \sum_b X_b^\top X_b$ \hfill // second moment
\FOR{$i=1,2,\dots,n$}
    \STATE $t_i \leftarrow T^{(0)}_{i:}$,\quad $w_i \leftarrow W_{i:}$
    \STATE Form $2\times 2$ system: 
    $A_i=\begin{bmatrix}t_iSt_i^\top & t_iS\mathbf{1}\\ \mathbf{1}^\top St_i^\top & \mathbf{1}^\top S\mathbf{1}\end{bmatrix}$,\;
    $b_i=\begin{bmatrix}t_iSw_i^\top\\ \mathbf{1}^\top Sw_i^\top\end{bmatrix}$
    \STATE Solve $A_i\begin{bmatrix}\alpha_i\\ \mu_i\end{bmatrix}=b_i$
\ENDFOR
\STATE $\bar W \leftarrow \mu \mathbf{1}^\top + \mathrm{diag}(\alpha)T^{(0)}$
\STATE \textbf{return} $(T^{(0)},\mu,\alpha,\bar W)$
\end{algorithmic}
\end{algorithm}

(3) ILA-AMP (Inter-Layer Aware Activation Mixed Precision): Although KOTMS can indirectly improve activation quantizability, the gain is often layer-dependent and may propagate through deep Transformers, causing a few low-bit layers to become bottlenecks and trigger cascading degradation. To address this, we treat the activation bitwidth as a controllable degree of freedom and allocate it across layers under a global budget. ILA-AMP constructs an efficient second-order surrogate based on the average validation NLL, including unary per-layer costs and adjacent pairwise interaction costs to capture local coupling amplification induced by distribution shifts and error propagation. Thanks to the chain-structured objective, the optimal mixed-precision configuration can be obtained exactly via dynamic programming. The procedure is summarized in Algorithm~\ref{alg:ila_amp_alg}.

\begin{algorithm}[h!]
\caption{ILA-AMP: Inter-Layer Aware Activation Mixed Precision}
\label{alg:ila_amp_alg}
\begin{algorithmic}[1]
\REQUIRE Quantized-weight model with $L$ layers; bit candidates $\mathcal{B}=\{2,4,6,8\}$; reference $b_{\max}=8$
\REQUIRE Validation set $\mathcal{D}_{\mathrm{val}}$; layer weights $\{w_\ell\}$; budget $B$
\ENSURE Optimal activation bits $\{b_\ell^\star\}_{\ell=1}^L$

\textbf{\textcolor[HTML]{24a645}{Build unary and adjacent pairwise costs via NLL}}
\STATE Set all layers to $b_{\max}$; $\mathrm{BaseNLL}\leftarrow v_{\mathrm{NLL}}(\mathbf{b}_{\max})$
\FOR{$\ell=1,2,\dots,L$}
    \FOR{each $b\in\mathcal{B}$}
        \STATE Set all layers to $b_{\max}$; set layer $\ell$ to $b$
        \STATE $C_\ell(b)\leftarrow v_{\mathrm{NLL}}(b_\ell=b,\text{ others }b_{\max})-\mathrm{BaseNLL}$
    \ENDFOR
\ENDFOR
\FOR{$\ell=2,3,\dots,L$}
    \FOR{each $b'\in\mathcal{B}$}
        \FOR{each $b\in\mathcal{B}$}
            \STATE Set all layers to $b_{\max}$; set layers $(\ell\!-\!1,\ell)$ to $(b',b)$
            \STATE $K_{\ell-1,\ell}(b',b)\leftarrow v_{\mathrm{NLL}}(\ell\!-\!1{=}b',\ell{=}b)-\mathrm{BaseNLL}-C_{\ell-1}(b')-C_\ell(b)$
        \ENDFOR
    \ENDFOR
\ENDFOR

\textbf{\textcolor[HTML]{24a645}{Dynamic programming under global budget}}
\STATE Initialize $\mathrm{DP}[\ell][c][b]\leftarrow +\infty$
\FOR{each $b\in\mathcal{B}$}
    \STATE $c\leftarrow w_1 b$; \IF{$c\le B$} \STATE $\mathrm{DP}[1][c][b]\leftarrow C_1(b)$ \ENDIF
\ENDFOR
\FOR{$\ell=2,3,\dots,L$}
    \FOR{$c=0,1,\dots,B$}
        \FOR{each $b\in\mathcal{B}$}
            \IF{$c-w_\ell b \ge 0$}
                \STATE $\mathrm{DP}[\ell][c][b]\leftarrow C_\ell(b)+\min_{b'\in\mathcal{B}}
                \big(\mathrm{DP}[\ell-1][c-w_\ell b][b']+K_{\ell-1,\ell}(b',b)\big)$
                \STATE Record argmin pointer for backtracking
            \ENDIF
        \ENDFOR
    \ENDFOR
\ENDFOR
\STATE Backtrack from $\min_{c\le B}\min_{b\in\mathcal{B}}\mathrm{DP}[L][c][b]$ to get $\{b_\ell^\star\}$
\STATE \textbf{return} $\{b_\ell^\star\}$
\end{algorithmic}
\end{algorithm}

Overall, Algorithm~\ref{alg:twla_simplified_alg} presents the simplified end-to-end TWLA pipeline: we first learn and fold Kronecker orthogonal rotations per layer (KOTMS), then perform E2M-ATQ ternarization in the rotated coordinate system, and finally solve the budget-constrained inter-layer-aware mixed-precision activation assignment with ILA-AMP. The resulting pseudocode faithfully reflects the PTQ pipeline used in all our experiments and can be directly applied to large language models at different scales.

\begin{algorithm}[h!]
\caption{TWLA: Overall PTQ Pipeline}
\label{alg:twla_simplified_alg}
\begin{algorithmic}[1]
\REQUIRE LLM with $L$ layers and weights $\{W_\ell\}$; calibration set $\mathcal{D}_{\mathrm{cal}}$; validation set $\mathcal{D}_{\mathrm{val}}$
\REQUIRE KOTMS hyperparams; E2M-ATQ hyperparams; ILA-AMP hyperparams and budget $B$
\ENSURE Ternary weights $\{\bar W_\ell\}$, rotations $\{(R_{1,\ell},R_{2,\ell})\}$, activation bits $\{b_\ell^\star\}$

\textbf{\textcolor[HTML]{f27830}{Step 1: KOTMS (function-preserving rotation + tri-modal shaping)}}
\FOR{$\ell=1,2,\dots,L$}
    \STATE $(R_{1,\ell},R_{2,\ell}) \leftarrow \textsc{KOTMS}(W_\ell)$ \hfill // Alg.~\ref{alg:kotms_alg}
    \STATE $R_\ell \leftarrow R_{1,\ell}\otimes R_{2,\ell}$
    \STATE $W_\ell \leftarrow W_\ell R_\ell$; fold/insert $R_\ell^\top$ onto activation pathway
\ENDFOR

\textbf{\textcolor[HTML]{1d73b6}{Step 2: E2M-ATQ (ternarize weights with Euclidean-to-Manifold updates)}}
\FOR{$\ell=1,2,\dots,L$}
    \STATE Collect layer inputs $X_\ell$ from $\mathcal{D}_{\mathrm{cal}}$
    \STATE $(T_\ell^{(0)},\mu_\ell,\alpha_\ell,\bar W_\ell)\leftarrow \textsc{E2M-ATQ}(W_\ell,X_\ell)$ \hfill // Alg.~\ref{alg:e2m_atq_alg}
    \STATE Replace $W_\ell$ by $\bar W_\ell$
\ENDFOR

\textbf{\textcolor[HTML]{24a645}{Step 3: ILA-AMP (allocate activation bits via DP with adjacent interactions)}}
\STATE $\{b_\ell^\star\}_{\ell=1}^L \leftarrow \textsc{ILA-AMP}(\mathcal{D}_{\mathrm{val}},B)$ \hfill // Alg.~\ref{alg:ila_amp_alg}
\STATE Deploy activation quantizers according to $\{b_\ell^\star\}$

\STATE \textbf{return} Quantized model
\end{algorithmic}
\end{algorithm}

\section{Detailed Proofs and Derivations}
\label{sec:proof}

\subsection{Monotonic Decrease of the Euclidean Initialization Updates}
\label{app:mono_decrease}

This appendix justifies the statement in the main text that the Euclidean initialization
updates follow a coordinate-descent scheme and yield a monotonically non-increasing objective
(cf.\ Eq.~(2) in the main paper).

\subsubsection{Row-wise Objective}
We analyze a single row for clarity; the extension to the full matrix follows by summing over rows.
Let $\mathbf{w}\in\mathbb{R}^{m}$ be a row of the weight matrix, and let
$\mu\in\mathbb{R}$, $\alpha\in\mathbb{R}$, and $\mathbf{t}\in\{-1,0,1\}^{m}$ be the row-wise
shift, scale, and ternary pattern, respectively. Define the (row-wise) quantization error
\begin{equation}
\mathcal{L}(\mu,\alpha,\mathbf{t})
\;=\;
\big\|\mathbf{w}-\mu\mathbf{1}-\alpha\mathbf{t}\big\|_2^2,
\label{eq:row_obj}
\end{equation}
where $\mathbf{1}\in\mathbb{R}^{m}$ is the all-one vector. For a ternary vector $\mathbf{t}$,
we denote its support size by
\begin{equation}
s(\mathbf{t}) \;=\; \|\mathbf{t}\|_0 \;=\;\sum_{k=1}^{m}|t_k|
\;=\;\sum_{k=1}^{m}t_k^2,
\label{eq:support_size}
\end{equation}
which corresponds to the number of non-zero ternary entries (i.e., excluding ternary zeros).

\subsubsection{Update Rules (Coordinate Descent)}
Given $(\mu^\tau,\alpha^\tau,\mathbf{t}^\tau)$ at iteration $\tau$, we perform the following updates:
\begin{enumerate}
\item \textbf{$\mu$-update.} With $(\alpha^\tau,\mathbf{t}^\tau)$ fixed, update
\begin{equation}
\mu^{\tau+1}
\;=\;
\mu^\tau + \delta_\mu^\tau,
\qquad
\delta_\mu^\tau
\;=\;
\frac{1}{m}\mathbf{1}^\top
\big(\mathbf{w}-\mu^\tau\mathbf{1}-\alpha^\tau\mathbf{t}^\tau\big).
\label{eq:mu_update}
\end{equation}

\item \textbf{$\alpha$-update (support-aware).} With $(\mu^{\tau+1},\mathbf{t}^\tau)$ fixed, update
\begin{equation}
\alpha^{\tau+1}
\;=\;
\arg\min_{\alpha}\,
\big\|\mathbf{w}-\mu^{\tau+1}\mathbf{1}-\alpha\mathbf{t}^\tau\big\|_2^2
\;=\;
\frac{\sum_{k=1}^{m} t_k^\tau\,(w_k-\mu^{\tau+1})}{\sum_{k=1}^{m}|t_k^\tau|},
\label{eq:alpha_update}
\end{equation}
where the denominator excludes ternary zeros ($t_k^\tau=0$).

\item \textbf{$\mathbf{t}$-update (row-wise ternary thresholding).}
With $(\mu^{\tau+1},\alpha^{\tau+1})$ fixed, update $\mathbf{t}$ by row-wise ternary thresholding,
which is equivalent to the element-wise minimization
\begin{equation}
t_k^{\tau+1}
\in
\arg\min_{t\in\{-1,0,1\}}
(w_k-\mu^{\tau+1}-\alpha^{\tau+1}t)^2,
\qquad k=1,\dots,m.
\label{eq:t_update}
\end{equation}
\end{enumerate}

\subsubsection{Monotonic Decrease}
\begin{lemma}[Decrease after the $\mu$-update]
\label{lem:mu_decrease}
Let $\mathbf{r}^\tau=\mathbf{w}-\mu^\tau\mathbf{1}-\alpha^\tau\mathbf{t}^\tau$ be the residual.
After updating $\mu$ as in Eq.~\ref{eq:mu_update} (with $\alpha^\tau,\mathbf{t}^\tau$ fixed),
the objective decreases by
\begin{equation}
\mathcal{L}(\mu^\tau,\alpha^\tau,\mathbf{t}^\tau)
-
\mathcal{L}(\mu^{\tau+1},\alpha^\tau,\mathbf{t}^\tau)
\;=\;
m(\delta_\mu^\tau)^2
\;\ge\;0.
\label{eq:mu_drop}
\end{equation}
\end{lemma}
\begin{proof}
Since $\mathbf{r}^{\tau+1}=\mathbf{r}^\tau-\delta_\mu^\tau\mathbf{1}$,
\[
\|\mathbf{r}^{\tau+1}\|_2^2
=
\|\mathbf{r}^\tau\|_2^2
-2\delta_\mu^\tau\,\mathbf{1}^\top\mathbf{r}^\tau
+m(\delta_\mu^\tau)^2.
\]
Using $\delta_\mu^\tau=\frac{1}{m}\mathbf{1}^\top\mathbf{r}^\tau$ yields
$\|\mathbf{r}^{\tau+1}\|_2^2=\|\mathbf{r}^\tau\|_2^2-m(\delta_\mu^\tau)^2$.
\end{proof}

\begin{lemma}[Decrease after the support-aware $\alpha$-update]
\label{lem:alpha_decrease}
Fix $(\mu,\mathbf{t})$ and let $\alpha^{\star}$ be the minimizer in Eq.~\ref{eq:alpha_update}.
Then for any $\alpha$,
\begin{equation}
\mathcal{L}(\mu,\alpha,\mathbf{t})
-
\mathcal{L}(\mu,\alpha^{\star},\mathbf{t})
\;=\;
\|\mathbf{t}\|_2^2\,(\alpha-\alpha^{\star})^2
\;=\;
s(\mathbf{t})\,(\alpha-\alpha^{\star})^2
\;\ge\;0.
\label{eq:alpha_drop}
\end{equation}
\end{lemma}
\begin{proof}
Let $\mathbf{u}=\mathbf{w}-\mu\mathbf{1}$. Then
$\mathcal{L}(\mu,\alpha,\mathbf{t})=\|\mathbf{u}-\alpha\mathbf{t}\|_2^2$.
This is a 1D least-squares problem with closed-form minimizer
$\alpha^{\star}=\frac{\mathbf{t}^\top\mathbf{u}}{\mathbf{t}^\top\mathbf{t}}$.
Expanding $\|\mathbf{u}-\alpha\mathbf{t}\|_2^2$ around $\alpha^{\star}$ gives
$\|\mathbf{u}-\alpha\mathbf{t}\|_2^2
=
\|\mathbf{u}-\alpha^{\star}\mathbf{t}\|_2^2
+(\mathbf{t}^\top\mathbf{t})(\alpha-\alpha^{\star})^2$.
Since $\mathbf{t}^\top\mathbf{t}=\sum_k t_k^2=\sum_k |t_k|=s(\mathbf{t})$, we obtain Eq.~\ref{eq:alpha_drop}.
\end{proof}

\begin{lemma}[Non-increase after the $\mathbf{t}$-update]
\label{lem:t_noninc}
Fix $(\mu,\alpha)$. If $\mathbf{t}^{\star}$ is updated element-wise as in Eq.~\ref{eq:t_update},
then
\begin{equation}
\mathcal{L}(\mu,\alpha,\mathbf{t}^{\star})
\;\le\;
\mathcal{L}(\mu,\alpha,\mathbf{t}).
\label{eq:t_noninc}
\end{equation}
\end{lemma}
\begin{proof}
With $(\mu,\alpha)$ fixed, the objective decomposes as
$\mathcal{L}(\mu,\alpha,\mathbf{t})=\sum_{k=1}^{m}(w_k-\mu-\alpha t_k)^2$.
Choosing each $t_k^{\star}\in\arg\min_{t\in\{-1,0,1\}}(w_k-\mu-\alpha t)^2$
minimizes every summand and therefore cannot increase the sum.
\end{proof}

\begin{proposition}[Monotonic decrease and convergence]
\label{prop:mono}
Let $\mathcal{L}^\tau=\mathcal{L}(\mu^\tau,\alpha^\tau,\mathbf{t}^\tau)$.
The update order $\mu\rightarrow\alpha\rightarrow\mathbf{t}$ satisfies
\begin{equation}
\mathcal{L}^{\tau+1}
\;\le\;
\mathcal{L}^{\tau},
\qquad \forall \tau\ge 0,
\label{eq:mono}
\end{equation}
and the sequence $\{\mathcal{L}^{\tau}\}_{\tau\ge 0}$ converges.
Moreover, we have the explicit descent bound
\begin{equation}
\mathcal{L}^{\tau+1}
\;\le\;
\mathcal{L}^{\tau}
-
m(\delta_\mu^\tau)^2
-
s(\mathbf{t}^{\tau})\,(\alpha^\tau-\alpha^{\tau+1})^2,
\label{eq:descent_bound}
\end{equation}
where the $\mathbf{t}$-update may further decrease the objective.
\end{proposition}
\begin{proof}
Applying Lemma~\ref{lem:mu_decrease} gives a decrease after the $\mu$-update.
Then Lemma~\ref{lem:alpha_decrease} yields a decrease after the $\alpha$-update.
Finally, Lemma~\ref{lem:t_noninc} shows the $\mathbf{t}$-update is non-increasing.
Thus $\mathcal{L}^{\tau+1}\le\mathcal{L}^{\tau}$ for all $\tau$.
Since $\mathcal{L}^{\tau}\ge 0$, the monotone sequence converges.
Combining Lemma~\ref{lem:mu_decrease} and Lemma~\ref{lem:alpha_decrease} yields
Eq.~\ref{eq:descent_bound}; the $\mathbf{t}$-update can only decrease further.
\end{proof}

\paragraph{Extension to the full matrix.}
For a matrix $\mathbf{W}\in\mathbb{R}^{n\times m}$ with row-wise parameters
$\boldsymbol{\mu}\in\mathbb{R}^{n}$, $\boldsymbol{\alpha}\in\mathbb{R}^{n}$, and
$\mathbf{T}\in\{-1,0,1\}^{n\times m}$, the full Euclidean objective is the sum of
row-wise objectives in Eq.~\ref{eq:row_obj}. Therefore, the same monotonic
non-increase holds for the entire update procedure.

\subsection{Row-wise Normal Equations and Numerical Stabilization for Metric-Aware Relocation}
\label{app:e2m_rowwise_normal_eq}

This section provides a step-by-step derivation of the row-wise $2\times2$ linear system in Eq.~\ref{eq:2x2_main}
by optimizing $(\boldsymbol{\mu},\boldsymbol{\alpha})$ under the calibration-defined metric induced by
$\mathbf{S}\succeq \mathbf{0}$ while fixing $\mathbf{T}=\mathbf{T}^{(0)}$. We also describe practical numerical
stabilization used when $\mathbf{S}$ is ill-conditioned or the $2\times2$ system becomes nearly singular.

\paragraph{Row-wise decoupling from the trace form.}
Recall $\mathbf{R}=\mathbf{W}-\boldsymbol{\mu}\mathbf{1}^\top-\mathrm{diag}(\boldsymbol{\alpha})\,\mathbf{T}^{(0)}$
and the objective (Eq.~\ref{eq:app_L2_trace})
\begin{equation}
L_2(\boldsymbol{\mu},\boldsymbol{\alpha};\mathbf{T}^{(0)})
\;=\;
\mathrm{Tr}\!\left(\mathbf{R}\mathbf{S}\mathbf{R}^\top\right),
\qquad \mathbf{S}=\mathbf{S}^\top\succeq \mathbf{0}.
\label{eq:app_L2_trace_recall}
\end{equation}
Let $\mathbf{r}_i\in\mathbb{R}^{1\times m}$ be the $i$-th row of $\mathbf{R}$, and let $\mathbf{w}_i,\mathbf{t}_i$
denote the $i$-th rows of $\mathbf{W}$ and $\mathbf{T}^{(0)}$, respectively. Then
\begin{equation}
\mathbf{r}_i
=
\mathbf{w}_i-\mu_i\mathbf{1}^\top-\alpha_i\mathbf{t}_i.
\label{eq:app_row_residual_def}
\end{equation}
Using the identity $\mathrm{Tr}(\mathbf{A}\mathbf{A}^\top)=\sum_i \mathbf{a}_i\mathbf{a}_i^\top$ for row vectors
$\mathbf{a}_i$, we obtain
\begin{align}
\mathrm{Tr}(\mathbf{R}\mathbf{S}\mathbf{R}^\top)
&=\sum_{i=1}^{n} \left(\mathbf{R}\mathbf{S}\mathbf{R}^\top\right)_{ii}
=\sum_{i=1}^{n} \mathbf{r}_i\mathbf{S}\mathbf{r}_i^\top.
\label{eq:app_rowwise_decomp}
\end{align}
Hence the optimization decouples across rows:
\begin{equation}
\min_{\boldsymbol{\mu},\boldsymbol{\alpha}}
\mathrm{Tr}(\mathbf{R}\mathbf{S}\mathbf{R}^\top)
\quad\Longleftrightarrow\quad
\forall i:\ \min_{\mu_i,\alpha_i}\ f_i(\mu_i,\alpha_i),
\qquad
f_i(\mu_i,\alpha_i)=\mathbf{r}_i\mathbf{S}\mathbf{r}_i^\top.
\label{eq:app_fi_def}
\end{equation}

\paragraph{Step-by-step differentiation (two equivalent derivations).}
Fix a row $i$ and define
\begin{equation}
\mathbf{u}_i \triangleq \mathbf{r}_i
=
\mathbf{w}_i-\mu_i\mathbf{1}^\top-\alpha_i\mathbf{t}_i,
\qquad
f_i(\mu_i,\alpha_i)=\mathbf{u}_i\mathbf{S}\mathbf{u}_i^\top.
\label{eq:app_ui_def}
\end{equation}
We now compute $\partial f_i/\partial \mu_i$ and $\partial f_i/\partial \alpha_i$ in a fully explicit manner.

\emph{(A) Differential (matrix calculus) route.}
Since $\mathbf{S}=\mathbf{S}^\top$, the differential of the quadratic form satisfies
\begin{align}
\mathrm{d}f_i
&=\mathrm{d}(\mathbf{u}_i\mathbf{S}\mathbf{u}_i^\top)
=(\mathrm{d}\mathbf{u}_i)\mathbf{S}\mathbf{u}_i^\top + \mathbf{u}_i\mathbf{S}(\mathrm{d}\mathbf{u}_i)^\top
\notag\\
&=(\mathrm{d}\mathbf{u}_i)\mathbf{S}\mathbf{u}_i^\top + \mathbf{u}_i\mathbf{S}^\top(\mathrm{d}\mathbf{u}_i)^\top
=2(\mathrm{d}\mathbf{u}_i)\mathbf{S}\mathbf{u}_i^\top.
\label{eq:app_df_step1}
\end{align}
From Eq.~\ref{eq:app_ui_def}, we have
\begin{equation}
\mathrm{d}\mathbf{u}_i = -(\mathrm{d}\mu_i)\mathbf{1}^\top - (\mathrm{d}\alpha_i)\mathbf{t}_i.
\label{eq:app_du}
\end{equation}
Plugging Eq.~\ref{eq:app_du} into Eq.~\ref{eq:app_df_step1} yields
\begin{align}
\mathrm{d}f_i
&=2\Big(-(\mathrm{d}\mu_i)\mathbf{1}^\top-(\mathrm{d}\alpha_i)\mathbf{t}_i\Big)\mathbf{S}\mathbf{u}_i^\top
\notag\\
&=-2(\mathrm{d}\mu_i)\underbrace{\mathbf{1}^\top\mathbf{S}\mathbf{u}_i^\top}_{\text{scalar}}
-2(\mathrm{d}\alpha_i)\underbrace{\mathbf{t}_i\mathbf{S}\mathbf{u}_i^\top}_{\text{scalar}}.
\label{eq:app_df_step2}
\end{align}
By identification of coefficients of $\mathrm{d}\mu_i$ and $\mathrm{d}\alpha_i$, we obtain
\begin{equation}
\frac{\partial f_i}{\partial \mu_i}=-2\,\mathbf{1}^\top\mathbf{S}\mathbf{u}_i^\top,
\qquad
\frac{\partial f_i}{\partial \alpha_i}=-2\,\mathbf{t}_i\mathbf{S}\mathbf{u}_i^\top.
\label{eq:app_grad_results}
\end{equation}

\emph{(B) Fully expanded scalar route (optional but explicit).}
Write $f_i=\sum_{p=1}^{m}\sum_{q=1}^{m} u_{ip}S_{pq}u_{iq}$, where $u_{ip}=w_{ip}-\mu_i-\alpha_i t_{ip}$.
Then
\[
\frac{\partial f_i}{\partial \alpha_i}
=
\sum_{p,q} \Big(\frac{\partial u_{ip}}{\partial \alpha_i}S_{pq}u_{iq}
+u_{ip}S_{pq}\frac{\partial u_{iq}}{\partial \alpha_i}\Big)
=
\sum_{p,q} \big(-t_{ip}S_{pq}u_{iq}-u_{ip}S_{pq}t_{iq}\big).
\]
Using symmetry $S_{pq}=S_{qp}$ and re-indexing the second term gives
$\frac{\partial f_i}{\partial \alpha_i}=-2\sum_{p,q} t_{ip}S_{pq}u_{iq}=-2\,\mathbf{t}_i\mathbf{S}\mathbf{u}_i^\top$,
matching Eq.~\ref{eq:app_grad_results}. The derivation for $\partial f_i/\partial \mu_i$ is identical with
$\partial u_{ip}/\partial \mu_i=-1$, yielding $-2\,\mathbf{1}^\top\mathbf{S}\mathbf{u}_i^\top$.

\paragraph{From stationarity to the $2\times2$ linear system.}
Setting Eq.~\ref{eq:app_grad_results} to zero yields the stationarity conditions
\begin{equation}
\mathbf{t}_i\mathbf{S}\mathbf{u}_i^\top=0,
\qquad
\mathbf{1}^\top\mathbf{S}\mathbf{u}_i^\top=0.
\label{eq:app_stationarity}
\end{equation}
Substitute $\mathbf{u}_i^\top=\mathbf{w}_i^\top-\mu_i\mathbf{1}-\alpha_i\mathbf{t}_i^\top$ into Eq.~\ref{eq:app_stationarity}:
\begin{align}
\mathbf{t}_i\mathbf{S}\big(\mathbf{w}_i^\top-\mu_i\mathbf{1}-\alpha_i\mathbf{t}_i^\top\big)&=0,
\label{eq:app_stat_t}
\\
\mathbf{1}^\top\mathbf{S}\big(\mathbf{w}_i^\top-\mu_i\mathbf{1}-\alpha_i\mathbf{t}_i^\top\big)&=0.
\label{eq:app_stat_1}
\end{align}
Expand each equation and collect coefficients of $\alpha_i$ and $\mu_i$:
\begin{align}
(\mathbf{t}_i\mathbf{S}\mathbf{t}_i^\top)\alpha_i + (\mathbf{t}_i\mathbf{S}\mathbf{1})\mu_i
&=\mathbf{t}_i\mathbf{S}\mathbf{w}_i^\top,
\label{eq:app_collect1}
\\
(\mathbf{1}^\top\mathbf{S}\mathbf{t}_i^\top)\alpha_i + (\mathbf{1}^\top\mathbf{S}\mathbf{1})\mu_i
&=\mathbf{1}^\top\mathbf{S}\mathbf{w}_i^\top.
\label{eq:app_collect2}
\end{align}
Writing Eqs.~\ref{eq:app_collect1}--\ref{eq:app_collect2} in matrix form gives the $2\times2$ normal equations:
\begin{equation}
\begin{bmatrix}
\mathbf{t}_i \mathbf{S}\mathbf{t}_i^\top & \mathbf{t}_i \mathbf{S}\mathbf{1} \\
\mathbf{1}^\top \mathbf{S}\mathbf{t}_i^\top & \mathbf{1}^\top \mathbf{S}\mathbf{1}
\end{bmatrix}
\begin{bmatrix}
\alpha_i \\
\mu_i
\end{bmatrix}
=
\begin{bmatrix}
\mathbf{t}_i \mathbf{S}\mathbf{w}_i^\top \\
\mathbf{1}^\top \mathbf{S}\mathbf{w}_i^\top
\end{bmatrix},
\label{eq:app_2x2_system}
\end{equation}
which matches Eq.~\ref{eq:2x2_main} in the main text.

\paragraph{Closed-form solution via Cramer's rule.}
Define the scalars
\begin{equation}
a_i=\mathbf{t}_i \mathbf{S}\mathbf{t}_i^\top,\quad
b_i=\mathbf{t}_i \mathbf{S}\mathbf{1},\quad
c=\mathbf{1}^\top \mathbf{S}\mathbf{1},\quad
d_i=\mathbf{t}_i \mathbf{S}\mathbf{w}_i^\top,\quad
e_i=\mathbf{1}^\top \mathbf{S}\mathbf{w}_i^\top.
\label{eq:app_scalar_defs}
\end{equation}
Then Eq.~\ref{eq:app_2x2_system} becomes
$\begin{bmatrix} a_i & b_i \\ b_i & c \end{bmatrix}\begin{bmatrix}\alpha_i\\ \mu_i\end{bmatrix}=
\begin{bmatrix} d_i\\ e_i\end{bmatrix}$.
Its determinant is
\begin{equation}
D_i = a_i c - b_i^2.
\label{eq:app_det}
\end{equation}
If $D_i\neq 0$, the unique solution is
\begin{equation}
\alpha_i^{*}=\frac{d_i c-b_i e_i}{D_i},
\qquad
\mu_i^{*}=\frac{a_i e_i-b_i d_i}{D_i}.
\label{eq:app_closed_form}
\end{equation}

\paragraph{Uniqueness and degeneracy (why $D_i$ can be small).}
The matrix in Eq.~\ref{eq:app_2x2_system} is the Gram matrix of the two vectors
$\mathbf{t}_i^\top$ and $\mathbf{1}$ under the $\mathbf{S}$-induced inner product
$\langle \mathbf{x},\mathbf{y}\rangle_{\mathbf{S}}=\mathbf{x}^\top\mathbf{S}\mathbf{y}$:
\[
G_i=
\begin{bmatrix}
\langle \mathbf{t}_i^\top,\mathbf{t}_i^\top\rangle_{\mathbf{S}} &
\langle \mathbf{t}_i^\top,\mathbf{1}\rangle_{\mathbf{S}}\\
\langle \mathbf{1},\mathbf{t}_i^\top\rangle_{\mathbf{S}} &
\langle \mathbf{1},\mathbf{1}\rangle_{\mathbf{S}}
\end{bmatrix}
=
\begin{bmatrix}
a_i & b_i\\
b_i & c
\end{bmatrix}.
\]
Since $\mathbf{S}\succeq \mathbf{0}$, we have $a_i\ge 0$, $c\ge 0$, and $D_i=\det(G_i)\ge 0$.
Moreover, $D_i=0$ if and only if $\mathbf{t}_i^\top$ and $\mathbf{1}$ are colinear in the seminorm induced by $\mathbf{S}$
(i.e., linearly dependent modulo the nullspace of $\mathbf{S}$ on $\mathrm{span}\{\mathbf{t}_i^\top,\mathbf{1}\}$).
A practically important degenerate case is $\mathbf{t}_i=\mathbf{0}$ (all ternary zeros), which implies
$a_i=b_i=d_i=0$ and hence $D_i=0$; then $f_i$ does not depend on $\alpha_i$ and the optimal solution reduces to
\begin{equation}
\mu_i^{*}=\frac{\mathbf{1}^\top\mathbf{S}\mathbf{w}_i^\top}{\mathbf{1}^\top\mathbf{S}\mathbf{1}},
\qquad
\alpha_i^{*}=0.
\label{eq:app_degenerate_allzero}
\end{equation}

\paragraph{Practical numerical stabilization.}
With limited calibration samples, $\mathbf{S}$ can be ill-conditioned, making $D_i$ numerically close to zero even when
theoretically positive. We apply moment regularization
\begin{equation}
\mathbf{S}\leftarrow \mathbf{S}+\varepsilon \mathbf{I},
\qquad \varepsilon>0,
\label{eq:app_S_reg}
\end{equation}
which improves conditioning and increases the Gram determinant $D_i$ away from machine precision.
In our implementation, $\varepsilon$ is chosen as a small fraction of the average diagonal mass,
e.g., $\varepsilon=\lambda\cdot \mathrm{Tr}(\mathbf{S})/m$ with $\lambda\in[10^{-6},10^{-3}]$.

When $|D_i|$ remains very small for rare rows, we safely solve Eq.~\ref{eq:app_2x2_system} using a stable routine
(e.g., a direct $2\times2$ solve in double precision) and apply a conservative fallback:
if $|D_i|<\tau$ (a small threshold), we use the Euclidean warm-start values $(\mu_i^{(0)},\alpha_i^{(0)})$ rather than
the potentially unstable closed form. These safeguards preserve the formulation while improving robustness under extreme
activation statistics.

\paragraph{Reconstruction.}
After solving all rows, we aggregate $\boldsymbol{\mu}^{*}$ and $\boldsymbol{\alpha}^{*}$ and reconstruct the relocated
E2M-ATQ weights as
\begin{equation}
\bar{\mathbf{W}}
=
\boldsymbol{\mu}^{*}\mathbf{1}^\top
+
\mathrm{diag}(\boldsymbol{\alpha}^{*})\,\mathbf{T}^{(0)}.
\label{eq:app_reconstruct_Wbar}
\end{equation}

\subsection{Storage and Computational Complexity of Kronecker-Structured Orthogonal Transforms}
\label{app:kron_complexity}

We analyze the storage and computational cost of parameterizing a large auxiliary transform
$\mathbf{R}\in\mathbb{R}^{m\times m}$ as a Kronecker product of two smaller orthogonal factors.
This provides a practical alternative to learning a dense $m\times m$ transform.

\subsubsection{Dense transform is prohibitive at LLM dimensions}
A dense auxiliary matrix $\mathbf{R}\in\mathbb{R}^{m\times m}$ requires storing $m^2$ parameters.
Applying it to a vector $\mathbf{v}\in\mathbb{R}^{1\times m}$ costs $O(m^2)$ multiply-adds.
For LLM hidden sizes $m\in\{4096,\,5120,\,8192,\ldots\}$, both the $m^2$ storage and $m^2$ compute
become prohibitive.

\subsubsection{Kronecker parameterization}
We parameterize the auxiliary transform as
\begin{equation}
\mathbf{R}=\mathbf{R}_1\otimes \mathbf{R}_2,\qquad
\mathbf{R}_1\in\mathbb{R}^{n_1\times n_1},\quad
\mathbf{R}_2\in\mathbb{R}^{n_2\times n_2},\quad
n_1n_2=m,
\label{eq:app_kron_param}
\end{equation}
where each factor is orthogonal: $\mathbf{R}_1^\top\mathbf{R}_1=\mathbf{I}_{n_1}$ and
$\mathbf{R}_2^\top\mathbf{R}_2=\mathbf{I}_{n_2}$, which implies $\mathbf{R}^\top\mathbf{R}=\mathbf{I}_m$.

\paragraph{Storage.}
A dense $\mathbf{R}$ stores $m^2$ entries.
In contrast, the Kronecker form stores only $n_1^2+n_2^2$ entries:
\begin{equation}
\underbrace{m^2}_{\text{dense storage}}
\quad\longrightarrow\quad
\underbrace{(n_1^2+n_2^2)}_{\text{Kronecker storage}}.
\label{eq:app_storage}
\end{equation}
When $n_1\approx n_2\approx \sqrt{m}$, we have $n_1^2+n_2^2\approx 2m$, i.e., nearly linear storage.

\paragraph{Concrete example.}
For $m=4096$, a dense matrix requires
\begin{equation}
m^2 = 4096^2 = 16{,}777{,}216
\end{equation}
stored entries.
Choosing $n_1=n_2=64$ (since $64\cdot 64=4096$) yields only
\begin{equation}
n_1^2+n_2^2 = 64^2+64^2 = 8192
\end{equation}
entries.
The storage reduction factor is therefore
\begin{equation}
\frac{m^2}{n_1^2+n_2^2}
=
\frac{16{,}777{,}216}{8192}
=
2048.
\label{eq:app_storage_ratio_4096}
\end{equation}
Thus, the Kronecker parameterization is $2048\times$ smaller in storage than a dense $4096\times 4096$ transform.

\subsubsection{Fast application via reshape--multiply--vectorize}
Let $\mathbf{v}\in\mathbb{R}^{1\times m}$.
Reshape it into a matrix $\mathbf{V}\in\mathbb{R}^{n_1\times n_2}$ such that
$\mathrm{vec}(\mathbf{V})=\mathbf{v}^\top$.
Using a standard Kronecker identity, we obtain
\begin{equation}
\mathbf{v}(\mathbf{R}_1\otimes \mathbf{R}_2)
=
\mathrm{vec}\!\left(\mathbf{R}_2^{\top}\mathbf{V}\mathbf{R}_1\right)^{\top}.
\label{eq:app_kron_identity}
\end{equation}
Therefore, we can apply $\mathbf{R}$ without forming an $m\times m$ matrix, by computing
\begin{equation}
\mathbf{Y}=\mathbf{V}\mathbf{R}_1 \in \mathbb{R}^{n_1\times n_2},
\qquad
\mathbf{Z}=\mathbf{R}_2^\top \mathbf{Y} \in \mathbb{R}^{n_1\times n_2},
\qquad
\mathbf{v}'=\mathrm{vec}(\mathbf{Z})^\top \in \mathbb{R}^{1\times m}.
\end{equation}

\paragraph{Compute complexity.}
The cost of $\mathbf{V}\mathbf{R}_1$ is $O(n_1 n_2 \cdot n_1)=O(n_1^2n_2)=O(m n_1)$.
The cost of $\mathbf{R}_2^\top\mathbf{Y}$ is $O(n_2 n_1 \cdot n_2)=O(n_2^2n_1)=O(m n_2)$.
Hence the total is
\begin{equation}
O(m n_1) + O(m n_2) = O\big(m(n_1+n_2)\big).
\label{eq:app_compute}
\end{equation}
With $n_1\approx n_2\approx \sqrt{m}$, this becomes $O(m^{3/2})$, substantially smaller than $O(m^2)$.

\paragraph{Concrete compute example ($m=4096$, $n_1=n_2=64$).}
A dense multiplication $\mathbf{v}\mathbf{R}$ uses on the order of $m^2=16.8$ million multiply-adds.
With the Kronecker form, the two multiplies cost on the order of
\begin{equation}
m(n_1+n_2) = 4096(64+64)=524{,}288,
\end{equation}
which is about $16{,}777{,}216 / 524{,}288 \approx 32\times$ fewer operations (up to constant factors).
This makes a structured orthogonal auxiliary transform practical at LLM scale.

\subsection{Why Orthogonal Mixing Suppresses Activation Outliers}
\label{app:act_outlier}

We provide a step-by-step argument that an orthogonal auxiliary transform can reduce
activation outliers by spreading energy across coordinates, while preserving the total energy.
Although our method later uses a structured orthogonal matrix, the argument below applies to any
sufficiently mixing orthogonal $\mathbf{R}\in\mathbb{R}^{m\times m}$.

\subsubsection{Outliers and the peak-to-RMS ratio}
Let $\mathbf{x}\in\mathbb{R}^{1\times m}$ denote a token-wise activation vector.
A common operational notion of ``outliers'' is that a few coordinates are much larger than
the typical scale.
We quantify this with the peak-to-RMS ratio
\begin{equation}
\mathrm{RMS}(\mathbf{x}) := \frac{\|\mathbf{x}\|_2}{\sqrt{m}},
\qquad
\kappa(\mathbf{x}) := \frac{\|\mathbf{x}\|_\infty}{\mathrm{RMS}(\mathbf{x})}
= \frac{\sqrt{m}\,\|\mathbf{x}\|_\infty}{\|\mathbf{x}\|_2}.
\label{eq:app_kappa_def}
\end{equation}
Large $\kappa(\mathbf{x})$ indicates spiky activations: the maximum coordinate is large relative to
the average energy per dimension.

\subsubsection{Orthogonal transforms preserve energy}
Assume $\mathbf{R}$ is orthogonal: $\mathbf{R}^\top\mathbf{R}=\mathbf{I}_m$.
Then, for $\mathbf{y}=\mathbf{x}\mathbf{R}$,
\begin{equation}
\|\mathbf{y}\|_2 = \|\mathbf{x}\mathbf{R}\|_2 = \|\mathbf{x}\|_2.
\label{eq:app_norm_preserve}
\end{equation}
Thus orthogonal mixing does not change the total activation energy, but can change how energy is
distributed across coordinates.

\subsubsection{A probabilistic bound on the maximum coordinate after mixing}
We now show that for a random orthogonal $\mathbf{R}$, each coordinate of $\mathbf{y}=\mathbf{x}\mathbf{R}$
concentrates around scale $\|\mathbf{x}\|_2/\sqrt{m}$, and the maximum over coordinates is only
logarithmically larger.

\paragraph{Step 1: One coordinate is an inner product with a random unit vector.}
Let $\mathbf{R}_j\in\mathbb{R}^{m}$ be the $j$-th column of $\mathbf{R}$.
For orthogonal $\mathbf{R}$ sampled uniformly at random (Haar), $\mathbf{R}_j$ is uniformly distributed on the
unit sphere $S^{m-1}$.
The $j$-th coordinate of $\mathbf{y}$ is
\begin{equation}
y_j = (\mathbf{x}\mathbf{R})_j = \langle \mathbf{x}^\top, \mathbf{R}_j\rangle.
\label{eq:app_inner_prod}
\end{equation}

\paragraph{Step 2: Spherical concentration for a fixed vector.}
Fix $\mathbf{x}$ and let $\mathbf{u}$ be uniform on $S^{m-1}$.
A standard concentration inequality on the sphere gives, for any $t>0$,
\begin{equation}
\Pr\!\left(\big|\langle \mathbf{x}^\top, \mathbf{u}\rangle\big|
\ge t\,\frac{\|\mathbf{x}\|_2}{\sqrt{m}}\right)
\le 2\exp\!\left(-\frac{(m-2)t^2}{2}\right).
\label{eq:app_sphere_tail}
\end{equation}
Applying Eq.~\ref{eq:app_sphere_tail} to $\mathbf{u}=\mathbf{R}_j$ yields the same bound for each $y_j$.

\paragraph{Step 3: Union bound over all coordinates.}
Using a union bound over $j=1,\ldots,m$,
\begin{equation}
\Pr\!\left(\|\mathbf{y}\|_\infty
\ge t\,\frac{\|\mathbf{x}\|_2}{\sqrt{m}}\right)
\le 2m\exp\!\left(-\frac{(m-2)t^2}{2}\right).
\label{eq:app_union}
\end{equation}
Set the right-hand side to $\delta\in(0,1)$ and solve for $t$:
\begin{equation}
2m\exp\!\left(-\frac{(m-2)t^2}{2}\right) \le \delta
\quad\Longrightarrow\quad
t \ge \sqrt{\frac{2\log(2m/\delta)}{m-2}}.
\label{eq:app_t_choice}
\end{equation}
Therefore, with probability at least $1-\delta$,
\begin{equation}
\|\mathbf{x}\mathbf{R}\|_\infty
=
\|\mathbf{y}\|_\infty
\le
\|\mathbf{x}\|_2\sqrt{\frac{2\log(2m/\delta)}{m-2}}.
\label{eq:app_infty_bound}
\end{equation}

\paragraph{Step 4: Peak-to-RMS ratio becomes logarithmic.}
Combining Eq.~\ref{eq:app_norm_preserve} and Eq.~\ref{eq:app_infty_bound}, we obtain
\begin{equation}
\kappa(\mathbf{x}\mathbf{R})
=
\frac{\sqrt{m}\,\|\mathbf{x}\mathbf{R}\|_\infty}{\|\mathbf{x}\mathbf{R}\|_2}
\le
\sqrt{m}\cdot
\sqrt{\frac{2\log(2m/\delta)}{m-2}}
=
\sqrt{\frac{2m}{m-2}\log\!\left(\frac{2m}{\delta}\right)}.
\label{eq:app_kappa_bound}
\end{equation}
For large $m$, this scales as $\kappa(\mathbf{x}\mathbf{R}) \lesssim \sqrt{2\log(2m/\delta)}$,
which grows only logarithmically in $m$.
In contrast, the worst-case $\kappa(\mathbf{x})$ can be as large as $\sqrt{m}$ when energy concentrates
in one coordinate.

\paragraph{Step 5: Existence (probabilistic method).}
Since a random orthogonal $\mathbf{R}$ satisfies Eq.~\ref{eq:app_infty_bound} with probability $1-\delta$,
there must exist at least one orthogonal matrix $\mathbf{R}$ achieving that bound.
Hence orthogonal mixing is \emph{theoretically sufficient} to suppress extreme activation peaks.

\subsubsection{Implication for low-bit activation quantization}
Consider symmetric uniform quantization with $b$ bits and per-tensor scale
$s(\mathbf{x})=\|\mathbf{x}\|_\infty$ (a common conservative choice).
The quantization step size is $\Delta = 2s(\mathbf{x})/(2^b-1)$.
Under standard high-resolution quantization modeling, the mean squared error per coordinate scales as
$O(\Delta^2)$, so reducing $s(\mathbf{x})$ directly reduces quantization error.

Applying Eq.~\ref{eq:app_infty_bound} shows that, after orthogonal mixing,
$s(\mathbf{x}\mathbf{R})=\|\mathbf{x}\mathbf{R}\|_\infty$ is upper-bounded by a logarithmic factor times
$\|\mathbf{x}\|_2/\sqrt{m}$, while $\|\mathbf{x}\|_2$ is unchanged.
Therefore, orthogonal mixing can significantly shrink the required dynamic range (and thus $\Delta$),
improving low-bit activation quantization by reducing the dominance of outliers.

\subsection{Zero-Peak Mass Regularization for TriGMM}
\label{app:kotms-zero-constraint}

The TriGMM objective in Eq.~\ref{eq:loss} is used as a codebook-aligned attraction field rather than a literal density fit. 
Without additional constraints, directly minimizing the tri-modal negative log-likelihood may admit degenerate configurations, e.g., $c_i\to 0$ (peak collapse) or an imbalanced assignment where a single component dominates, which undermines the intended ternary-aligned tri-modal shaping.

\paragraph{Posterior responsibilities.}
For channel $i$ and entry $z_{ij}$, define the component means
$\mu_{i+}=+c_i$, $\mu_{i0}=0$, and $\mu_{i-}=-c_i$. 
Let $\phi(\cdot;\mu,\sigma^2)$ denote the Gaussian density and let the mixture weights satisfy
$\pi_{+}=\pi_{-}=(1-\pi_0)/2$ with $\pi_0\in(0,1)$.
The posterior responsibility of component $k\in\{+,0,-\}$ is
\begin{equation}
\label{eq:app_resp}
r_{ijk}
=
\frac{
\pi_k\,\phi(z_{ij};\mu_{ik},\sigma_i^2)
}{
\sum_{\ell\in\{+,0,-\}}
\pi_\ell\,\phi(z_{ij};\mu_{i\ell},\sigma_i^2)
}.
\end{equation}
We measure the \emph{average zero-peak responsibility} in channel $i$ as
\begin{equation}
\label{eq:app_zero_avg}
\bar r_{i0}
=
\frac{1}{m}\sum_{j=1}^{m} r_{ij0}.
\end{equation}

\paragraph{Zero-peak mass constraint.}
To explicitly control the mass assigned to the ternary zero attractor and prevent collapse,
we steer $\bar r_{i0}$ toward a target sparsity ratio $\rho\in(0,1)$ via
\begin{equation}
\label{eq:app_zero_reg}
L_{\mathrm{zero}}
=
\frac{1}{n}\sum_{i=1}^{n}\big(\bar r_{i0}-\rho\big)^2,
\qquad
L_{\mathrm{shape}}
=
L_{\mathrm{TriGMM}}+\beta\,L_{\mathrm{zero}},
\end{equation}
where $\beta>0$ controls the strength of the regularizer. 
This constraint makes the zero-peak mass an explicit and controllable structural prior, consistent with the support-set semantics of ternary quantization (i.e., entries attracted to the zero peak are encouraged to map to the $0$ state), while empirically preventing degenerate tri-modal shaping such as $c_i\to 0$.

\subsection{Small-Variance Limit: TriGMM as a Soft Projection}
\label{app:kotms-soft-proj}

This section justifies the interpretation that minimizing $L_{\mathrm{TriGMM}}$ induces a soft projection of each scalar $z_{ij}$ onto the ternary-aligned attractor set $\{-c_i,0,+c_i\}$.

\paragraph{Start from the per-entry negative log-mixture.}
Fix a channel $i$ and a scalar $z=z_{ij}$. Consider the negative log-mixture term
\begin{equation}
\label{eq:app_nlogmix}
\ell_i(z)
:=
-\log\!\Big(
\pi_{+}\phi(z;+c_i,\sigma_i^2)
+
\pi_{0}\phi(z;0,\sigma_i^2)
+
\pi_{-}\phi(z;-c_i,\sigma_i^2)
\Big).
\end{equation}
Using
$\phi(z;\mu,\sigma^2)=\frac{1}{\sqrt{2\pi\sigma^2}}\exp\!\big(-\frac{(z-\mu)^2}{2\sigma^2}\big)$,
we can rewrite the mixture as
\begin{equation}
\label{eq:app_mix_exp}
\pi_{+}\phi(z;+c_i,\sigma_i^2)
+
\pi_{0}\phi(z;0,\sigma_i^2)
+
\pi_{-}\phi(z;-c_i,\sigma_i^2)
=
\frac{1}{\sqrt{2\pi\sigma_i^2}}
\sum_{s\in\{-c_i,0,+c_i\}}
\pi_s\,
\exp\!\Big(
-\frac{(z-s)^2}{2\sigma_i^2}
\Big),
\end{equation}
where we use $\pi_{+c_i}=\pi_{+}$, $\pi_{0}=\pi_{0}$, and $\pi_{-c_i}=\pi_{-}$ for notational convenience.
Substituting Eq.~\ref{eq:app_mix_exp} into Eq.~\ref{eq:app_nlogmix} yields
\begin{equation}
\label{eq:app_nlog_lse}
\ell_i(z)
=
\frac{1}{2}\log(2\pi\sigma_i^2)
-
\log\!\left(
\sum_{s\in\{-c_i,0,+c_i\}}
\pi_s\,
\exp\!\Big(
-\frac{(z-s)^2}{2\sigma_i^2}
\Big)
\right).
\end{equation}

\paragraph{Log-sum-exp domination in the small-variance regime.}
When $\sigma_i^2$ is small, the sum inside the logarithm is dominated by the largest exponential term, i.e., the attractor $s$ closest to $z$ in squared distance.
Formally, applying the standard log-sum-exp approximation gives
\begin{equation}
\label{eq:app_lse_approx}
-\log\!\left(
\sum_{s}
\pi_s\,
\exp\!\Big(
-\frac{(z-s)^2}{2\sigma_i^2}
\Big)
\right)
\approx
\min_{s\in\{-c_i,0,+c_i\}}
\left\{
\frac{(z-s)^2}{2\sigma_i^2}
-\log \pi_s
\right\}.
\end{equation}
If the mixture weights are fixed and not extreme (e.g., symmetric side peaks $\pi_{+}=\pi_{-}$), the term $-\log\pi_s$ contributes only an additive bias compared to the dominant quadratic term. 
Combining Eqs.~\ref{eq:app_nlog_lse}--\ref{eq:app_lse_approx}, we obtain
\begin{equation}
\label{eq:app_softproj}
\ell_i(z)
\approx
\frac{1}{2\sigma_i^2}
\min_{s\in\{-c_i,0,+c_i\}}
(z-s)^2
+\mathrm{const}.
\end{equation}

\paragraph{Implication for TriGMM minimization.}
Summing Eq.~\ref{eq:app_softproj} over all entries $(i,j)$ shows that minimizing $L_{\mathrm{TriGMM}}$ approximately minimizes the squared distance of each $z_{ij}$ to the ternary-aligned attractor set $\{-c_i,0,+c_i\}$, i.e., it performs a differentiable \emph{soft projection} prior to the subsequent hard ternary assignment.
This explains why TriGMM shaping is geometrically consistent with the hard ternary projection step used by our quantizer.

\subsection{Why Modeling Only Adjacent-Layer Interactions is Sufficient}
\label{app:adjacent_only}

This subsection provides a theoretical and practical rationale for restricting second-order interaction terms to
adjacent layers. Under mild stability and smoothness assumptions around the reference configuration
$\mathbf{b}_{\max}$, we show that pairwise couplings between non-adjacent layers admit an explicit upper bound that
decays with layer distance. We further clarify the optimization implications: while higher-order adjacent interactions
(e.g., triplets) still preserve a chain structure and remain exactly solvable by dynamic programming (DP) with an
augmented state, introducing interactions across arbitrary layer pairs yields a dense factor graph with high treewidth,
which precludes the efficient chain-structured DP used in our solver. Finally, estimating dense interactions is
significantly more expensive and less statistically stable under limited calibration data.

\paragraph{Notation.}
Consider a network with $L$ layers and candidate activation bitwidths $\mathcal{B}=\{2,4,6,8\}$.
A layer-wise configuration is denoted by $\mathbf{b}=(b_1,\ldots,b_L)\in\mathcal{B}^L$.
Let $b_{\max}=8$ and $\mathbf{b}_{\max}=(b_{\max},\ldots,b_{\max})$ be the reference configuration.
We adopt the validation average per-token negative log-likelihood (NLL) as the value function:
\begin{equation}
v_{\mathrm{NLL}}(\mathbf{b})
:=\mathbb{E}_{u}\big[\ell(x_L(u;\mathbf{b}))\big],
\label{eq:app_vnll_def_new}
\end{equation}
where $u$ indexes validation sequences, $x_L(u;\mathbf{b})$ is the final hidden state, and $\ell(\cdot)$ is the per-token NLL.
We analyze the increment relative to the reference model:
\begin{equation}
\Delta v(\mathbf{b})
:=v_{\mathrm{NLL}}(\mathbf{b})-v_{\mathrm{NLL}}(\mathbf{b}_{\max}).
\label{eq:app_delta_v_new}
\end{equation}
Recall the first-order cost:
\begin{equation}
C_\ell(b)
=
v_{\mathrm{NLL}}\!\Big(b_{\ell}=b,\; b_{k\neq \ell}=b_{\max}\Big)
-
v_{\mathrm{NLL}}(\mathbf{b}_{\max}),
\label{eq:app_C_def_new}
\end{equation}
and the adjacent interaction cost:
\begin{equation}
K_{\ell-1,\ell}(b',b)
=
v_{\mathrm{NLL}}\!\Big(b_{\ell-1}=b',\; b_{\ell}=b,\; b_{k\notin\{\ell-1,\ell\}}=b_{\max}\Big)
-
v_{\mathrm{NLL}}(\mathbf{b}_{\max})
-
C_{\ell-1}(b')
-
C_{\ell}(b).
\label{eq:app_K_adj_def_new}
\end{equation}
For analysis, define the general (not necessarily adjacent) pairwise interaction term for $1\le i<j\le L$:
\begin{equation}
K_{i,j}(b_i,b_j)
=
v_{\mathrm{NLL}}\!\Big(b_{i}=b_i,\; b_{j}=b_j,\; b_{k\notin\{i,j\}}=b_{\max}\Big)
-
v_{\mathrm{NLL}}(\mathbf{b}_{\max})
-
C_{i}(b_i)
-
C_{j}(b_j).
\label{eq:app_K_general_def_new}
\end{equation}

\paragraph{Perturbation model around $\mathbf{b}_{\max}$.}
Let $x_\ell$ denote the input to layer $\ell$ along the forward pass.
Along the reference trajectory $\mathbf{b}_{\max}$:
\begin{equation}
x_{\ell+1}^{0}=F_\ell(x_\ell^{0}), \qquad \ell=1,\ldots,L-1.
\label{eq:app_baseline_forward_new}
\end{equation}
Lowering the activation precision at layer $\ell$ from $b_{\max}$ to $b_\ell$ perturbs the layer mapping, modeled as:
\begin{equation}
x_{\ell+1}=F_\ell(x_\ell)+e_\ell(x_\ell;b_\ell), \qquad e_\ell(\cdot;b_{\max})\equiv 0,
\label{eq:app_perturbed_forward_new}
\end{equation}
and define the deviation from the reference trajectory as $\delta_\ell:=x_\ell-x_\ell^{0}$.

\paragraph{Assumptions (effective stability and smoothness).}
We state assumptions in a form that does not require every individual layer to be contractive; instead, we assume an
effective bound on Jacobian products along the reference trajectory.

\textbf{(A1) Bounded Jacobian products.}
Let $J_\ell(x):=\partial F_\ell(x)/\partial x$ and define the Jacobian product along the reference trajectory:
\begin{equation}
A_{p\to q}
:=
J_{q-1}(x_{q-1}^{0})\,J_{q-2}(x_{q-2}^{0})\cdots J_{p}(x_{p}^{0}),
\qquad 1\le p<q\le L.
\label{eq:app_A_def_new}
\end{equation}
Assume there exists $\rho\in(0,1)$ such that:
\begin{equation}
\|A_{p\to q}\|\le \rho^{\,q-p}
\quad \text{for all } 1\le p<q\le L,
\label{eq:assump_product_rho_new}
\end{equation}
where $\|\cdot\|$ is the operator norm.

\textbf{(A2) Bounded quantization perturbation.}
For any $b\in\mathcal{B}$ and for $x$ in the relevant neighborhood:
\begin{equation}
\|e_\ell(x;b)\|\le \varepsilon_\ell(b).
\label{eq:assump_eps_new}
\end{equation}

\textbf{(A3) Second- and third-order smoothness of the loss.}
Assume $\ell(x_L)$ is three-times differentiable in $x_L$, and there exist constants $M$ and $T$ such that:
\begin{equation}
\|\nabla^2 \ell(x_L)\|\le M,
\qquad
\|\nabla^3 \ell(x_L)\|\le T
\label{eq:assump_M_T_new}
\end{equation}
throughout the same neighborhood.

\paragraph{Lemma 1 (Propagation of perturbations).}
Under Eq.~\ref{eq:app_perturbed_forward_new} and Eq.~\ref{eq:assump_eps_new}, the deviation obeys:
\begin{equation}
\|\delta_{\ell+1}\|
\le
\|J_\ell(\xi_\ell)\|\cdot\|\delta_\ell\|+\varepsilon_\ell(b_\ell),
\label{eq:lem1_recursion_general_new}
\end{equation}
for some $\xi_\ell$ between $x_\ell^{0}$ and $x_\ell$.
In particular, if only layer $i$ is perturbed (all other layers at $b_{\max}$), then:
\begin{equation}
\|\delta_{j}^{(i)}\|
\le
\|A_{i\to j}\|\cdot \varepsilon_i(b_i)
\le
\rho^{\,j-i}\,\varepsilon_i(b_i).
\label{eq:lem1_single_layer_attenuation_new}
\end{equation}

\paragraph{Lemma 2 (Decay of non-adjacent pairwise interactions).}
Fix $1\le i<j\le L$ and consider the four configurations in Eq.~\ref{eq:app_K_general_def_new}:
the baseline $\mathbf{b}_{\max}$, the two single-layer perturbations at $i$ and $j$, and the joint perturbation at $(i,j)$.
Let $\delta x_L^{(i)}$ and $\delta x_L^{(j)}$ denote the baseline-referenced changes in the final hidden state when
perturbing only layer $i$ or only layer $j$, respectively.
Under the linearization around the reference trajectory:
\begin{equation}
\delta x_L^{(i)} \approx A_{i\to L}\,e_i(x_i^{0};b_i),
\qquad
\delta x_L^{(j)} \approx A_{j\to L}\,e_j(x_j^{0};b_j).
\label{eq:lem2_linearized_deltas_new}
\end{equation}
The definition in Eq.~\ref{eq:app_K_general_def_new} cancels the baseline and the two single-layer effects, leaving
the cross term at second order.
A third-order Taylor expansion of $\ell(\cdot)$ around $x_L^{0}$ yields:
\begin{equation}
K_{i,j}(b_i,b_j)
=
\mathbb{E}_u\!\left[
(\delta x_L^{(i)})^\top \nabla^2 \ell(\tilde{x}_u)\,(\delta x_L^{(j)})
\right]
+\; R_{i,j},
\label{eq:lem2_decomposition_new}
\end{equation}
where $\tilde{x}_u$ lies between the reference and the joint-perturbed final state, and $R_{i,j}$ collects third-order
(and higher) remainder terms.

Using Eq.~\ref{eq:assump_M_T_new} and Eq.~\ref{eq:assump_product_rho_new} to bound the linearized perturbations, we have:
\begin{equation}
\big|K_{i,j}(b_i,b_j)\big|
\le
\frac{M}{2}\,\|\delta x_L^{(i)}\|\,\|\delta x_L^{(j)}\|
+\big|R_{i,j}\big|.
\label{eq:lem2_basic_bound_new}
\end{equation}
By splitting $A_{i\to L}=A_{j\to L}A_{i\to j}$:
\begin{equation}
\|\delta x_L^{(i)}\|
\le
\|A_{j\to L}\|\,\|A_{i\to j}\|\,\varepsilon_i(b_i),
\qquad
\|\delta x_L^{(j)}\|
\le
\|A_{j\to L}\|\,\varepsilon_j(b_j).
\label{eq:lem2_delta_bounds_new}
\end{equation}
Substituting into Eq.~\ref{eq:lem2_basic_bound_new} yields the distance-dependent upper bound:
\begin{equation}
\big|K_{i,j}(b_i,b_j)\big|
\le
\frac{M}{2}\,
\|A_{j\to L}\|^2\,
\|A_{i\to j}\|\,
\varepsilon_i(b_i)\,\varepsilon_j(b_j)
+\big|R_{i,j}\big|
\le
\frac{M}{2}\,
\rho^{\,2(L-j)}\,
\rho^{\,j-i}\,
\varepsilon_i(b_i)\,\varepsilon_j(b_j)
+\big|R_{i,j}\big|.
\label{eq:app_K_decay_bound_new}
\end{equation}
Moreover, the remainder can be explicitly controlled using $\|\nabla^3\ell\|\le T$:
\begin{equation}
\big|R_{i,j}\big|
\le
\frac{T}{6}\,
\mathbb{E}_u\!\Big[\big(\|\delta x_L^{(i)}\|+\|\delta x_L^{(j)}\|\big)^3\Big],
\label{eq:app_remainder_bound_new}
\end{equation}
which is cubic in the perturbation magnitudes and negligible in the small-perturbation regime around $\mathbf{b}_{\max}$.

\paragraph{Implication (adjacent interactions dominate).}
Eq.~\ref{eq:app_K_decay_bound_new} shows that under the effective stability condition in
Eq.~\ref{eq:assump_product_rho_new}, the leading second-order interaction between layers $i$ and $j$ decays
geometrically with their distance $(j-i)$ through the factor $\rho^{\,j-i}$, in addition to attenuation to the output
through $\rho^{\,2(L-j)}$.
Hence, adjacent pairs ($j=i+1$) capture the dominant second-order effects, while longer-range pairwise couplings are
attenuated and can be approximated by accumulating local effects along the chain.
This provides a theoretical basis for retaining only adjacent $K_{\ell-1,\ell}$ terms in the surrogate objective.

\paragraph{Estimation and optimization implications.}
\textbf{Estimating dense pairwise interactions is expensive and noisy.}
With $|\mathcal{B}|$ candidate bitwidths, estimating adjacent pairwise tables requires $(L-1)|\mathcal{B}|^2$ interaction
configurations, whereas estimating all pairs requires ${L\choose 2}|\mathcal{B}|^2$ configurations, which is quadratic in $L$
and substantially increases calibration time and estimation variance under limited calibration data.

\textbf{Chain-structured DP is preserved by local factors, but not by dense pairwise graphs.}
With adjacent pairwise interactions only, the energy has a chain form:
\begin{equation}
E(\mathbf{b})
=
\sum_{\ell=1}^{L} C_\ell(b_\ell)
+
\sum_{\ell=2}^{L} K_{\ell-1,\ell}(b_{\ell-1},b_\ell),
\label{eq:app_chain_energy_new}
\end{equation}
which supports efficient DP (including the budget dimension used in Appendix~\ref{app:dp}).
If one adds interactions for all layer pairs:
\begin{equation}
E_{\mathrm{full}}(\mathbf{b})
=
\sum_{\ell=1}^{L} C_\ell(b_\ell)
+
\sum_{1\le i<j\le L} K_{i,j}(b_i,b_j),
\label{eq:app_dense_energy_new}
\end{equation}
the resulting interaction graph is dense and has high treewidth, so exact optimization no longer admits the linear-time
chain DP structure.
Finally, we note that incorporating higher-order adjacent interactions, such as triplets
$T_{\ell-2,\ell-1,\ell}(b_{\ell-2},b_{\ell-1},b_{\ell})$, still preserves a chain factorization and remains exactly solvable
by DP after augmenting the state to retain the two most recent assignments, but it increases both calibration cost
(from $|\mathcal{B}|^2$ to $|\mathcal{B}|^3$ per local factor) and solver complexity accordingly.
For these reasons, we adopt adjacent pairwise interactions as a stable and computationally efficient surrogate.

\subsection{Dynamic Programming for Chain-Structured Mixed-Precision Allocation}
\label{app:dp}

This subsection presents a detailed derivation and implementation-ready dynamic programming (DP) solver for the
chain-structured surrogate objective used in the main text. We restate the objective, show why it admits an exact
DP solution, and then provide the state definition, initialization, recursion, optimality proof sketch, and
backtracking procedure.

\paragraph{Problem statement.}
Consider a network with $L$ layers. Each layer $\ell$ is assigned an activation bitwidth $b_\ell\in\mathcal{B}$,
where $\mathcal{B}=\{2,4,6,8\}$, and the layer-wise configuration is $\mathbf{b}=(b_1,\ldots,b_L)$.
Let $B$ be the global bit budget. Since we allocate bitwidths uniformly across layers (i.e., each layer has equal budget weight),
the total budget consumption is
\begin{equation}
\sum_{\ell=1}^{L} b_\ell \;\le\; B.
\label{eq:app_dp_budget}
\end{equation}
We minimize a chain-structured surrogate objective consisting of unary (per-layer) costs and pairwise (adjacent-layer)
interaction costs:
\begin{equation}
\min_{\mathbf{b}\in\mathcal{B}^L}
\;\;
\sum_{\ell=1}^{L} C_\ell(b_\ell)
\;+\;
\sum_{\ell=2}^{L} K_{\ell-1,\ell}(b_{\ell-1},b_\ell),
\qquad
\text{s.t. } \sum_{\ell=1}^{L} b_\ell \le B.
\label{eq:app_dp_obj}
\end{equation}

\paragraph{Why dynamic programming applies (chain optimal substructure).}
The objective in Eq.~\ref{eq:app_dp_obj} is a sum of local terms on a chain: each $b_\ell$ appears only in
$C_\ell(b_\ell)$ and in at most two pairwise terms $K_{\ell-1,\ell}(b_{\ell-1},b_\ell)$ and
$K_{\ell,\ell+1}(b_\ell,b_{\ell+1})$.
Therefore, if we fix the bitwidth at layer $\ell$ and fix the cumulative budget spent up to $\ell$, the remaining
choices for layers $1{:}\ell-1$ are independent of layers $\ell+1{:}L$ except through $b_\ell$.
This yields the standard optimal substructure property needed for exact DP.

\paragraph{DP state definition.}
For $\ell\in\{1,\ldots,L\}$, cumulative budget $c\in\{0,1,\ldots,B\}$, and current-layer bitwidth $b\in\mathcal{B}$,
define
\begin{equation}
\mathrm{DP}[\ell][c][b]
\;=\;
\min_{\substack{b_1,\ldots,b_{\ell-1}\in\mathcal{B}\\
\sum_{k=1}^{\ell} b_k=c,\;\; b_\ell=b}}
\Bigg(
\sum_{k=1}^{\ell} C_k(b_k)
+
\sum_{k=2}^{\ell} K_{k-1,k}(b_{k-1},b_k)
\Bigg),
\label{eq:app_dp_state}
\end{equation}
i.e., the minimum surrogate objective value among all assignments up to layer $\ell$ that (i) spend exactly $c$
budget and (ii) assign bitwidth $b$ to layer $\ell$.

\paragraph{Initialization.}
At $\ell=1$, there is no pairwise term. For each $b\in\mathcal{B}$, the budget is $c=b$ and
\begin{equation}
\mathrm{DP}[1][c][b] \;=\; C_1(b),\qquad c=b.
\label{eq:app_dp_init}
\end{equation}
All other states are infeasible and are set to $+\infty$.

\paragraph{Deriving the recursion.}
For any $\ell\ge 2$, consider a feasible partial assignment ending at layer $\ell$ with $b_\ell=b$ and total budget $c$.
Let the previous layer choose $b_{\ell-1}=b'$. Then the budget spent on layers $1{:}\ell-1$ must be
$c' = c - b$, and feasibility requires $c'\ge 0$. The surrogate objective decomposes as
\begin{align}
&\sum_{k=1}^{\ell} C_k(b_k) + \sum_{k=2}^{\ell} K_{k-1,k}(b_{k-1},b_k)
\notag\\
&\quad=
\underbrace{\Big(\sum_{k=1}^{\ell-1} C_k(b_k) + \sum_{k=2}^{\ell-1} K_{k-1,k}(b_{k-1},b_k)\Big)}_{\text{objective on }1{:}\ell-1}
\;+\;
C_\ell(b)
\;+\;
K_{\ell-1,\ell}(b',b).
\label{eq:app_dp_decomp}
\end{align}
The bracketed part is exactly the quantity minimized by $\mathrm{DP}[\ell-1][c'][b']$ under the constraints
$b_{\ell-1}=b'$ and budget $c'$. Hence, minimizing over all possible predecessor choices $b'\in\mathcal{B}$ yields
\begin{equation}
\mathrm{DP}[\ell][c][b]
=
C_\ell(b)
+
\min_{b'\in\mathcal{B}}
\left\{
\mathrm{DP}[\ell-1][c-b][b']
+
K_{\ell-1,\ell}(b',b)
\right\},
\label{eq:app_dp_rec}
\end{equation}
subject to the feasibility constraint $c-b\ge 0$; if $c-b<0$, we set $\mathrm{DP}[\ell][c][b]=+\infty$.

\paragraph{Optimal value under a budget \texorpdfstring{$B$}{B}.}
The DP state uses an exact budget $c$. To enforce the global constraint $\sum_\ell b_\ell\le B$, we take the best
value among all terminal states with $c\le B$:
\begin{equation}
\min_{c\le B}\ \min_{b\in\mathcal{B}}\ \mathrm{DP}[L][c][b].
\label{eq:app_dp_opt}
\end{equation}

\paragraph{Backtracking to recover \texorpdfstring{$\{b_\ell^*\}$}{b*}.}
During the forward DP, we store the argmin predecessor
\begin{equation}
\mathrm{Prev}[\ell][c][b]
\;\in\;
\arg\min_{b'\in\mathcal{B}}
\left\{
\mathrm{DP}[\ell-1][c-b][b']
+
K_{\ell-1,\ell}(b',b)
\right\}.
\label{eq:app_dp_prev}
\end{equation}
After selecting the terminal pair $(c^*,b_L^*)$ that attains Eq.~\ref{eq:app_dp_opt}, we recover the full assignment by
iterating backward for $\ell=L,L-1,\ldots,2$:
\begin{equation}
b_{\ell-1}^* \leftarrow \mathrm{Prev}[\ell][c^*][b_\ell^*],
\qquad
c^* \leftarrow c^* - b_\ell^*.
\label{eq:app_dp_backtrack}
\end{equation}
This yields the optimal layer-wise bitwidth assignment $\{b_\ell^*\}_{\ell=1}^{L}$.

\paragraph{Correctness (principle of optimality).}
Eq.~\ref{eq:app_dp_rec} follows from the optimal substructure of the chain objective:
any optimal solution for layers $1{:}\ell$ ending at $(c,b)$ must contain an optimal solution for layers $1{:}\ell-1$
ending at $(c-b, b')$ for some predecessor $b'$, otherwise we could replace the prefix by a better one and
strictly reduce the total objective, contradicting optimality. Thus the DP computes the exact optimum.

\paragraph{Complexity.}
Let $|\mathcal{B}|$ be the number of candidate bitwidths. The recursion in Eq.~\ref{eq:app_dp_rec} evaluates a
$\min$ over $|\mathcal{B}|$ predecessors for each $(\ell,c,b)$, giving time complexity
$O\!\left(L\cdot B\cdot |\mathcal{B}|^2\right)$ and memory $O\!\left(L\cdot B\cdot |\mathcal{B}|\right)$.
In practice, memory can be reduced to $O(B\cdot |\mathcal{B}|)$ by keeping only the previous layer's DP table, while
storing backpointers separately (or recomputing them on demand).

\section{Additional Experimental Results}
\label{app:results}
\subsection{More Detailed Results}
\label{app:detailed main results}
In this appendix, we provide the full expanded results corresponding to Table~\ref{tab:main} in the main paper.

\paragraph{A16 (full-precision activations).}
\begin{table}[h!]
\centering
\caption{Zero-shot accuracy on Arc-Challenge (AC), Arc-Easy (AE), HellaSwag (HS), LAMBADA-openai (LO), LAMBADA-standard (LS), PIQA (PQ), and WinoGrande (WG) under a \textbf{16-bit activation} quantization setting.}
\label{tab:app:tab1}
\resizebox{0.85\linewidth}{!}{
\tablestyle{3pt}{1.1}
\begin{tabular}{clcccccccccc}
\toprule
\textbf{Model} & \textbf{Method} & \textbf{\#Bits(W)} & \textbf{\#Bits(A)} & \textbf{AE}$\uparrow$ & \textbf{AC}$\uparrow$ & \textbf{HS}$\uparrow$ & \textbf{LO}$\uparrow$ & \textbf{LS}$\uparrow$ & \textbf{PQ}$\uparrow$ & \textbf{WG}$\uparrow$ & \textbf{$\text{Avg.}^7$}$\uparrow$  \\
\midrule
\multirow{7}{*}{LLaMA2-7B} & FP16        & 16   & \multirow{7}{*}{16} & 74.66 & 46.25 & 75.96 & 73.45 & 68.19 & 78.73 & 69.22 & 69.49  \\
\cdashline{2-12}
 & GPTQ & 2 &  & 32.45 & 22.61 & 30.53 & 7.72 & 4.08 & 55.44 & 50.83 & 29.09 \\
 & QuaRot & 2 &  & 40.95 & 24.83 & 35.77 & 23.87 & 16.94 & 58.60 & 53.12 & 36.30 \\
 & SliM-LLM & 2MP &  & 50.44 & 25.59 & 52.21 & 52.33 & 39.77 & 62.45 & 59.95 & 48.96 \\
 & PB-LLM & 1.7 &  & 29.29 & 23.04 & 28.33 & 8.77 & 10.15 & 53.37 & 50.12 & 29.01 \\
 & PT$^2$-LLM & 1.58 &  & 47.01 & 25.08 & 46.82 & 32.35 & 28.75 & 62.95 & 56.75 & 42.82 \\
 & \textbf{TWLA} & 1.58 &  & \textbf{65.87} & \textbf{37.63} & \textbf{65.75} & \textbf{70.33} & \textbf{60.94} & \textbf{74.65} & \textbf{65.19} & \textbf{62.91} \\
\midrule

\multirow{7}{*}{LLaMA2-13B} & FP16 & 16 & \multirow{7}{*}{16} & 77.57 & 49.15 & 79.39 & 76.71 & 70.06 & 80.47 & 71.98 & 72.19 \\
\cdashline{2-12}
 & GPTQ & 2 &  & 26.39 & 27.39 & 28.15 & 3.98 & 4.68 & 50.05 & 49.80 & 27.21 \\
 & QuaRot & 2 &  & 51.98 & 29.61 & 49.63 & 51.25 & 38.68 & 65.34 & 58.64 & 49.30 \\
 & SliM-LLM & 2MP &  & 56.71 & 30.46 & 51.76 & 55.17 & 43.21 & 66.87 & 59.83 & 52.00 \\
 & PB-LLM & 1.7 &  & 27.82 & 24.74 & 27.35 & 2.19 & 2.41 & 51.58 & 47.91 & 26.29 \\
 & PT$^2$-LLM & 1.58 &  & 62.85 & 39.67 & 55.22 & 59.20 & 45.21 & 70.77 & 62.89 & 56.54 \\
 & \textbf{TWLA} & 1.58 &  & \textbf{70.24} & \textbf{42.41} & \textbf{72.59} & \textbf{76.40} & \textbf{66.82} & \textbf{76.71} & \textbf{68.75} & \textbf{67.70} \\
\midrule

\multirow{7}{*}{LLaMA2-70B} & FP16 & 16 & \multirow{7}{*}{16} & 81.02 & 57.34 & 83.78 & 79.57 & 74.67 & 82.70 & 77.90 & 76.71 \\
\cdashline{2-12}
 & GPTQ & 2 &  & 49.83 & 29.10 & 44.20 & 44.21 & 31.48 & 63.71 & 56.59 & 45.59 \\
 & QuaRot & 2 &  & 75.50 & 45.60 & 65.20 & 77.30 & 70.60 & 76.00 & 71.30 & 68.79 \\
 & SliM-LLM & 2MP &  & -- & -- & -- & -- & -- & -- & -- & -- \\
 & PB-LLM & 1.7 &  & 50.05 & 30.38 & 50.65 & 54.30 & 45.00 & 63.82 & 61.96 & 50.88 \\
 & PT$^2$-LLM & 1.58 &  & 71.00 & 37.71 & 66.17 & 62.35 & 55.73 & 72.96 & 71.35 & 62.47 \\
 & \textbf{TWLA} & 1.58 &  & \textbf{80.73} & \textbf{53.75} & \textbf{71.87} & \textbf{78.40} & \textbf{73.67} & \textbf{79.40} & \textbf{77.35} & \textbf{73.60} \\
\midrule

\multirow{7}{*}{LLaMA3-8B} & FP16 & 16 & \multirow{7}{*}{16} & 77.90 & 52.82 & 79.07 & 75.63 & 68.58 & 80.63 & 72.93 & 72.51 \\
\cdashline{2-12}
 & GPTQ & 2 &  & 28.28 & 23.29 & 28.46 & 2.17 & 0.60 & 52.07 & 49.80 & 26.38 \\
 & QuaRot & 2 &  & 40.87 & 24.15 & 36.34 & 17.50 & 14.85 & 60.34 & 55.56 & 35.66 \\
 & SliM-LLM & 2MP &  & 31.82 & 21.58 & 29.42 & 3.15 & 5.51 & 52.50 & 48.46 & 27.49 \\
 & PB-LLM & 1.7 &  & 31.52 & 19.88 & 29.73 & 11.66 & 10.89 & 53.86 & 49.72 & 29.32 \\
 & PT$^2$-LLM & 1.58 &  & 34.22 & 22.43 & 43.86 & 35.95 & 26.70 & 56.86 & 53.28 & 39.04 \\
 & \textbf{TWLA} & 1.58 &  & \textbf{67.30} & \textbf{39.42} & \textbf{66.90} & \textbf{64.93} & \textbf{59.21} & \textbf{74.76} & \textbf{68.35} & \textbf{62.98} \\
\midrule

\multirow{7}{*}{Qwen3-8B} & FP16 & 16 & \multirow{7}{*}{16} & 80.93 & 56.74 & 74.98 & 64.18 & 61.11 & 77.37 & 68.35 & 69.09 \\
\cdashline{2-12}
 & GPTQ & 2 &  & 30.35 & 21.93 & 31.46 & 10.79 & 7.51 & 52.99 & 49.72 & 29.25 \\
 & QuaRot & 2 &  & -- & -- & -- & -- & -- & -- & -- & -- \\
 & SliM-LLM & 2MP &  & 42.45 & 28.92 & 42.40 & 23.00 & 20.85 & 60.94 & 55.33 & 39.13 \\
 & PB-LLM & 1.7 &  & 41.08 & 25.09 & 38.15 & 28.53 & 24.57 & 59.25 & 52.33 & 38.43 \\
 & PT$^2$-LLM & 1.58 &  & -- & -- & -- & -- & -- & -- & -- & -- \\
 & \textbf{TWLA} & 1.58 &  & \textbf{71.38} & \textbf{45.99} & \textbf{63.09} & \textbf{60.29} & \textbf{55.13} & \textbf{73.78} & \textbf{64.72} & \textbf{62.05} \\
\midrule

\multirow{7}{*}{Qwen3-14B} & FP16 & 16 & \multirow{7}{*}{16} & 83.08 & 60.49 & 78.82 & 67.84 & 64.47 & 79.76 & 72.85 & 72.47 \\
\cdashline{2-12}
 & GPTQ & 2 &  & 32.49 & 22.35 & 29.01 & 7.37 & 6.31 & 54.03 & 48.70 & 28.61 \\
 & QuaRot & 2 &  & -- & -- & -- & -- & -- & -- & -- & -- \\
 & SliM-LLM & 2MP &  & 52.70 & 33.96 & 53.75 & 35.10 & 34.80 & 66.65 & 59.75 & 48.10 \\
 & PB-LLM & 1.7 &  & 46.21 & 29.78 & 43.15 & 40.07 & 34.87 & 62.24 & 57.30 & 44.80 \\
 & PT$^2$-LLM & 1.58 &  & 53.03 & 25.63 & 50.65 & 37.22 & 35.19 & 62.95 & 59.75 & 46.35 \\
 & \textbf{TWLA} & 1.58 &  & \textbf{76.26} & \textbf{52.30} & \textbf{70.98} & \textbf{69.05} & \textbf{63.95} & \textbf{76.55} & \textbf{70.64} & \textbf{68.48} \\
\midrule

\multirow{7}{*}{Qwen3-32B} & FP16 & 16 & \multirow{7}{*}{16} & 83.21 & 61.09 & 82.60 & 67.24 & 58.04 & 81.99 & 72.77 & 72.13 \\
\cdashline{2-12}
 & GPTQ & 2 &  & 38.72 & 30.63 & 45.48 & 22.22 & 17.87 & 59.58 & 52.80 & 38.18 \\
 & QuaRot & 2 &  & -- & -- & -- & -- & -- & -- & -- & -- \\
 & SliM-LLM & 2MP &  & 74.03 & 52.21 & 68.98 & 52.62 & 55.80 & 78.99 & 70.88 & 64.79 \\
 & PB-LLM & 1.7 &  & 71.04 & 48.72 & 63.18 & 65.94 & 57.38 & 72.14 & 69.06 & 63.92 \\
 & PT$^2$-LLM & 1.58 &  & -- & -- & -- & -- & -- & -- & -- & -- \\
 & \textbf{TWLA} & 1.58 &  & \textbf{79.53} & \textbf{55.33} & \textbf{76.89} & \textbf{67.41} & \textbf{57.33} & \textbf{78.99} & \textbf{71.27} & \textbf{69.54} \\
\bottomrule
\end{tabular}}
\end{table}

When activations remain in full precision (A16), the detailed breakdown in Table~\ref{tab:app:tab1} shows that classical weight-only PTQ baselines such as GPTQ and QuaRot incur substantial average-accuracy losses across model families and scales, with particularly pronounced drops on reasoning-heavy benchmarks (e.g., ARC-C and HellaSwag). In contrast, TWLA consistently delivers the strongest ternary results at the lowest average weight precision (1.58 bits). For example, on LLaMA3-8B, TWLA improves Avg$^7$ from 39.04 (PT$^2$-LLM) to 62.98, a +23.94-point gain (relative +61.3\%), retaining 86.9\% of FP16 performance. On Qwen3-14B, TWLA increases Avg$^7$ from 48.10 (SliM-LLM) to 68.48 (+20.38 points, +42.4\%), reaching 94.5\% of FP16 (72.47). On LLaMA2-70B, TWLA attains 73.60 Avg$^7$ versus 68.79 for QuaRot (+4.81 points, +7.0\%), and remains only 3.11 points below FP16 (76.71), i.e., a 4.05\% gap.

\paragraph{A6 (6-bit activations).}
\begin{table}[h!]
\centering
\caption{Zero-shot accuracy on Arc-Challenge (AC), Arc-Easy (AE), HellaSwag (HS), LAMBADA-openai (LO), LAMBADA-standard (LS), PIQA (PQ), and WinoGrande (WG) under a \textbf{6-bit activation} quantization setting.}
\label{tab:app:tab2}
\resizebox{0.85\linewidth}{!}{
\tablestyle{3pt}{1.1}
\begin{tabular}{clcccccccccc}
\toprule
\textbf{Model} & \textbf{Method} & \textbf{\#Bits(W)} & \textbf{\#Bits(A)} & \textbf{AE}$\uparrow$ & \textbf{AC}$\uparrow$ & \textbf{HS}$\uparrow$ & \textbf{LO}$\uparrow$ & \textbf{LS}$\uparrow$ & \textbf{PQ}$\uparrow$ & \textbf{WG}$\uparrow$ & \textbf{$\text{Avg.}^7$}$\uparrow$  \\
\midrule
\multirow{7}{*}{LLaMA2-7B}
& GPTQ       & 2     & 6      & 29.34 & 23.46 & 29.22 & 7.06  & 3.16  & 52.45 & 50.91 & 27.94 \\
& QuaRot     & 2     & 6      & 40.70 & 23.72 & 35.21 & 21.41 & 15.58 & 58.00 & 51.93 & 35.22 \\
& ResQ       & 2.3MP & 6.1MP  & 55.70 & 29.52 & 52.30 & 53.55 & 39.85 & 67.90 & 59.43 & 51.18 \\
& SliM-LLM   & 2MP   & 6      & 38.75 & 24.15 & 42.60 & 25.15 & 17.30 & 58.05 & 53.43 & 37.06 \\
& PB-LLM     & 1.7   & 6      & 30.35 & 21.67 & 33.75 & 11.70 & 12.15 & 52.56 & 52.01 & 30.60 \\
& PT$^2$-LLM & 1.58  & 6      & 44.82 & 25.79 & 42.97 & 22.07 & 17.98 & 60.97 & 55.67 & 38.61 \\
& \textbf{TWLA} & 1.58 & 6MP  & \textbf{64.56} & \textbf{36.69} & \textbf{64.92} & \textbf{70.13} & \textbf{59.23} & \textbf{74.70} & \textbf{63.61} & \textbf{61.98} \\
\midrule

\multirow{7}{*}{LLaMA2-13B}
& GPTQ       & 2     & 6      & 26.39 & 28.16 & 26.69 & 0.17  & 0.58  & 50.92 & 48.38 & 25.90 \\
& QuaRot     & 2     & 6      & 49.45 & 27.82 & 48.74 & 49.56 & 37.18 & 65.29 & 57.62 & 47.95 \\
& ResQ       & 2.3MP & 6.1MP  & 60.40 & 34.64 & 56.75 & 60.45 & 48.55 & 71.11 & 63.54 & 56.49 \\
& SliM-LLM   & 2MP   & 6      & 48.20 & 27.30 & 48.50 & 31.65 & 29.20 & 63.38 & 55.01 & 43.32 \\
& PB-LLM     & 1.7   & 6      & 28.05 & 24.06 & 28.10 & 0.70  & 0.55  & 51.90 & 52.88 & 26.61 \\
& PT$^2$-LLM & 1.58  & 6      & 58.40 & 36.23 & 52.58 & 44.94 & 37.30 & 67.31 & 59.01 & 51.82 \\
& \textbf{TWLA} & 1.58 & 6MP  & \textbf{69.23} & \textbf{42.49} & \textbf{71.02} & \textbf{75.98} & \textbf{65.28} & \textbf{76.28} & \textbf{67.88} & \textbf{66.88} \\
\midrule

\multirow{7}{*}{LLaMA2-70B}
& GPTQ       & 2     & 6      & 39.81 & 25.43 & 33.98 & 28.31 & 18.36 & 58.22 & 47.12 & 35.89 \\
& QuaRot     & 2     & 6      & 74.10 & 45.00 & 64.70 & 76.20 & 68.50 & 75.60 & 71.70 & 67.97 \\
& ResQ       & 2.3MP & 6.1MP  & 37.90 & 32.60 & 64.80 & 75.70 & 68.10 & 53.50 & 72.10 & 57.81 \\
& SliM-LLM   & 2MP   & 6      & --    & --    & --    & --    & --    & --    & --    & --    \\
& PB-LLM     & 1.7   & 6      & 39.10 & 26.79 & 42.15 & 38.95 & 29.70 & 55.71 & 53.83 & 40.89 \\
& PT$^2$-LLM & 1.58  & 6      & 58.70 & 32.96 & 51.75 & 38.10 & 38.80 & 64.65 & 57.75 & 48.96 \\
& \textbf{TWLA} & 1.58 & 6MP  & \textbf{80.07} & \textbf{53.50} & \textbf{71.60} & \textbf{77.60} & \textbf{73.60} & \textbf{79.93} & \textbf{77.82} & \textbf{73.45} \\
\midrule

\multirow{7}{*}{LLaMA3-8B}
& GPTQ       & 2     & 6      & 25.84 & 26.11 & 27.17 & 0.54  & 0.10  & 49.56 & 50.51 & 25.69 \\
& QuaRot     & 2     & 6      & 41.31 & 24.57 & 35.50 & 18.20 & 12.15 & 58.16 & 58.67 & 35.51 \\
& ResQ       & 2.3MP & 6.1MP  & 51.00 & 29.52 & 50.75 & 39.90 & 37.20 & 64.85 & 59.75 & 47.57 \\
& SliM-LLM   & 2MP   & 6      & 29.15 & 23.72 & 31.50 & 2.15  & 2.30  & 51.36 & 49.88 & 27.15 \\
& PB-LLM     & 1.7   & 6      & 32.10 & 19.54 & 35.65 & 12.25 & 10.00 & 52.99 & 49.41 & 30.28 \\
& PT$^2$-LLM & 1.58  & 6      & 29.93 & 21.79 & 35.79 & 14.22 & 12.90 & 52.99 & 48.02 & 30.81 \\
& \textbf{TWLA} & 1.58 & 6MP  & \textbf{65.15} & \textbf{35.86} & \textbf{65.84} & \textbf{63.87} & \textbf{58.20} & \textbf{73.83} & \textbf{67.32} & \textbf{61.44} \\
\midrule

\multirow{7}{*}{Qwen3-8B}
& GPTQ       & 2     & 6      & 26.81 & 24.06 & 28.03 & 2.87  & 1.94  & 51.52 & 51.30 & 26.65 \\
& QuaRot     & 2     & 6      & --    & --    & --    & --    & --    & --    & --    & --    \\
& ResQ       & 2.3MP & 6.1MP  & --    & --    & --    & --    & --    & --    & --    & --    \\
& SliM-LLM   & 2MP   & 6      & 33.90 & 25.00 & 36.15 & 12.15 & 9.70 & 54.95 & 47.99 & 31.41 \\
& PB-LLM     & 1.7   & 6      & 30.90 & 22.70 & 36.75 & 14.90 & 13.90 & 53.92 & 48.62 & 31.67 \\
& PT$^2$-LLM & 1.58  & 6      & --    & --    & --    & --    & --    & --    & --    & -- \\
& \textbf{TWLA} & 1.58 & 6MP  & \textbf{69.88} & \textbf{45.27} & \textbf{56.85} & \textbf{54.97} & \textbf{51.73} & \textbf{72.77} & \textbf{61.73} & \textbf{59.03} \\
\midrule

\multirow{7}{*}{Qwen3-14B}
& GPTQ       & 2     & 6      & 29.50 & 22.27 & 27.56 & 4.62  & 4.09  & 51.14 & 48.15 & 26.76 \\
& QuaRot     & 2     & 6      & --    & --    & --    & --    & --    & --    & --    & --    \\
& ResQ       & 2.3MP & 6.1MP  & --    & --    & --    & --    & --    & --    & --    & --    \\
& SliM-LLM   & 2MP   & 6      & 44.25 & 28.75 & 44.70 & 15.15 & 14.10 & 60.12 & 55.17 & 37.46 \\
& PB-LLM     & 1.7   & 6      & 40.80 & 24.74 & 42.05 & 23.25 & 19.40 & 58.32 & 52.41 & 37.28 \\
& PT$^2$-LLM & 1.58  & 6      & 42.21 & 27.72 & 43.20 & 14.33 & 12.12 & 58.13 & 54.19 & 36.00 \\
& \textbf{TWLA} & 1.58 & 6MP  & \textbf{74.41} & \textbf{50.85} & \textbf{69.89} & \textbf{67.88} & \textbf{62.25} & \textbf{75.57} & \textbf{70.17} & \textbf{67.29} \\
\midrule

\multirow{7}{*}{Qwen3-32B}
& GPTQ       & 2     & 6      & 32.28 & 26.71 & 36.34 & 10.64 & 8.67  & 54.41 & 49.01 & 31.15 \\
& QuaRot     & 2     & 6      & --    & --    & --    & --    & --    & --    & --    & --    \\
& ResQ       & 2.3MP & 6.1MP  & --    & --    & --    & --    & --    & --    & --    & --    \\
& SliM-LLM   & 2MP   & 6      & 61.70 & 39.93 & 51.40 & 37.45 & 40.90 & 65.89 & 62.90 & 52.43 \\
& PB-LLM     & 1.7   & 6      & 61.70 & 39.93 & 51.40 & 37.45 & 40.90 & 65.89 & 62.90 & 51.45 \\
& PT$^2$-LLM & 1.58  & 6      & --    & --    & --    & --    & --    & --    & --    & --  \\
& \textbf{TWLA} & 1.58 & 6MP  & \textbf{77.49} & \textbf{53.18} & \textbf{75.33} & \textbf{66.33} & \textbf{56.20} & \textbf{77.27} & \textbf{71.11} & \textbf{68.13} \\
\bottomrule
\end{tabular}
}
\end{table}
Under 6-bit activations, Table~\ref{tab:app:tab2} shows that performance gaps widen substantially: GPTQ/QuaRot often stay in low-accuracy regimes, and even stronger baselines (SliM-LLM and PB-LLM) exhibit clear degradation relative to A16. In contrast, TWLA remains the top performer across all evaluated models and is notably stable as activation precision is reduced. Concretely, on LLaMA2-13B, TWLA improves Avg$^7$ from 56.49 (ResQ) to 66.88 (+10.39 points, +18.4\%). On Qwen3-32B, TWLA raises Avg$^7$ from 51.45 (PB-LLM) to 68.13 (+16.68 points, +32.4\%). The collapse of conventional methods is particularly evident on LLaMA2-70B, where GPTQ attains 35.89 Avg$^7$ while TWLA reaches 73.45 (+37.56 points, +104.6\%). Importantly, TWLA’s own drop from A16 to A6 is minimal on large models (e.g., LLaMA2-70B: 73.60 $\rightarrow$ 73.45, only a 0.20\% relative decrease).

\paragraph{A4 (4-bit activations).}
\begin{table}[h!]
\centering
\caption{Zero-shot accuracy on Arc-Challenge (AC), Arc-Easy (AE), HellaSwag (HS), LAMBADA-openai (LO), LAMBADA-standard (LS), PIQA (PQ), and WinoGrande (WG) under a \textbf{4-bit activation} quantization setting.}
\label{tab:app:tab3}
\resizebox{0.85\linewidth}{!}{
\tablestyle{3pt}{1.1}
\begin{tabular}{clcccccccccc}
\toprule
\textbf{Model} & \textbf{Method} & \textbf{\#Bits(W)} & \textbf{\#Bits(A)} & \textbf{AE}$\uparrow$ & \textbf{AC}$\uparrow$ & \textbf{HS}$\uparrow$ & \textbf{LO}$\uparrow$ & \textbf{LS}$\uparrow$ & \textbf{PQ}$\uparrow$ & \textbf{WG}$\uparrow$ & \textbf{$\text{Avg.}^7$}$\uparrow$  \\
\midrule
\multirow{7}{*}{LLaMA2-7B} & GPTQ & 2 & 4 & 26.26 & 27.73 & 26.35 & 0.00 & 0.00 & 49.89 & 50.36 & 25.80 \\
& QuaRot & 2 & 4 & 35.02 & 22.87 & 31.34 & 10.27 & 8.48 & 56.75 & 51.07 & 30.83 \\
& ResQ & 2.3MP & 4.2MP & 46.65 & 28.50 & 48.05 & 45.30 & 32.10 & 63.32 & 52.72 & 45.23 \\
& SliM-LLM & 2MP & 4 & 26.80 & 25.43 & 28.10 & 0.15 & 0.15 & 49.24 & 49.80 & 25.67 \\
& PB-LLM & 1.7 & 4 & 28.20 & 25.26 & 29.15 & 0.70 & 0.35 & 50.11 & 51.46 & 26.46 \\
& PT$^2$-LLM & 1.58 & 4 & 29.20 & 26.21 & 30.22 & 0.90 & 0.65 & 51.23 & 52.77 & 27.31 \\
& \textbf{TWLA} & 1.58 & 4MP & \textbf{58.88} & \textbf{34.73} & \textbf{60.81} & \textbf{64.45} & \textbf{53.76} & \textbf{71.27} & \textbf{62.12} & \textbf{58.00} \\
\midrule

\multirow{7}{*}{LLaMA2-13B} & GPTQ & 2 & 4 & 26.39 & 27.82 & 25.93 & 0.00 & 0.00 & 48.53 & 48.78 & 25.35 \\
& QuaRot & 2 & 4 & 39.56 & 25.77 & 39.94 & 32.08 & 23.37 & 60.12 & 55.09 & 39.42 \\
& ResQ & 2.3MP & 4.2MP & 56.75 & 32.94 & 54.80 & 50.10 & 38.70 & 68.66 & 60.93 & 51.84 \\
& SliM-LLM & 2MP & 4 & 27.15 & 25.17 & 30.45 & 0.40 & 0.00 & 50.87 & 51.07 & 26.44 \\
& PB-LLM & 1.7 & 4 & 26.10 & 26.54 & 26.95 & 0.00 & 0.00 & 48.69 & 50.51 & 25.54 \\
& PT$^2$-LLM & 1.58 & 4 & 27.10 & 26.54 & 28.95 & 0.00 & 0.00 & 49.69 & 51.22 & 26.19 \\
& \textbf{TWLA} & 1.58 & 4MP & \textbf{66.58} & \textbf{39.85} & \textbf{68.20} & \textbf{73.24} & \textbf{60.82} & \textbf{74.27} & \textbf{67.17} & \textbf{64.30} \\
\midrule

\multirow{7}{*}{LLaMA2-70B} & GPTQ & 2 & 4 & 26.39 & 26.88 & 26.14 & 0.00 & 0.00 & 49.46 & 52.57 & 25.92 \\
& QuaRot & 2 & 4 & 63.50 & 34.20 & 55.20 & 57.50 & 35.20 & 68.60 & 61.30 & 53.64 \\
& ResQ & 2.3MP & 4.2MP & 42.60 & 34.60 & 62.40 & 70.00 & 60.90 & 55.10 & 69.40 & 56.43 \\
& SliM-LLM & 2MP & 4 & -- & -- & -- & -- & -- & -- & -- & -- \\
& PB-LLM & 1.7 & 4 & 26.95 & 24.40 & 29.50 & 0.55 & 0.45 & 50.44 & 50.75 & 26.18 \\
& PT$^2$-LLM & 1.58 & 4 & 26.95 & 25.40 & 29.59 & 0.35 & 0.15 & 50.22 & 50.66 & 26.19 \\
& \textbf{TWLA} & 1.58 & 4MP & \textbf{78.60} & \textbf{51.88} & \textbf{70.00} & \textbf{75.93} & \textbf{69.07} & \textbf{77.33} & \textbf{74.90} & \textbf{71.10} \\
\midrule

\multirow{7}{*}{LLaMA3-8B} & GPTQ & 2 & 4 & 26.56 & 25.00 & 26.46 & 0.00 & 0.00 & 50.38 & 51.07 & 25.64 \\
& QuaRot & 2 & 4 & 31.61 & 21.42 & 29.26 & 4.81 & 4.97 & 53.43 & 48.30 & 27.69 \\
& ResQ & 2.3MP & 4.2MP & 42.90 & 26.96 & 43.85 & 23.35 & 21.75 & 59.64 & 55.38 & 39.12 \\
& SliM-LLM & 2MP & 4 & 27.15 & 25.68 & 28.05 & 0.05 & 0.00 & 48.69 & 47.99 & 25.37 \\
& PB-LLM & 1.7 & 4 & 29.70 & 22.78 & 30.30 & 0.85 & 0.50 & 51.03 & 48.86 & 26.29 \\
& PT$^2$-LLM & 1.58 & 4 & 28.71 & 22.02 & 30.10 & 0.75 & 0.30 & 50.83 & 47.99 & 25.81 \\
& \textbf{TWLA} & 1.58 & 4MP & \textbf{60.23} & \textbf{34.56} & \textbf{59.83} & \textbf{51.87} & \textbf{45.47} & \textbf{71.38} & \textbf{63.30} & \textbf{55.23} \\
\midrule

\multirow{7}{*}{Qwen3-8B} & GPTQ & 2 & 4 & 25.17 & 26.37 & 25.47 & 0.00 & 0.00 & 50.38 & 49.64 & 25.29 \\
& QuaRot & 2 & 4 & -- & -- & -- & -- & -- & -- & -- & -- \\
& ResQ & 2.3MP & 4.2MP & -- & -- & -- & -- & -- & -- & -- & -- \\
& SliM-LLM & 2MP & 4 & 26.35 & 26.28 & 27.10 & 0.00 & 0.00 & 48.75 & 50.75 & 25.60 \\
& PB-LLM & 1.7 & 4 & 28.85 & 23.81 & 28.65 & 0.15 & 0.05 & 52.18 & 50.12 & 26.26 \\
& PT$^2$-LLM & 1.58 & 4 & -- & -- & -- & -- & -- & -- & -- & -- \\
& \textbf{TWLA} & 1.58 & 4MP & \textbf{57.19} & \textbf{35.83} & \textbf{51.82} & \textbf{40.49} & \textbf{39.97} & \textbf{68.22} & \textbf{59.45} & \textbf{50.42} \\
\midrule

\multirow{7}{*}{Qwen3-14B} & GPTQ & 2 & 4 & 26.60 & 25.94 & 25.74 & 0.00 & 0.00 & 52.07 & 50.20 & 25.79 \\
& QuaRot & 2 & 4 & -- & -- & -- & -- & -- & -- & -- & -- \\
& ResQ & 2.3MP & 4.2MP & -- & -- & -- & -- & -- & -- & -- & -- \\
& SliM-LLM & 2MP & 4 & 26.35 & 25.77 & 27.40 & 0.00 & 0.00 & 50.27 & 49.57 & 25.62 \\
& PB-LLM & 1.7 & 4 & 28.05 & 23.63 & 31.90 & 0.80 & 0.65 & 50.60 & 49.01 & 29.23 \\
& PT$^2$-LLM & 1.58 & 4 & -- & -- & -- & -- & -- & -- & -- & -- \\
& \textbf{TWLA} & 1.58 & 4MP & \textbf{68.39} & \textbf{45.05} & \textbf{65.29} & \textbf{61.48} & \textbf{57.17} & \textbf{72.91} & \textbf{63.61} & \textbf{62.00} \\
\midrule

\multirow{7}{*}{Qwen3-32B} & GPTQ & 2 & 4 & 24.62 & 25.60 & 26.13 & 0.00 & 0.00 & 49.29 & 48.30 & 24.85 \\
& QuaRot & 2 & 4 & -- & -- & -- & -- & -- & -- & -- & -- \\
& ResQ & 2.3MP & 4.2MP & -- & -- & -- & -- & -- & -- & -- & -- \\
& SliM-LLM & 2MP & 4 & 26.10 & 24.77 & 29.80 & 0.85 & 0.55 & 50.89 & 51.22 & 26.31 \\
& PB-LLM & 1.7 & 4 & 27.10 & 25.77 & 28.80 & 0.30 & 0.40 & 51.09 & 51.07 & 26.36 \\
& PT$^2$-LLM & 1.58 & 4 & -- & -- & -- & -- & -- & -- & -- & -- \\
& \textbf{TWLA} & 1.58 & 4MP & \textbf{73.34} & \textbf{51.83} & \textbf{67.80} & \textbf{65.90} & \textbf{55.50} & \textbf{74.29} & \textbf{68.10} & \textbf{65.25} \\
\bottomrule
\end{tabular}
}
\end{table}
The A4 setting represents an extreme low-precision activation regime. As shown in Table~\ref{tab:app:tab3}, most baselines exhibit severe failure modes, with pronounced degradation on LAMBADA (LO/LS) and $\mathrm{Avg}^7$ approaching chance-level on several models. By contrast, TWLA remains markedly resilient at A4 and preserves strong margins, especially at scale. For LLaMA2-70B, TWLA improves $\mathrm{Avg}^7$ from 53.64 (QuaRot) to 71.10 (+17.46 points, +32.6\%); on LAMBADA-standard, TWLA reaches 77.33 versus 21.50 for QuaRot (+55.83 points, +259.7\%). On Qwen3-32B, TWLA increases $\mathrm{Avg}^7$ from 26.36 (PB-LLM) to 65.25 (+38.89 points, +147.5\%), and from 24.85 (GPTQ) to 65.25 (+40.40 points, +162.6\%). Overall, these results highlight TWLA’s strong robustness and high quantization quality under 4-bit activations.

\paragraph{Summary.}
Overall, the expanded tables consistently show that TWLA offers superior robustness and scalability as activation precision decreases from 16 to 6 and further to 4 bits. While existing PTQ methods rapidly lose accuracy (and often collapse) under low-bit activations, TWLA maintains strong and stable performance across model families and sizes. This stability is evident from the small performance drift as activations are reduced on large models; for example, on Qwen3-32B, Avg$^7$ changes only slightly from 69.54 (A16) to 68.13 (A6, $-2.03\%$) and remains high at 65.25 under A4 ($-4.23\%$ from A6; $-6.17\%$ from A16). These results support TWLA as a practical PTQ solution for extreme low-precision inference with ternary weights and low-bit activations.

\subsection{Ablation study on Calibration Data}
\label{app:dataset}
\begin{table}[h!]
\centering
\caption{Ablation study on calibration set type.}
\label{tab:app:dataset}
\resizebox{0.7\linewidth}{!}{
\tablestyle{3pt}{1.1}
\begin{tabular}{cccccc}
\toprule
\textbf{Model} & \textbf{Calibration Data Type} & \textbf{Wikitext2}$\downarrow$ & \textbf{C4}$\downarrow$  & \textbf{PTB}$\downarrow$ & \textbf{$\text{Avg.}^7$}$\uparrow$  \\
\midrule
\multirow{3}{*}{LLaMA-2-7B}
& Wikitext2 & 8.31 &12.47 &58.12 &58.00  \\
& C4        &11.76  &10.77  &60.10 &57.39  \\
& PTB       &13.10  &19.29  &54.70 &56.09  \\
\midrule
\multirow{3}{*}{LLaMA-3-8B}
& Wikitext2 &12.83  &19.31  &21.93 &55.23  \\
& C4        &16.37  &15.57  &21.78 &54.96  \\
& PTB       &16.59  &16.21  &18.49 &53.39  \\
\midrule
\multirow{3}{*}{Qwen3-8B}
& Wikitext2 &16.26  &22.72  &28.65 & 50.42 \\
& C4        &21.37  &17.33  &29.14 & 49.16 \\
& PTB       &21.95  &23.08  &23.37 & 47.73 \\
\bottomrule
\end{tabular}}
\end{table}

\begin{figure}
    \centering
    \includegraphics[width=0.8\linewidth]{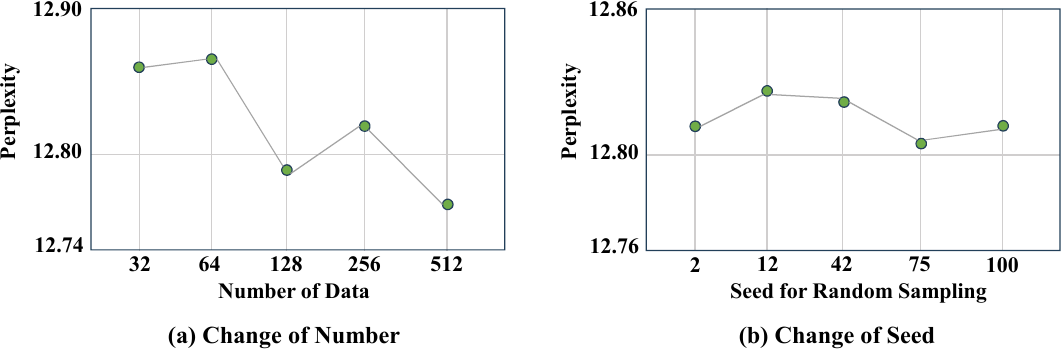}
    \caption{Perplexity of LLaMA3-8B (W1.58A4) using calibration data sampled with different number or seeds from WikiText2.}
    \label{fig:dataset}
\end{figure}

We conducted an ablation on calibration data choice using three corpora (WikiText2, C4, and PTB) to examine how the calibration set affects TWLA quantization. As summarized in Table~\ref{tab:app:dataset}, calibrating on WikiText2 or C4 yields comparable $\mathrm{Avg.}$ accuracy across all tested models, whereas PTB consistently underperforms by a clear margin. One possible factor is the domain and diversity mismatch of PTB relative to the evaluation suites, which may yield less representative activation statistics; we leave a controlled domain-matching test for future work. In addition, calibrating on WikiText2 (resp.\ C4) improves perplexity on the same in-domain benchmark (WikiText2 / C4), highlighting a pronounced domain-matching advantage: the calibration corpus not only impacts downstream accuracy but also directly benefits the perplexity of the corresponding evaluation domain.

We further investigated the impact of calibration data on TWLA. Specifically, when binarizing LLaMA3-8B on WikiText-2, we fixed all quantization hyperparameters and varied the calibration set by (i) subsampling different numbers of calibration samples and (ii) resampling with multiple random seeds that control both sample selection and token order. For each setting, we re-quantized once and evaluated perplexity on a held-out split. Across all sizes and seeds, the resulting perplexity fluctuates by $<$1\% relative to the mean. As shown in Fig.~\ref{fig:dataset}, the curves remain nearly flat as calibration size increases, and seed-wise traces largely overlap, indicating that even our smallest tested calibration subsets perform on par with larger ones. This stability suggests that TWLA does not rely on large-scale calibration corpora to obtain reliable activation statistics; instead, a small but representative subset is sufficient to reach near-saturated performance, which substantially reduces the calibration overhead in practical deployments. These observations confirm that TWLA is highly robust to calibration data selection.

\subsection{Ablation study on ILA-AMP}
\label{app:amp}
\begin{figure}
    \centering
    \includegraphics[width=0.7\linewidth]{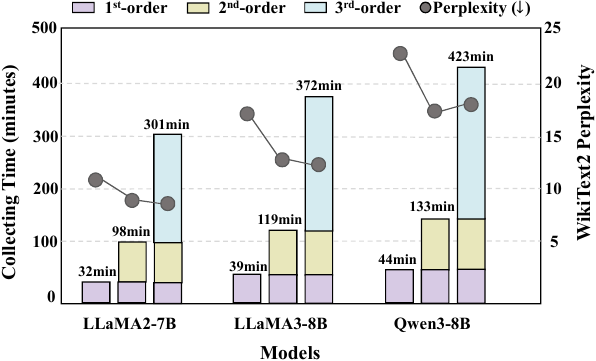}
    \caption{Calibration cost versus quantization quality under different interaction orders. Bars report the wall-clock time (minutes) spent collecting validation NLL statistics for 1st-, 2nd-, and 3rd-order interaction modeling, and dots report WikiText2 perplexity (lower is better).}
    \label{fig:amp}
\end{figure}

\paragraph{Benefit--cost trade-off across interaction orders.}

To characterize the calibration overhead induced by different interaction orders, we measure the cost by the wall-clock time required to collect validation NLL statistics.
Let $|\mathcal{B}|$ denote the number of candidate bitwidths (with $|\mathcal{B}|=4$ in our setting) and $L$ the number of layers.
When interactions are restricted to local adjacent neighborhoods, the calibration cost (in terms of the number of NLL-evaluation configurations, which typically translates into wall-clock time) grows approximately as
\begin{equation}
\mathrm{Cost}(k)\ \propto\ N_k,
\qquad
N_1 = L\,|\mathcal{B}|,
\quad
N_2 = (L-1)\,|\mathcal{B}|^2,
\quad
N_3 = (L-2)\,|\mathcal{B}|^3,
\label{eq:app_cost_order}
\end{equation}
where $k\in\{1,2,3\}$ denotes the interaction order.
The 1st-order variant uses only unary layer sensitivities (i.e., single-layer NLL perturbations) and thus implicitly assumes that quantization effects are independent across layers.
The 2nd-order variant (ILA-AMP) additionally models adjacent-layer pairwise couplings to capture error propagation across neighboring layers.
The 3rd-order variant further incorporates adjacent triplet interactions.
As shown in Fig.~\ref{fig:amp}, we compare 1st-, 2nd-, and 3rd-order designs on three representative models,
\textsc{LLaMA2-7B}, \textsc{LLaMA3-8B}, and \textsc{Qwen3-8B}, reporting both the NLL-collection time and the WikiText2 perplexity (lower is better).
Moving from 1st-order to 2nd-order (ILA-AMP) yields a clear and consistent perplexity reduction at an acceptable increase in calibration time, indicating that adjacent pairwise couplings capture the dominant non-additive effects caused by cross-layer distribution shifts.
In contrast, extending to 3rd-order incurs a substantial additional time overhead while providing little to no further perplexity improvement, and it may even slightly degrade perplexity.
For example, on \textsc{Qwen3-8B}, upgrading from 2nd-order to 3rd-order increases the NLL-collection time by nearly 300 minutes (from 133 minutes to 423 minutes), yet the perplexity slightly increases.
Overall, these results empirically justify our choice of a 2nd-order adjacent-interaction surrogate: it delivers most of the accuracy benefit with a significantly better cost--effectiveness profile than higher-order alternatives.

\paragraph{Layer-wise activation bit allocation under a 4-bit budget.}
\begin{figure}[h!]
    \centering
    \includegraphics[width=\linewidth]{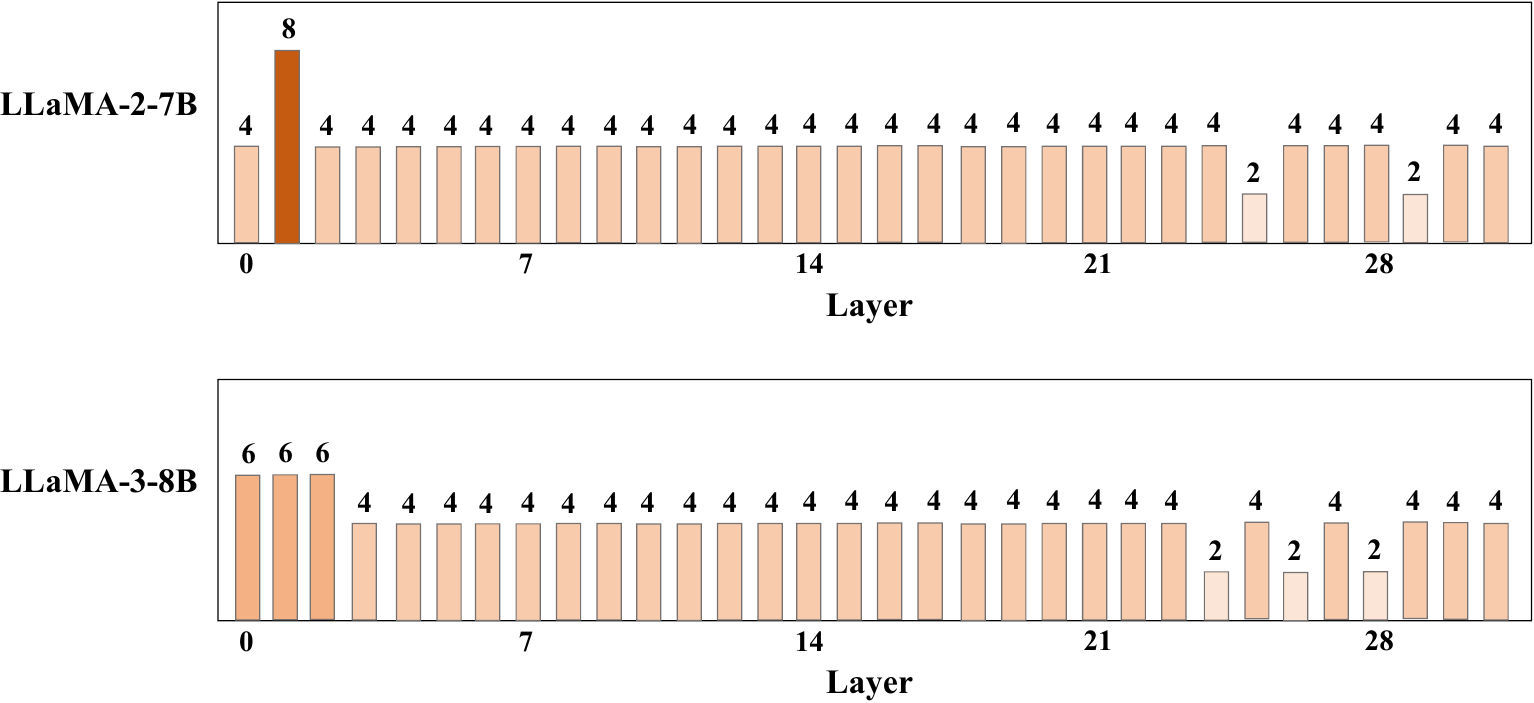}
    \caption{Layer-wise activation bit allocation produced by ILA-AMP under an average activation bit budget of 4 on the WikiText2 calibration set for LLaMA2-7B and LLaMA3-8B. Numbers above bars denote the assigned per-layer activation bit-width.}
    \label{fig:layer-bit}
\end{figure}

We visualize the layer-wise activation bit assignments produced by ILA-AMP under an average activation budget of 4 on the WikiText2 calibration set (as shown in Fig.~\ref{fig:layer-bit}). Both LLaMA2-7B and LLaMA3-8B exhibit a sparse adjustment pattern: most layers stay at 4-bit, while a few critical layers are assigned higher precision and a few less sensitive layers are reduced to 2-bit to satisfy the global budget. Concretely, LLaMA2-7B allocates a single early layer to 8-bit, keeps the vast majority of layers at 4-bit, and compensates with two late layers at 2-bit. LLaMA3-8B assigns 6-bit to the first three layers, maintains 4-bit for most layers, and uses three late 2-bit layers as budget offsets. This structure suggests that early layers are more influential under activation quantization: noise injected there perturbs the activation distribution that feeds all subsequent layers, and the resulting distribution shifts can accumulate and amplify along the stack. In contrast, some late layers are comparatively less sensitive because the remaining propagation path is shorter, leaving less room for amplification; thus, they can tolerate more aggressive quantization with limited accuracy loss. Overall, the observed “higher precision early, lower precision late, sparse deviations from 4-bit” allocation is consistent with ILA-AMP’s adjacent-interaction modeling, which prioritizes precision where cross-layer error propagation is most consequential.

\paragraph{Ablation on layer-sensitivity metrics.}

\begin{table}[t]
\centering
\caption{WikiText2 perplexity under activation mixed precision with average bits(A)=4, using different layer-sensitivity metrics.}
\label{tab:mp_metric}
\resizebox{0.6\linewidth}{!}{
\tablestyle{3pt}{1.1}
\begin{tabular}{c c c c c c c}
\toprule
\multirow{2}{*}{\textbf{Model}}  & \multirow{2}{*}{\textbf{Avg. Bits(A)}} & \multicolumn{5}{c}{\textbf{MP metric (PPL$\downarrow$)}} \\
\cmidrule(lr){3-7}
& & \textbf{LIM} & \textbf{ZD} & \textbf{Act.} & \textbf{NLL} & \textbf{ILA-AMP} \\
\midrule
LLaMA-2-7B & 4 & 17.49 & 18.01 & 16.33 & 11.92 & \textbf{8.31} \\
LLaMA-3-8B & 4 & 23.21 & 23.37 & 20.05 & 15.03 & \textbf{12.83} \\
Qwen3-8B  & 4 & 33.17 & 29.76 & 24.64 & 21.79 & \textbf{16.15} \\
\bottomrule
\end{tabular}}
\end{table}

We further investigate how different layer-sensitivity metrics affect activation mixed-precision allocation under the same average activation-bit budget. Specifically, we compare four representative metrics against ILA-AMP on three models (\textsc{LLaMA2-7B}, \textsc{LLaMA3-8B}, and \textsc{Qwen3-8B}), using WikiText2 for calibration and evaluating quantization quality by WikiText2 perplexity (PPL; lower is better). 
\begin{itemize}
    \item \textbf{LIM} (Layer Input Modification) measures layer importance via the negative cosine similarity between the layer’s input and output embeddings: larger input–output changes indicate higher importance, and the metric requires calibration data.
    \item \textbf{ZD} (Z-score Distribution) quantifies importance by the fraction of outlier weights in the target layer; it does not require calibration data, and a higher outlier ratio typically implies greater sensitivity. 
    \item \textbf{Activation-based scoring} computes the Frobenius norm of layer activations, assuming that layers with larger activation energy process more critical information.
    \item \textbf{NLL Increase (NLL)}, which defines sensitivity as the increase in validation negative log-likelihood (NLL) when quantizing a single layer while keeping all other layers at a higher-precision reference, thereby directly reflecting per-layer NLL perturbations.
\end{itemize}

For each metric, we generate a layer-wise bit assignment under the same average bits(A)=4 constraint and report the resulting perplexity. As shown in Table~\ref{tab:mp_metric}, ILA-AMP consistently achieves the lowest perplexity across all models (8.31/12.83/16.15), substantially outperforming heuristic metrics. Compared to the strongest single-metric baseline NLI (which can be viewed as the first-order variant of ILA-AMP), ILA-AMP reduces PPL by 30.3\% on \textsc{LLaMA2-7B} (11.92$\rightarrow$8.31), 14.6\% on \textsc{LLaMA3-8B} (15.03$\rightarrow$12.83), and 25.9\% on \textsc{Qwen3-8B} (21.79$\rightarrow$16.15). These results indicate that under tight bit budgets, layer-wise heuristics alone are often insufficient, and explicitly modeling cross-layer error propagation and adjacent-layer coupling (ILA-AMP) is crucial for stable and high-quality mixed-precision allocation.

\section{Distribution Visualizations}
\label{app:distribution}
In this section, we visualize the effect of KOTMS from two complementary perspectives: a single-layer mechanistic view
and a cross-layer heterogeneity view.
First, we focus on a representative layer of \textsc{Qwen3-8B} (Layer 12) and compare the statistical distributions
of all weight matrices and the corresponding input activations before and after applying KOTMS within TWLA.
This single-layer analysis is designed to provide an intuitive understanding of how KOTMS improves \emph{distribution alignment}
and \emph{outlier suppression}.
Second, we extend the analysis across layers by visualizing the pre-/post-KOTMS activation distributions at different
layers of \textsc{Qwen3-8B}.
This reveals that the quantizability gains are heterogeneous across layers, motivating the necessity of ILA-AMP to
allocate activation precision adaptively rather than uniformly.
\subsection{Single-layer visualization of KOTMS}
\label{app:vis_kotms_single_layer}

\begin{figure}
    \centering
    \includegraphics[width=0.8\linewidth]{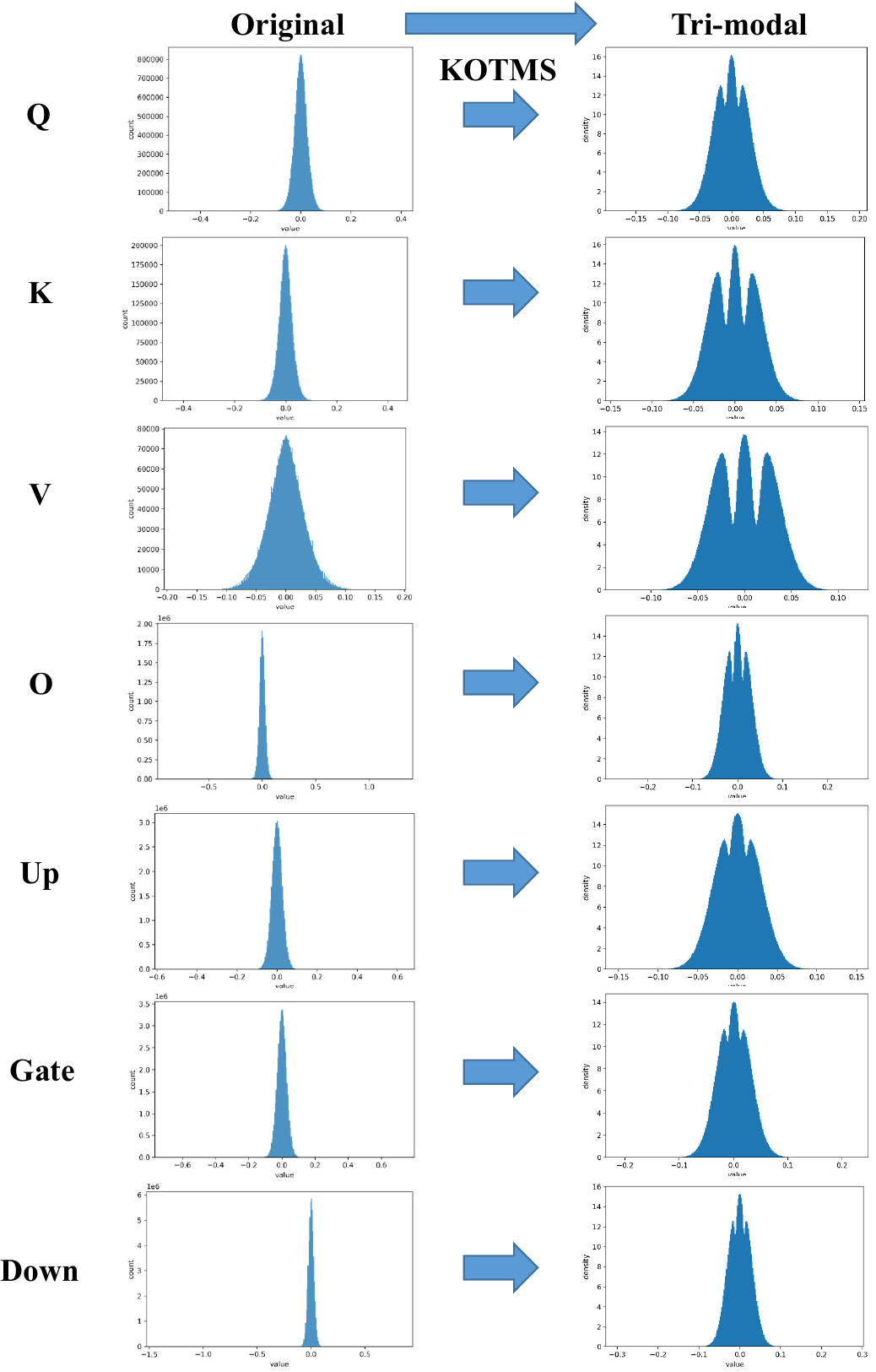}
    \caption{The weight distribution of the 12th layer in Qwen3-8B before and after TWLA.}
    \label{fig:visiual_w}
\end{figure}

Fig.~\ref{fig:visiual_w}   visualizes all weight matrices in Layer 12 of \textsc{Qwen3-8B} before and after KOTMS.
Before applying KOTMS, the weights exhibit a Gaussian-like unimodal distribution, where most values concentrate near
zero and decay smoothly toward both tails.
Such a continuous unimodal shape is inherently misaligned with ternarization: when weights do not form separable
clusters, mapping them to a discrete ternary codebook (e.g., $\{-1,0,+1\}$) typically incurs substantial quantization
error, especially under extreme low-bit regimes.
Moreover, we observe pronounced outliers in several submodules (e.g., down-projection related matrices), manifested as
heavy tails and a few extreme values.
These outliers inflate the effective dynamic range and dominate the quantization scale, which compresses the resolution
allocated to the bulk of weights and makes stable ternarization more difficult.

After applying KOTMS, the weight distributions change noticeably.
The originally unimodal profiles are reshaped into a more structured multi-cluster pattern, which is qualitatively more
compatible with the clustered structure desired by ternary quantization.
More importantly, KOTMS substantially suppresses extreme outliers and tightens the tails, leading to a markedly reduced
dynamic range.
As a result, the quantization scale is less dominated by a small number of abnormal values, improving the statistical
alignment between pretrained weights and the ternary codebook and providing a more stable foundation for high-accuracy
ternarization.

\begin{figure}
    \centering
    \includegraphics[width=\linewidth]{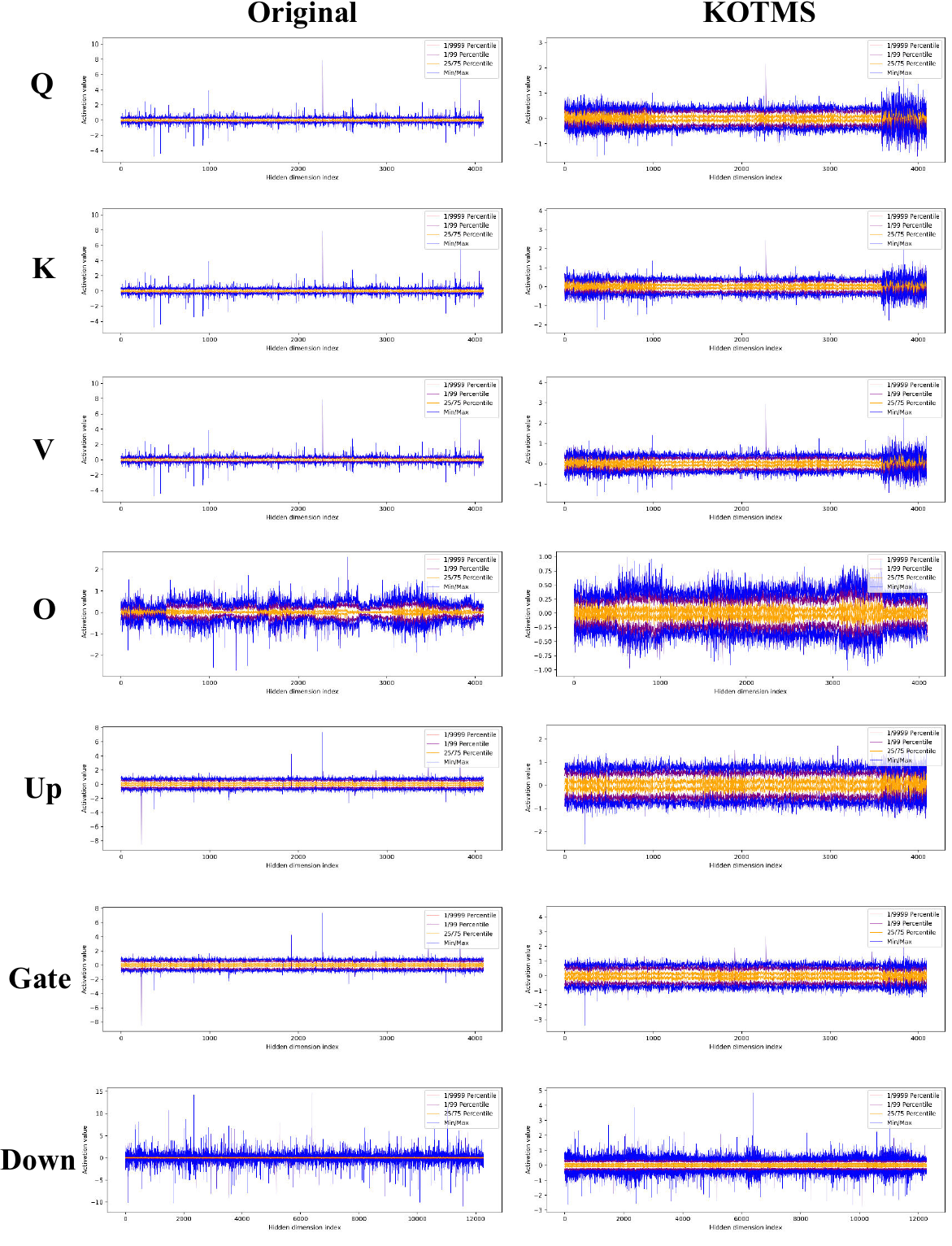}
    \caption{The activation distribution of the 12th layer in Qwen3-8B before and after TWLA.}
    \label{fig:visiual_x}
\end{figure}

A consistent trend is observed for the corresponding input activations.
Prior to KOTMS, the activation distributions exhibit clear long tails and extreme values, which are known to trigger
scale inflation and saturation under low-bit activation quantization, thereby amplifying quantization errors.
After KOTMS projection, the activations become more concentrated with visibly tightened tails, and both the number and
magnitude of extreme values are significantly reduced.
This indicates that KOTMS not only reshapes weight statistics toward ternarization-friendly forms, but also improves the
activation-side quantization conditions by mitigating outlier interference.

Fig.~\ref{fig:visiual_x} further corroborates these observations from a quantile perspective.
Before processing, the high-quantile region rises sharply, indicating extreme tail values that deviate from the main
density mass and suggesting that the original activation statistics are unfavorable for low-bit quantization.
After KOTMS, the high-quantile region drops substantially, the quantile curves become more concentrated and smoother,
and the tails are noticeably thinner.
In addition, variations across tokens become more continuous with far fewer localized anomalies.
Together, these visualizations provide direct evidence that KOTMS reorganizes the activation geometry via orthogonal
projection, suppresses outliers, and improves distribution alignment for low-bit quantization.

\subsection{Cross-layer heterogeneity in activation quantizability after KOTMS}
\label{app:vis_kotms_cross_layer}

To further reveal how KOTMS affects activation quantizability across layers, we visualize quantile plots of the
Q-projection (\textbf{Q} matrix) input activations in \textsc{Qwen3-8B} at Layers 4, 12, 24, and 36, comparing the
original activations against those after KOTMS projection (Fig.~\ref{fig:visiual_amp}).
Since KOTMS constructs an orthogonal auxiliary matrix through \emph{weight-side} optimization, applying this orthogonal
projection to activations effectively rotates and reorganizes the representation geometry, which can alleviate
anisotropy and suppress extreme values, thereby mitigating scale inflation and saturation that typically hinder
low-bit activation quantization.
Importantly, the outlier-smoothing benefit is clearly \emph{layer-dependent}, exhibiting pronounced cross-layer
heterogeneity.
As shown in Fig.~\ref{fig:visiual_amp}, after KOTMS, Layer 4 still displays noticeably more high-quantile
spikes and residual extreme fluctuations, whereas Layer 24 exhibits a substantially tighter quantile band with much
weaker high-quantile excursions, indicating a significantly larger activation-side quantizability gain at Layer 24
than at Layer 4.
This systematic heterogeneity implies that a uniform activation quantization setting across layers (e.g., the same
bitwidth or the same quantization strength) is suboptimal: it either incurs excessive quantization error in layers
where outliers remain prominent, or wastes precision budget in layers that have already become quantization-friendly
after KOTMS.
Therefore, these cross-layer quantile visualizations provide strong empirical evidence for the necessity of ILA-AMP,
which allocates activation precision adaptively according to layer-wise quantization difficulty and marginal benefit,
yielding a more robust and cost-effective end-to-end quantization outcome.

\begin{figure}
    \centering
    \includegraphics[width=\linewidth]{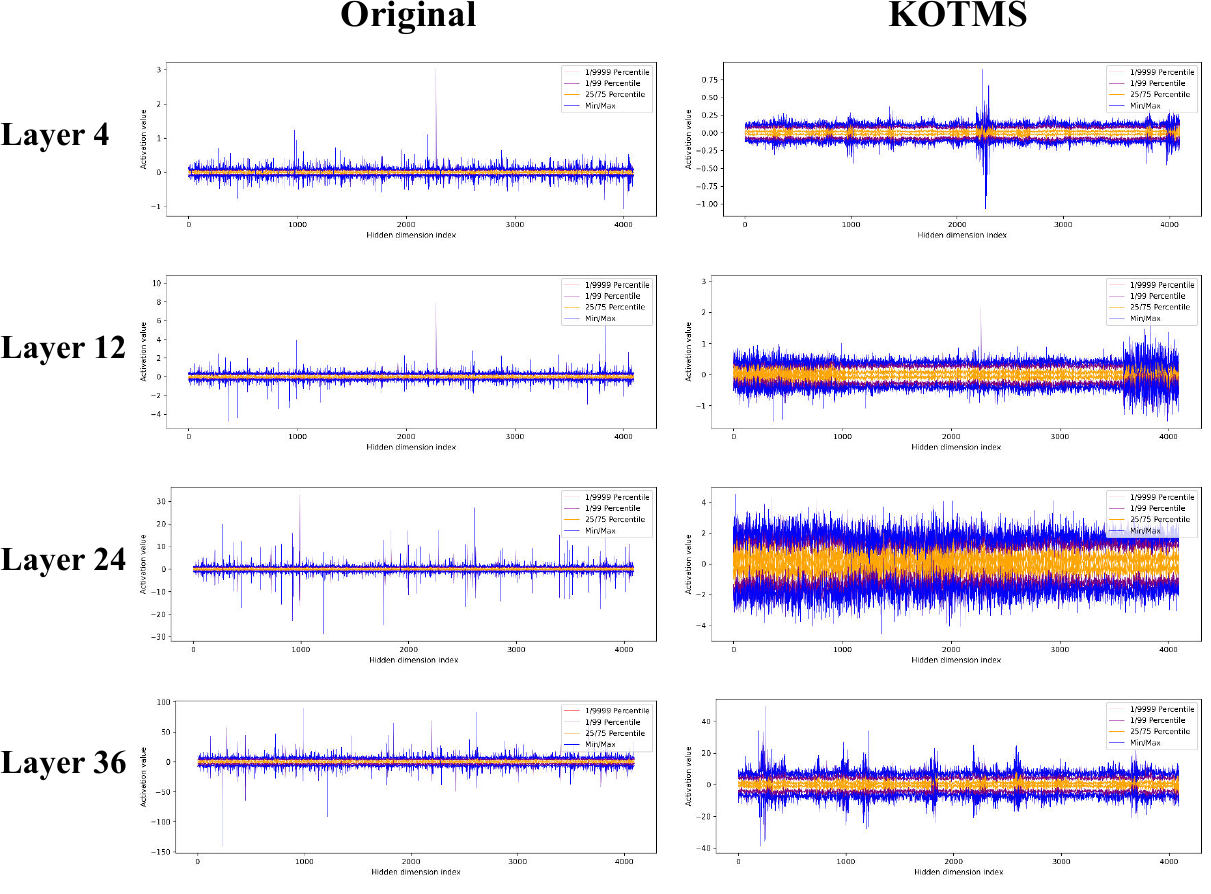}
    \caption{Cross-layer heterogeneity of activation outlier suppression after KOTMS on \textsc{Qwen3-8B}.
Quantile plots (Min/Max, 1/99\%, and 25/75\%) of the \textbf{Q}-projection input activations are shown for Layers 4, 12, 24, and 36, comparing the original activations (left) with those after KOTMS projection (right).}
    \label{fig:visiual_amp}
\end{figure}

\end{document}